\documentclass[12pt]{article}

\usepackage[utf8x]{inputenc}
\usepackage{xcolor, graphicx}

\usepackage[colorlinks,citecolor=blue,urlcolor=blue]{hyperref}

\usepackage[square,numbers]{natbib}
\usepackage{mathtools, amsfonts, amssymb}
\numberwithin{equation}{section}
\mathtoolsset{showonlyrefs=true}
\allowdisplaybreaks[2]

\usepackage{enumerate}
\usepackage{comment}
\usepackage{subcaption}
\usepackage{siunitx}
\usepackage[ruled,vlined]{algorithm2e}
\usepackage{rotating}

\providecommand{\keywords}[1]
{
  \small	
  \textit{Some key words---} #1
}

\makeatletter
\newsavebox{\hangingbox}
\newcommand{\newhanging}{\par\parindent=\wd\hangingbox}
\newcommand{\SetKwHanging}[2]{%
  \algocf@newcommand{#1}[1]{%
    \sbox\hangingbox{\hbox{\KwSty{#2}\algocf@typo\ }}%
    {\let\\\newhanging\hangindent=\wd\hangingbox\hangafter=1\unhbox\hangingbox##1}%
  }%
}%
\makeatother

\SetKwHanging{KwInput}{Input:}
\SetKwHanging{KwInitialization}{Initialization:}
\SetKwHanging{KwFPCA}{Computation of the MFPCA components:}
\SetKwHanging{KwGraph}{Creation of the graph:}
\SetKwHanging{KwGMM}{Estimation of the number of clusters:}
\SetKwHanging{KwProbaN}{Computation of the probabilities to be in each node:}
\SetKwHanging{KwProbaU}{Computation of the probabilities to be in each element of the final partition:}
\SetKwHanging{KwStop}{Stopping criterion:}
\SetKwHanging{KwAggregate}{Aggregation of two nodes:}
\SetKwHanging{KwConstruction}{Children nodes construction:}
\SetKwHanging{KwRecursion}{Recursion:}
\SetKwHanging{KwOutput}{Output:}

\usepackage{numprint}
\usepackage{booktabs}

\usepackage{chngcntr}
\usepackage{apptools}
\AtAppendix{\counterwithin{lemma}{section}}

\DeclareSIUnit{\octet}{o}
\newcounter{thm}[section]
\newcounter{appen}[section]

\newtheorem{lemma}{Lemma}

\newcounter{rem}
\newtheorem{remark}[rem]{Remark}

\newenvironment{proof}[1][Proof]{\noindent \textbf{#1.}
}{\rule{0.5em}{0.5em}}
\newcounter{scenario}[section]
\newenvironment{scenario}[1][]{\refstepcounter{scenario}\par\medskip
   \textbf{Scenario~\thescenario. #1} \rmfamily}{\medskip}
\newcommand\given[1][]{\:#1\vert\:} 

\newcommand{\EE}{\mathbb{E}}
\newcommand{\PP}{\mathbb{P}}

\newcommand{\RR}{\mathbb{R}}

\newcommand{\dd}{{\rm d}}

\newcommand{\EspCond}[2]{\EE\left( #1 \given #2 \right)}
\newcommand{\VarCond}[2]{\Var\left( #1 \given #2 \right)}
\newcommand{\CovCond}[2]{\Cov\left( #1 \given #2 \right)}

\newcommand{\setT}{\mathcal{T}}
\newcommand{\setH}{\mathcal{H}}
\newcommand{\pointt}{\boldsymbol{t}}
\newcommand{\pointts}{\boldsymbol{s}}

\newcommand{\sLp}[2]{\mathcal{L}^{#1}\left(#2\right)} 
\newcommand{\inLp}[1]{\left\langle#1\right\rangle_2} 
\newcommand{\inH}[1]{\langle\!\langle#1\rangle\!\rangle} 
\newcommand{\normH}[1]{\mid\!\mid\!\mid\!#1\!\mid\!\mid\!\mid} 
\newcommand{\LnormH}[1]{\left|\! \left|\! \left|  #1\right| \! \right|  \! \right| } 

\newcommand{\nLaw}[2]{\mathcal{N}\left(#1, #2\right)} 

\DeclareMathOperator{\Var}{Var}
\DeclareMathOperator{\Cov}{Cov}
\DeclareMathOperator*{\argmax}{arg\,max}

\newcommand{\1}{\mathbf{1}}

\def\BIC{\textsc{bic}}
\def\ARI{\textsc{ari}}

\title{Clustering multivariate functional data using unsupervised binary trees}
\author{%
Steven Golovkine\thanks{Groupe Renault \& CREST - UMR 9194, Rennes, France, \href{mailto:steven.golovkine@ensai.fr}{steven.golovkine@ensai.fr}}
\and
Nicolas Klutchnikoff\thanks{Univ Rennes, CNRS, IRMAR - UMR 6625, F-35000 Rennes, France, \href{mailto:nicolas.klutchnikoff@univ-rennes2.fr}{nicolas.klutchnikoff@univ-rennes2.fr}}
\and
Valentin Patilea\thanks{Ensai, CREST - UMR 9194, Rennes, France, \href{mailto:valentin.patilea@ensai.fr}{valentin.patilea@ensai.fr}}
}
\date{\today}

\begin{document}

\maketitle

\begin{abstract}
We propose a model-based clustering algorithm for a general class of functional data for which the components could be curves or images. The random functional data realizations could be measured with error at discrete, and possibly random, points in the definition domain. The idea is to build a set of binary trees by recursive splitting of the observations. The number of groups are determined in a data-driven way. The new algorithm provides easily interpretable results and fast predictions for online data sets. Results on simulated datasets reveal good performance in various complex settings. The methodology is applied to the analysis of vehicle trajectories on a German roundabout. 
\end{abstract}



\section{Introduction}

Motivated by a large number of applications ranging from sports to the automotive industry and healthcare, there is a great interest in modeling observation entities in the form of a sequence of possibly vector-valued measurements, recorded intermittently at several discrete points in time.  Functional data analysis (FDA) considers such data as being values on the realizations of a stochastic process, recorded with some error, at discrete random times. The purpose of FDA is to study such trajectories, also called curves or functions. See, \emph{e.g.}, \cite{ramsay_functional_2005,wang_functional_2016,horvath_inference_2012} for some recent references. The amount of such data collected grows rapidly as does the cost of their labeling. Thus, there is an increasing interest in methods that aim to identify homogeneous groups within functional datasets.

Clustering procedures for functional data have been widely studied in the last two decades, see, for instance,  \cite{hall_achieving_2012,hall_classification_2013,del_hall_19,carroll_unexpected_2013,jacques_functional_2014} and references therein. See also \cite{bouveyron_model_based_2019} for a recent textbook. In particular, for Gaussian processes, \cite{tarpey_clustering_2003} show that the cluster centers found with $k$-means are linear combinations of the eigenfunctions from the functional Principal Components Analysis (fPCA). A discriminative functional mixture model is developed by \cite{bouveyron_discriminative_2015} for the analysis of a bike sharing system from cities around the world.

Algorithms built to handle multivariate functional data have gained much attention in the last few years. Some of these methods are based on $k$-means algorithm with a specific distance function adapted to multivariate functional data. See \emph{e.g} \cite{tokushige_crisp_2007,ieva_multivariate_2013}. \cite{kayano_functional_2010} consider Self-Organizing Maps built on the coefficients of the curves into orthonormalized Gaussian basis expansion. \cite{jacques_model-based_2014} developed model-based clustering for multivariate functional data. In their model, they assume a cluster-specific Gaussian distribution for the principal component scores. The eigen-elements are estimated using an approximation of the curves into a finite dimensional functional space and are cluster specific. \cite{schmutz_clustering_2020} have recently extended the previous model by modeling all principal components whose estimated variances are non-null. The underlying model for these methods usually consider only amplitude variations. Unlike others, \cite{park_clustering_2017} present a specific model for functional data to consider phase variations. Finally, \cite{traore_clustering_2019} propose a mix between dimension reduction and non-parametric approaches by deriving the envelope and the spectrum from the curves and have applied it to nuclear safety experiments.

We propose a clustering algorithm based on the recursive construction of binary trees to recover the groups. Our idea extends the CUBT method \cite{fraiman_interpretable_2013} to the functional setting, and we therefore call it $\mathtt{fCUBT}$. At each node of the tree, a model selection test is performed, after expanding the multivariate functional data into a well chosen basis. Similarly to \cite{pelleg_x-means_2000} and \cite{cheng_model-selection-based_2004}, using the Bayesian Information Criterion (\BIC), we test whether there is evidence that the data structure is a mixture model or not. A significant advantage of our procedure compared to \cite{jacques_model-based_2014, schmutz_clustering_2020}, is its ability to estimate the number of groups within the data while this number have to be pre-specified in the other methods. Moreover, the tree structure allows us to consider only a small number of principal components at each node of the tree and not to estimate a global number of components for the clustering. Considering tree methods designed for functional data, \cite{staerman_functional_2019} extend the popular Isolation Forest algorithm used for anomaly detection, to the functional context. Our $\mathtt{fCUBT}$ algorithm also allows for classes defined by certain types of phase variations.

The remainder of the paper is organized as follows. In Section \ref{sec:model}, we define a model for a mixture of curves for multivariate functional data with the coordinates having possibly different definition domains. Given a dataset, that is a set of, possibly noisy, intermittent measures of an independent sample of realizations of the stochastic process,  in Section \ref{sec:estimation}, we explain how compute the different quantities that are required in the clustering procedure. In Section \ref{sec:fcubt}, we develop the construction of our clustering algorithm, named $\mathtt{fCUBT}$. In Section \ref{sec:emp_analysis}, we study the behavior of $\mathtt{fCUBT}$ and compare its performance with competing methods both on simulated and real datasets. Our algorithm performs well to estimate the number of groups in the data as well as grouping similar objects together. Once the tree has been grown, it can be used to predict the labels given new observations. The prediction accuracy is compared with the ones derived from supervised methods, and exhibits good performance. A real data application on vehicle trajectories analysis illustrates the effectiveness of our approach. Section \ref{sec:extensions} presents an extension of the method to images data based on the eigendecomposition of the image observations using the FCP-TPA algorithm \cite{allen_multi-way_2013}. The proofs are relegated to the Supplementary Material.

\section{Model and methodology}\label{sec:model}

\subsection{Notion of multivariate functional data}

The structure of our data, referred to as \emph{multivariate functional data}, is very similar to that presented in \cite{happ_multivariate_2018}. The data consist of independent trajectories  of a  vector-valued stochastic process $X= (X^{(1)}, \dots, X^{(P)})^\top$, $P\geq 1$. (Here and in the following, for any matrix $A$, $A^\top$ denotes its transpose.) For each $1\leq p \leq P$, let $\mathcal T_p$ be a rectangle in some Euclidean space $\mathbb R ^{d_p}$ with $d_p\geq 1$, as for instance, $\mathcal T_p = [0,1]^{d_p}$. Each coordinate $X^{(p)}:\setT_p \rightarrow \RR$ is assumed to belong to  $\sLp{2}{\setT_p}$, the Hilbert space of squared-integrable real-valued functions defined on $\setT_p$, endowed with the usual inner product that we denote by $\inLp{\cdot, \cdot}$. Thus $X$ is a stochastic process indexed by $\pointt = (t_1,\ldots,t_P)$ belonging to the $P-$fold Cartesian product $\setT:=\setT_1 \times \cdots\times \setT_P$ and taking values in the $P-$fold Cartesian product space $\setH \coloneqq \sLp{2}{\setT_1} \times \dots \times \sLp{2}{\setT_P}$. 

We consider the function $\inH{\cdot, \cdot} : \setH \times \setH \rightarrow \RR$,
\begin{equation}\label{eq:innerprodH}
\inH{f, g} \coloneqq \sum_{p=1}^{P} \inLp{f^{(p)}, g^{(p)}}= \sum_{p=1}^{P}\int_{\setT _p} f^{(p)}(t_p)g^{(p)}(t_p)\dd t_p, \quad f, g \in \setH.
\end{equation}  
Then, $\setH$ is a Hilbert space with respect to the inner product $\inH{\cdot, \cdot}$ \cite{happ_multivariate_2018}. We denote by $\normH{\cdot}$ the norm induced by $\inH{\cdot, \cdot}$. Let $\mu : \setT \rightarrow \setH$ denote the mean function of the process $X$, $\mu(\pointt) \coloneqq \EE(X(\pointt)), \pointt \in \setT.$

\subsection{A mixture model for curves}\label{subsec:mixture_curves}

The standard model-based clustering approaches consider that data is sampled from a mixture of probability densities on a finite dimensional space. As pointed out by \cite{jacques_model-based_2014}, this approach is not directly applicable to functional data since the notion of probability density generally does not exist for functional random variables. See also \cite{delaigle_defining_2010}. Consequently, in general, model-based clustering approaches assume a mixture of parametric distributions on the coefficients of the representation of the process $X$ in a basis. This is also the way we proceed in the following. 

Let $K \geq 1$ be an integer, and let $Z$ be a random variable taking values in $\{1, \dots, K\}$ such that $\PP(Z = k) = p_k$ with $p_k > 0$ and $\sum_{k=1}^{K}p_k = 1.$ The variable $Z$ is a latent variable, also called the class label, representing the cluster membership of the realizations of the process.

Although $X = \{X(\pointt)\}_{\pointt \in \setT}$ could be defined on a set of vectors and could be vector-valued, we shall call \emph{curves} its independent realizations which generate the data. We consider that the stochastic process $X$ follows a \emph{functional mixture model with $K$ components}:
\begin{equation}\label{eq:mixture_curve}
X(\pointt) = \sum_{k=1}^{K}\mu_k(\pointt)\1_{\{Z = k\}} + \sum_{j \geq 1}^{} \xi_j\phi_j(\pointt), \quad \pointt \in \setT,
\end{equation}
where
\begin{itemize}
	\item $\mu_1, \dots, \mu_K\in \setH$ are the mean curves per cluster.
	\item $\{\phi_j\}_{j \geq 1}$ is an orthonormal basis of $\setH$, $ \inH{\phi_j, \phi_{j^\prime}} = 0$ and  $\normH{\phi_j}^2 = 1$, $\forall 1\leq j\neq j^\prime <\infty$.
	\item $\xi_j$, $j\geq 1$ are real-valued random variables which are conditionally independent given $Z$. For each $1\leq k\leq K$, the conditional distribution of $\xi_j$ given $Z = k$ is a zero-mean Gaussian distribution with variance $\sigma_{kj}^2\geq 0 $, for all $j \geq 1$. Moreover,  $\sum_{k=1}^K \sum_{j \geq 1}^{}\sigma_{kj}^2 < \infty$ and for any $k\neq k^\prime$, $\sum_{j \geq 1}^{}|\sigma_{kj}^2 - \sigma_{k^\prime j}^2|  >0$. 
\end{itemize}

The condition that $\{\phi_j\}_{j \geq 1}$ is an orthonormal basis is not really necessary; it just allows us to write some technical conditions in a simpler way. Our assumptions imply that $\sum_{j \geq 1}^{}\Var(\xi_j) < \infty$ and thus $\EE\left({\normH{X}^2}\right) < \infty$. The definition of the processes $X$ \eqref{eq:mixture_curve} does not require us to know the basis $\{\phi_j\}_{j\geq 1}$. However, for  inference purposes, one needs to consider a representation in a  workable  basis. The following result presents the relationship between the two representations. 

\begin{lemma}\label{lemma:c_j}
Let $X$ be defined as in \eqref{eq:mixture_curve} for some orthonormal basis  $\{\phi_j\}_{j \geq 1}$.
Let $\{\psi_j\}_{j \geq 1}$ be another orthonormal basis in $\setH$ and consider
\begin{equation}\label{eq:c_j}
c_j = \inH{X - \mu, \psi_j}, \quad j \geq 1 \quad\text{where}\quad \mu(\cdot) = \sum_{k=1}^{K}p_k\mu_k(\cdot).
\end{equation}
Then
\begin{equation}\label{eq:c_j_law}
c_j \given Z = k \sim \nLaw{m_{kj}}{\tau_{kj}^2}, \;\;\text{where}\;\; m_{kj} = \inH{\mu_k - \mu, \psi_j} \;\;\text{and}\;\; \tau_{kj}^2 = \sum_{l \geq 1}^{}\inH{\phi_{l}, \psi_j}^2\sigma_{kl}^2.
\end{equation}
Moreover,
\begin{equation*}
\CovCond{c_i, c_j}{Z = k} = \sum_{l \geq 1}^{}\inH{\phi_l, \psi_i}\inH{\phi_l, \psi_j}\sigma_{kl}^2, \quad i, j \geq 1.
\end{equation*}
Each $c_j$ has a centered Gaussian distribution and, for any $i,j\geq 1$, 
$$\Cov(c_i, c_j) = \sum_{k=1}^{K}p_k\left(\sum_{l \geq 1}^{}\inH{\phi_{l}, \psi_i}\inH{\phi_{l}, \psi_j}\sigma_{kl}^2 + \inH{\mu_k - \mu, \psi_i}\inH{\mu_k - \mu, \psi_j}\right).$$
\end{lemma}

\begin{remark}
 Despite the Gaussian assumption for the $\xi_j$, the model \eqref{eq:mixture_curve} is quite general because there is practically no assumption concerning the basis $\{\phi_j\}_{j \geq 1}$.  Lemma \ref{lemma:c_j} shows that, no matter what the user's choice may be for the orthonormal basis $\{\psi_j\}_{j \geq 1}$, the clusters will be preserved after expressing the realizations of the process into this basis. However, depending on the objective, some bases might be more suitable than another.
\end{remark}

\begin{remark}\label{rem:diff_b}
A careful look at the proof of Lemma \ref{lemma:c_j} reveals that it remains true even if  $\{\phi_j\}_{j \geq 1}$ is not an orthonormal basis, which could be appealing for extending our framework. For instance, consider the case where the clusters of curves correspond to representations in different bases. As an illustration, let us consider the case $K = 2$ and let $\{\phi_{1,j}\}_{j\geq 1}$ and $\{\phi_{2,j}\}_{j\geq 1}$ be orthonormal systems. The union of these sets is not necessarily an orthonormal system however. The centered realizations from each cluster corresponds to the representations $\sum_{j \geq 1}^{} \eta_{1,j}\phi_{1,j}(\pointt)$ and  $\sum_{j \geq 1}^{} \eta_{2,j}\phi_{2,j}(\pointt)$, respectively. Here, $ \eta_{1j},  \eta_{2j}$, $j\geq 1$ are independent zero-mean Gaussian variables with variances $\sigma^2_{1,j}$ and $\sigma^2_{2,j}$, respectively, such that  $\sum_{j \geq 1}^{}\{ \sigma^2_{1,j} + \sigma^2_{2,j} \}< \infty$. We can then write the process $X$ in the form \eqref{eq:mixture_curve}~: $X(\pointt) - \mathbb E [X(\pointt)   ] =  \sum_{j \geq 1}^{} \xi_j\phi_j(\pointt)$ with 
$$
\xi_{2j-1} = \1_{\{Z = 1\}} \eta_{1,j} +  \1_{\{Z = 2\}} \delta_{0}  \quad \text{ and } \quad  \xi_{2j} = \1_{\{Z = 1\}} \delta_{0} +  \1_{\{Z = 2\}} \eta_{2,j} ,\quad j\geq 1,
$$
and 
$$
\phi_{2j-1} (\pointt) = \phi_{1,j}(\pointt)  \quad \text{ and } \quad \phi_{2j} (\pointt) = \phi_{2,j} (\pointt), \quad j\geq 1,\;\pointt \in\setT,
$$
 where $\delta_0$ is  the Dirac mass at the origin.
Thus each $\xi_j$ is a mixture between two centered normal variables with positive and zero variance, respectively. Given any orthonormal basis $\{\psi_j\}_{j \geq 1}$ , the $c_j = \inH{X - \mu, \psi_j}$ will still have the properties described in Lemma \ref{lemma:c_j}.
\end{remark}

\begin{remark}
Our framework and Lemma \ref{lemma:c_j} could also capture mixtures induced by some  phase variations. This kind of modeling has received increasing attention recently. See, \emph{e.g.}, \cite{marron_functional_2015,park_clustering_2017}. Indeed,   let $h_k:\setT\rightarrow \setT$, 
$1\leq k \leq K$, be different  maps defined on $\setT$. 
For instance, when $P=1$ and  $\setT_P = [0,1]$,  the maps could be in the form  $h_k(t) = \{b_kt^{a_k}+ (1-b_k)\}^{c_k}$ for some  $a_k,c_k>0$ and $0<b_k\leq 1$.    
We can then define $\{\phi_{k,j}\}_{j\geq 1}$ as $\phi_{k,j}(\pointt) = \overline{\phi}(h_k(\pointt))$, $\pointt\in\setT$,  $k=1,\ldots,K$, where $\{\overline{\phi}_j\}_{j \geq 1}$ is some reference orthonormal basis. We next consider that each cluster corresponds to a representation in a basis $\{\phi_{k,j}\}_{j \geq 1}$, $1\leq k \leq K$. Remark \ref{rem:diff_b} shows that Lemma \ref{lemma:c_j} remains valid.  More generally, each representation in a basis  $\{\phi_{k,j}\}_{j \geq 1}$, $1\leq k \leq K$, could be a mixture of two or more components as in \eqref{eq:mixture_curve}. After gathering the $K$ representations, the result would still be a functional mixture in the sense of \eqref{eq:mixture_curve}.  
\end{remark}

 Remarks 1 to 3 indicate that our modeling approach is quite flexible and can capture many interesting situations.  In applications, one cannot use an infinite number of terms in the representation of $X$, and such a representation must be truncated. We show in Lemma \ref{lemma:approx} in the Supplementary Material that such a truncation could be arbitrarily accurate.

\subsection{Multivariate Karhunen-Lo\`eve representation}

 A convenient basis for representing the realizations of $X$ is that given by the eigenfunctions of the covariance operator. In this section, following \cite{happ_multivariate_2018}, we recall the formal definition of that basis in the context of our model.

Let $C$ denote the $P\times P$ matrix-valued covariance function which, for $\pointts,\pointt \in \setT$, is defined as
\begin{equation*}\label{eq:covariance_function}
C(\pointts, \pointt) \coloneqq \EE\left(\{X(\pointts) - \EE(X(\pointts)\}  \{X(\pointt) - \EE(X(\pointt))\}^\top\right), \quad \pointts, \pointt \in \setT.
\end{equation*}
More precisely, for $1\leq p,q\leq P$, the $(p, q)$th entry of the matrix $C(\pointts, \pointt)$ is the covariance function between the $p$th and the $q$th components of the process $X$:
\begin{equation*}
C_{p, q}(s_p, t_q) \coloneqq \EE\left(\{X^{(p)}(s_p) - \EE(X^{(p)}(s_p)\} \{X^{(q)}(t_q) - \EE(X^{(q)}(t_q))\}\right), \quad s_p\in\setT_p, t_q \in\setT_q.
\end{equation*}
Following \cite{happ_multivariate_2018}, it is assumed herein that 
$$
\max_{1\leq p \leq P} \sup_{s_p\in\setT_p}\int_{\setT_q} C^2_{p, q}(s_p, t_q) \dd t_q <\infty,
$$
and that, for all $1\leq p \leq P$, $C_{p, q}(s_p, \cdot) $ is uniformly continuous in the sense that
$$
\forall \epsilon >0 \;\exists \delta >0 \; : \;|t_q-t_q^\prime|<\delta \;\Rightarrow \; \max_{1\leq p \leq P} \sup_{s_p\in \setT_p}\left| C_{p, q}(s_p, t_q)  - C_{p, q}(s_p, t_q^\prime) \right| <\epsilon.
$$
In particular, $C_{p, q}(\cdot, \cdot) $ belongs to  $\sLp{2}{\setT_p\times \setT_q}$.

Let $\Gamma: \setH \mapsto \setH$ denote the covariance operator of $X$, defined as the integral operator with kernel $C$. That is, for $f \in \setH$ and $\pointt\in\mathcal T$, the $q$th component of $\Gamma f(\pointt)$ is given by
\begin{equation*}
(\Gamma f)^{(q)}(t_q) \coloneqq \inH{C_{\cdot, q}(\cdot, t_q), f(\cdot)}, \quad t_q \in \setT_q. 
\end{equation*}
By the results in \cite{happ_multivariate_2018}, and the theory of Hilbert-Schmidt operators, \emph{e.g.} \cite[Chapter \textsc{VI}]{reed_methods_1980}, there exists a complete orthonormal basis $\{\varphi_j\}_{j \geq 1}\subset \mathcal H$ and a  sequence of real numbers $\lambda_1\geq \lambda_{2}\geq \ldots\geq 0$ such that $\Gamma \varphi_j = \lambda_{j} \varphi_j$ and $\lambda_{j}\rightarrow 0$ as $j\rightarrow \infty.$ 
The $\lambda_{j}$'s are the eigenvalues of the covariance operator $\Gamma$ and the $\varphi_j$'s are the associated eigenfunctions. Then, $X$ as defined in \eqref{eq:mixture_curve} allows for the Karhunen-Lo\`eve representation
\begin{equation}\label{dec_2}
X(\pointt) = \mu (\pointt) + \sum_{j\geq 1}\mathfrak{c}_{j} \varphi_j(\pointt), \quad \pointt\in\mathcal T,
\quad \text{ with } \quad \mathfrak{c}_{j}= \inH{X - \mu,\varphi_j},
\end{equation}
and 
$\Cov(\mathfrak{c}_{j},\mathfrak{c}_{l})=\lambda_{j} \1_{ \{j=l\} }$. Let us call $\{\varphi_j\}_{j \geq 1}$ the multivariate functional principal component analysis (MFPCA) basis. 

Let $J\geq 1$ and assume that $\lambda_{1}> \lambda_{2}\ldots>\lambda_{J}> \lambda_{J+1}$, which in particular, implies that the first $J$ eigenvalues are nonzero. By an easy extension of \cite[Theorem 3.2]{horvath_inference_2012}, we can deduce that, up to a sign, the elements of the MFPCA basis are characterized by the following property~:
\begin{align}
\varphi_1 &= \arg\max_{\varphi} \inH{\Gamma \varphi, \varphi} \quad\text{such that}\quad \normH{\varphi} = 1, \notag \\
\varphi_j &= \arg\max_{\varphi} \inH{\Gamma \varphi, \varphi} \quad\text{such that}\quad \normH{\varphi} = 1 ~\text{and}~ \inH{\varphi, \varphi_l} = 0,\;\;\forall l < j \leq J.\label{property_fpca}
\end{align}
The MFPCA basis is the one which will induce the most accurate truncation for a given $J$ (see the Lemma \ref{lemma:best_basis} for a proof). Therefore, among the workable bases one could use in practice, the MFPCA basis is likely to be a privileged one.

By Lemma \ref{lemma:c_j}, whenever $X$ is defined as in \eqref{eq:mixture_curve}, for any $J\geq 1$, the distribution of the vector $\mathfrak{c} = (\mathfrak{c}_{1},\ldots,\mathfrak{c}_{J})^\top$ defined in \eqref{dec_2} is a centered (multivariate) Gaussian mixture model (GMM) distribution
\begin{equation}\label{GMM_def}
g(\mathfrak{c}) = \sum_{k=1}^K p_k f_k(\mathfrak{c} \vert \mathbf{m}_{J,k}, \Sigma_{J,k}), \quad \mathfrak{c}\in\mathbb R^J,
\end{equation}
where, for each $1\leq k\leq K$, $f_k(\cdot \vert \mathbf{m}_{J, k}, \Sigma_{J, k})$ is the probability density function of a multivariate Gaussian distribution of the $k$th given the values of its parameters $\mathbf{m}_{J, k}$ and $\Sigma_{J, k}$ such that
$$
\mathbf{m}_{J,k} = (\inH{\mu_k - \mu, \varphi_1} ,\ldots,\inH{\mu_k - \mu, \varphi_J}  )^\top,
$$
and the $(i,j)-$entry of the matrix  $\Sigma_{J,k}$ given by 
$$
\CovCond{\mathfrak{c}_i, \mathfrak{c}_j}{Z = k} = \sum_{l \geq 1}^{}\inH{\phi_l, \varphi_i}\inH{\phi_l, \varphi_j}\sigma_{kl}^2, \quad 1\leq  i, j \leq J.
$$
We point out that the GMM defined in \eqref{GMM_def} is consistent with respect to  $J$  in the sense that, for any $J\geq 1$, $(\mathfrak{c}_{1},\ldots,\mathfrak{c}_{J})^\top$ and  $(\mathfrak{c}_{1},\ldots,\mathfrak{c}_{J},\mathfrak{c}_{J+1})^\top$ have the same label $Z$, which also coincides with the label of the curve $X$. In particular, in general, the label $Z$ of a curve in the sample, could be identified by $\mathfrak{c}_{1}$. Moreover, since $\Var(\mathfrak{c}_{1}) > \Var(\mathfrak{c}_{j})$, $\forall j>1$, it is likely that the first coefficient in the Karhunen-Lo\`eve representation is informative for identifying the mixture structure. 

Let 
\begin{equation}\label{eq:MKL_truncated}
X_{\lceil J \rceil}(\pointt) = \mu(\pointt) + \sum_{j = 1}^{J}\mathfrak{c}_j\varphi_j(\pointt), \quad \pointt \in \mathcal{T}, \quad J \geq 1,
\end{equation}
be the truncated Karhunen-Lo\`eve expansion of the process $X$ and
\begin{equation}\label{eq:KL_truncated}
X_{\lceil J^{(p)} \rceil}^{(p)}(t_p) = \mu^{(p)}(t_p) + \sum_{j = 1}^{J^{(p)}}\mathbf{c}_j^{(p)}\rho_j^{(p)}(t_p), \quad t \in \setT_p, \quad J^{(p)} \geq 1, \quad 1 \leq p \leq P,
\end{equation}
be the truncated Karhunen-Lo\`eve expansion of the components of the process $X$. \cite{happ_multivariate_2018} derive a direct relationship between the truncated representations \eqref{eq:KL_truncated} of the single elements $X^{(p)}$ and the truncated representation \eqref{eq:MKL_truncated} of the multivariate functional data $X$. We recall this result in the following lemma.

\begin{lemma}[\cite{happ_multivariate_2018}, Proposition 5]\label{lemma:KLtoMKL}
Let $X_{\lceil J \rceil}$ be the truncation of $X$ as in \eqref{eq:MKL_truncated} and $X_{\lceil J^{(p)} \rceil}^{(p)}$ be the truncation of $X^{(p)}$ for each $1 \leq p \leq P$ as in \eqref{eq:KL_truncated}.

\begin{enumerate}
\item Let $\Gamma^{(p)}$ be the univariate covariance operator associated with $X^{(p)}$. The positive eigenvalues of $\Gamma^{(p)}$, $\lambda^{(p)}_1\geq \ldots \geq \lambda^{(p)}_{J^{(p)}} > 0, J^{(p)} < J$ correspond the positive eigenvalues of the matrix $\mathbf{A}^{(p)} \in \mathbb{R}^{J \times J}$ with entries
$$\mathbf{A}_{jj^\prime}^{(p)} =  (\lambda_j \lambda_{j^\prime})^{1/2} \inLp{\varphi_j^{(p)}, \varphi_{j^\prime}^{(p)}}, \quad j,j^\prime = 1, \ldots, J.$$
The eigenfunctions of $\Gamma^{(p)}$ are given by
$$\rho_j^{(p)}(t_p) = \left(\lambda_j^{(p)}\right)^{-1/2} \sum\nolimits_{j^\prime=1}^J \lambda_{j^\prime}^{1/2}  [\boldsymbol{u}_j^{(p)}]_{j^\prime} \varphi_{j^\prime}^{(p)}(t_p), \quad t_p \in \setT_,~ j = 1, \ldots, J^{(p)},$$
where $\boldsymbol{u}_j^{(p)}$ denotes an (orthonormal) eigenvector of $\mathbf{A}^{(p)}$ associated with  eigenvalue $\lambda_j^{(p)}$ and $[\boldsymbol{u}_j^{(p)}]_{j^\prime}$ denotes the $j^\prime$-th entry of this vector. Finally, the univariates scores are
$$\mathbf{c}_j^{(p)} 
= \inLp{X^{(p)}, \rho_j^{(p)}}
= \left(\lambda_j^{(p)}\right)^{-1/2} \sum\nolimits_{j^\prime=1}^J \lambda_{j^\prime}^{1/2}  [\boldsymbol{u}_j^{(p)}]_{j^\prime} \sum\nolimits_{j^{\prime\prime} = 1}^J \mathfrak{c}_{j^{\prime\prime}} \inLp{\varphi_{j^\prime}^{(p)}, \varphi_{j^{\prime\prime}}^{(p)}}.$$

\item The positive eigenvalues of $\Gamma$, $\lambda_1, \ldots, \lambda_J > 0$, with $J \leq \sum_{p = 1}^{P} J^{(p)} \eqqcolon J_+$ are the positive eigenvalues of the matrix $\mathbf{Z} \in \mathbb{R}^{J_+ \times J_+}$ consisting of blocks $\mathbf{Z}^{(pp^\prime)} \in \mathbb{R}^{J^{(p)} \times J^{(p^\prime)}}$ with entries
$$\mathbf{Z}_{jj^\prime}^{(pp^\prime)} = \Cov \left(\mathbf{c}_j^{(p)}, \mathbf{c}_{j^\prime}^{(p^\prime)}\right), \quad j = 1, \ldots ,  J^{(p)},~ j^\prime = 1, \ldots ,  J^{(p^\prime)},~ p,p^\prime = 1, \ldots, P.$$
The eigenfunctions of $\Gamma$ are given by $\varphi_j^{(p)}(t_p) = \sum\nolimits _{j^\prime = 1}^{J^{(p)}} [\boldsymbol{v}_j]_{j^\prime}^{(p)} \rho_j^{(p)}(t_p), t_p \in \setT_p, j = 1, \ldots, J,$ where $[\boldsymbol{v}_j]^{(p)} \in \mathbb{R}^{J^{(p)}}$ denotes the $p$-th block of an (orthonormal) eigenvector $\boldsymbol{v}_j$ of $\mathbf{Z}$  associated with eigenvalue $\lambda_j$. The scores $\mathfrak{c}_j \in\mathbb R $ are given by
$$\mathfrak{c}_j = \sum\nolimits _{p = 1}^P \sum\nolimits _{j^\prime = 1}^{J^{(p)}}  [\boldsymbol{v}_j]_{j^\prime}^{(p)} \mathbf{c}_{j^\prime}^{(p)}.$$
\end{enumerate}
\end{lemma}

Lemma \ref{lemma:KLtoMKL} allows us to compute the scores of the $P$-dimensional stochastic process $X$ using the scores of each of the $P$ components $X^{(p)}$  as univariate building bloks.  However, when $P$ grows, it might not be suitable to perform $P$ univariate fPCA from a computational point of view. Recently, \cite{hu_sparse_2020} extend this result to large-dimensional processes by imposing a sparsity assumption. This possible extension will be investigated in future work. 

\section{Learning the parameters}\label{sec:estimation}

In real data applications, the realizations of $X$ are usually measured with error at discrete, and possibly random, points in the definition domain. Therefore, let us consider $N$ curves $X_1, \dots, X_n, \dots, X_N$ generated as a random sample of the $P$-dimensional stochastic process $X$ with continuous trajectories. For each $1 \leq n \leq N$, and given a vector of positive integers $\boldsymbol{M}_n = (M_n^{(1)}, \dots, M_n^{(P)}) \in \mathbb{R}^P$, let $T_{n, \boldsymbol{m}} = (T_{n, m_1}^{(1)}, \dots, T_{n, m_P}^{(P)}), 1 \leq m_p \leq M_n^{(p)}, 1 \leq p \leq P$, be the random observation times for the curve $X_n$. These times are obtained as independent copies of a variable $\boldsymbol{T}$ taking values in $\mathcal{T}$. The vectors $\boldsymbol{M}_1, \dots, \boldsymbol{M}_N$ represent an independent sample of an integer-valued random vector $\boldsymbol{M}$ with expectation $\boldsymbol{\mu}_{\boldsymbol{M} }$ which increases with $N$. We assume that the realizations of $X$, $\boldsymbol{M}$ and $\boldsymbol{T}$ are mutually independent. The observations associated with a curve, or trajectory, $X_n$ consist of the pairs $(Y_{n, \boldsymbol{m}}, T_{n, \boldsymbol{m}}) \in \mathbb{R}^P \times \mathcal{T}$, where $\boldsymbol{m} = (m_1, \dots, m_P), 1 \leq m_p \leq M_n^{(p)}$, $1 \leq p \leq P$, and $Y_{n, \boldsymbol{m}}$ is defined as
\begin{equation}\label{data_real}
Y_{n, \boldsymbol{m}} = X_n(T_{n, \boldsymbol{m}}) + \varepsilon_{n, \boldsymbol{m}}, \quad  1 \leq n \leq N,
\end{equation}
with the $\varepsilon_{n, \boldsymbol{m}}$ being independent copies of a centered error random vector $\varepsilon\in \mathbb{R}^P$ with finite variance. We use the notation $X_n(\boldsymbol{t})$ for the value at $\boldsymbol{t}$ of the realization $X_n$ of $X$. The $N$-sample of $X$ is composed of two sub-populations: a \textit{learning set} of $N_0$ curves to estimate the mixture components of the process $X$ and a set of $N_1$ curves to be classified using the previous grouping as a classifier that we call the \textit{online set}. Thus, $1 \leq N_0, N_1 < N$ and $N_0 + N_1 = N$. Let $X_1, \dots, X_{N_0}$ denote the curves corresponding to the \textit{learning set}. The learning sample will be used to estimate the mean and covariance functions, as well as the eigencomponents, of the process $X$. Our first objective is to construct a partition $\mathcal{U}$ of the space $\setH$ using the learning sample. Then, the second aim is to use the partition $\mathcal{U}$ as a classifier for a possibly very large set of $N_1$ new curves. Let $X_{[1]} = X_{N_0 + 1}, \dots, X_{[N_1]} = X_N,$ denote the curves from the \textit{online set} to be classified using the partition $\mathcal{U}$. 

\subsection{Estimation of mean and covariance}

We develop estimators for the mean and the covariance functions of a component $X^{(p)}, 1 \leq p \leq P$ from the process $X$. These estimators are used to compute estimators of eigenvalues and eigenfunctions of $X^{(p)}$ for the  expansion \eqref{eq:KL_truncated}. It is worthwhile to notice that, because of Lemma \ref{lemma:KLtoMKL}, we do not need to estimate the covariance between $X^{(p)}$ and $X^{(q)}$ for $p \neq q$.

Let $\widehat{X}_n^{(p)}$ be a suitable nonparametric estimator of the curve $X_n^{(p)}$ applied with the $M_n^{(p)}$ pairs $(Y_{n, m_p}^{(p)}, T_{n, m_p}^{(p)}), n = 1, \dots, N_0$, as for instance a local polynomial estimator such as the one defined in \cite{golovkine_learning_2020}. With at hand, the $\widehat{X}_n$'s tuned for the mean estimation, let
\begin{equation}\label{eq:est_mean}
\widehat \mu_{N_0}^{(p)}(t_p) = \frac{1}{N_0} \sum_{n=1}^{N_0} \widehat{X}_n^{(p)}(t_p), \quad t_p \in \setT_p. 
\end{equation}

For the covariance function, following \cite{yao_functional_2005-1}, we distinguish the diagonal from the non-diagonal points. With at hand, the $\widehat{X}_n^{(p)}$'s tuned for the covariance function estimation,
\begin{equation}\label{eq:est_cov1}
\widehat{C}_{p, p}(s_p,t_p) = \frac{1}{N_0} \sum_{n=1}^{N_0} \widehat{X}_n^{(p)}(s_p)\widehat{X}_n^{(p)}(t_p) - \widehat{\mu}_{N_0}^{(p)}(s_p)\widehat{\mu}_{N_0}^{(p)}(t_p) ,\quad s_p, t_p \in \setT_p, \;\;s_p \neq t_p.
\end{equation}
The diagonal of the covariance is then estimated using two-dimensional kernel smoothing with $\widehat{C}_{p, p}(s_p,t_p), s_p \neq t_p$ as input data. See \cite{yao_functional_2005-1} for the details.

\subsection{Derivation of the MFPCA components}\label{subsec:mfpca_prac}

Following \cite{happ_multivariate_2018}, using Lemma \ref{lemma:KLtoMKL}, we estimate the multivariate components for $X$ by plugging the univariate components computed from each $X^{(p)}$. These estimations are done as  follows. First, we perform an univariate fPCA on each of the components of $X$ separately. For a component $X^{(p)}$, the eigenfunctions and eigenvectors are computed as a matrix analysis of the estimated covariance $\widehat{C}_{p, p}$, from \eqref{eq:est_cov1}. This results in a set of eigenfunctions $(\widehat{\rho}_1^{(p)}, \ldots, \widehat{\rho}_{J^{(p)}}^{(p)})$ associated with a set of eigenvalues $(\widehat{\lambda}_1^{(p)}, \ldots, \widehat{\lambda}_{J^{(p)}}^{(p)})$ for a given truncation integer $J^{(p)}$. Then, the univariate scores for a realization $X_n^{(p)}$ of $X^{(p)}$ are given by $\widehat{\mathbf{c}}_{j, n}^{(p)} = \langle\widehat{X}_n^{(p)}, \widehat{\rho}_j^{(p)}\rangle_2, ~1 \leq j \leq J^{(p)}$. These scores might be estimated by numerical integration. However, in some cases, \emph{e.g.} for sparse data, it may be more suitable to use the PACE method (see \cite{yao_functional_2005-1}). We then define the matrix $\mathcal{Z} \in \mathbb{R}^{N_0 \times J_+}$, where on each row we concatenate the scores obtained for the $P$ components of  the $n$th observation~: 
$(\widehat{\mathbf{c}}_{1, n}^{(1)}, \ldots, \widehat{\mathbf{c}}_{J^{(1)}, n}^{(1)}, \ldots, \widehat{\mathbf{c}}_{1, n}^{(P)}, \ldots, \widehat{\mathbf{c}}_{J^{(p)}, n}^{(P)})$. An estimate $\widehat{\mathbf{Z}} \in \mathbb{R}^{J_+ \times J_+}$ of the matrix $\mathbf{Z}$, from Lemma \ref{lemma:KLtoMKL}, is given by $\widehat{\mathbf{Z}} = (N_0 - 1)^{-1}\mathcal{Z}^\top\mathcal{Z}$. An eigenanalysis of the matrix $\widehat{\mathbf{Z}}$ is done to estimate the eigenvectors $\widehat{\boldsymbol{v}}_j$ and eigenvalues $\widehat{\lambda}_j$. The multivariate eigenfunctions are estimated with
$\widehat{\varphi}_j^{(p)}(t_p) = \sum\nolimits_{j^\prime = 1}^{J^{(p)}}[\widehat{\boldsymbol{v}}_j]_{j^\prime}^{(p)}\widehat{\rho}_{j^\prime}^{(p)}(t_p), t_p \in \setT_p, 1 \leq j \leq J_+, 1 \leq p \leq P.$
and the multivariate scores with
$\widehat{\mathfrak{c}}_{j, n} = \mathcal{Z}_{{n,\cdot}}\widehat{\boldsymbol{v}}_j, 1 \leq n \leq N_0, 1 \leq j \leq J_+$.
The multivariate Karhunen-Lo\`eve expansion of the process $X$ is thus 
\begin{equation}\label{eq:KL_estim}
\widehat{X}_{n}(\pointt) = \widehat{\mu}_{N_0}(\pointt) + \sum_{j = 1}^J \widehat{\mathfrak{c}}_{j, n}\widehat{\varphi}_j(\pointt), \quad \pointt \in \mathcal{T}.
\end{equation} 
where $\widehat{\mu}_{N_0}(\cdot) = (\widehat \mu_{N_0}^{(1)}(\cdot), \ldots, \widehat \mu_{N_0}^{(P)}(\cdot))$ is the vector of the estimated mean functions.

\section{Multivariate functional clustering}\label{sec:fcubt}

Let $\mathcal{S}$ be a sample of realizations of the process $X$, defined in \eqref{eq:mixture_curve}. We consider the problem of learning a partition $\mathcal{U}$ such that every element $U$ of $\mathcal{U}$ gathers similar elements of $\mathcal{S}$. Our clustering procedure follows the idea of clustering using unsupervised regression trees (CUBT), considered by \cite{fraiman_interpretable_2013}, which we adapt to functional data. In the following, we describe in detail the Functional Clustering Using Unsupervised Binary Trees (\texttt{fCUBT}) algorithm.

\subsection{Building the maximal tree}

%
%
%
%
%
%

Let $\mathcal{S}_{N_0} = \{X_1, \dots, X_{N_0}\}$ be a training sample composed of $N_0$ independent realizations of the stochastic process $X \in \setH$ defined in \eqref{eq:mixture_curve}. In the following, a tree $\mathfrak{T}$ is a full binary tree, meaning every node has zero or two children, which represents a nested partition of the sample  $\mathcal{S}_{N_0}$. We will denote  the depth of a tree $\mathfrak{T}$ by $\mathfrak{D} \geq 1$. 

A tree $\mathfrak{T}$ starts with the root node $\mathfrak{S}_{0, 0}$ to which we assign the whole space sample  $\mathcal{S}_{N_0}$. Next,
every node $\mathfrak{S}_{\mathfrak{d, j}} \subset \mathcal{S}_{N_0} $ 
is indexed by the pair $(\mathfrak{d, j})$ where $\mathfrak{d}$ is the depth index of the node, with $0 \leq \mathfrak{d} < \mathfrak{D}$, and $\mathfrak{j}$ is the node index, with $0 \leq \mathfrak{j} < 2^\mathfrak{d}$. A non-terminal node $(\mathfrak{d, j})$ has two children, corresponding to disjoint subsets $\mathfrak{S}_{\mathfrak{d} + 1, 2\mathfrak{j}}$ and $\mathfrak{S}_{\mathfrak{d} + 1, 2\mathfrak{j} + 1}$ of $\mathcal{S}_{N_0}$ such that $\mathfrak{S}_{\mathfrak{d}, \mathfrak{j}} = \mathfrak{S}_{\mathfrak{d} + 1, 2\mathfrak{j}} \cup \mathfrak{S}_{\mathfrak{d} + 1, 2\mathfrak{j} + 1}.$ A terminal node $(\mathfrak{d, j})$ has no children.

A tree $\mathfrak{T}$ is thus defined using a top-down procedure by recursively splitting. At each stage, a node $(\mathfrak{d, j})$ is possibly split into two subnodes, namely the left and right child, with indices $(\mathfrak{d} + 1, 2\mathfrak{j})$ and $(\mathfrak{d} + 1, 2\mathfrak{j} + 1)$, respectively, provided it fulfills some condition. A multivariate functional principal components analysis as presented in Section \ref{subsec:mfpca_prac},  with $J$ components,   is then conducted on the elements of $\mathfrak{S}_{\mathfrak{d, j}}$. This results in a set of eigenvalues  $\Lambda_{\mathfrak{d, j}} = (\lambda_{\mathfrak{d, j}}^1, \dots, \lambda_{\mathfrak{d, j}}^{J}  )$ associated with a set of eigenfunctions  $\Phi_{\mathfrak{d, j}} = (\varphi_{\mathfrak{d, j}}^1, \dots, \varphi_{\mathfrak{d, j}}^{J})$. The matrix of scores $C_{\mathfrak{d, j}}$ is then defined with the columns built with the projections of the elements of $\mathcal{S}_{\mathfrak{d, j}}$ onto the elements of $\Phi_{\mathfrak{d, j}}$. More precisely, to each $X_n \in \mathfrak{S}_{\mathfrak{d, j}}$, there is a corresponding column   of size $J$  defined as
\begin{equation}\label{eq:proj_matrix}
\mathfrak{c}_{n, \mathfrak{d, j}} = 
\begin{pmatrix}
	\inH{X_n - \mu_{\mathfrak{d, j}}, \varphi_{\mathfrak{d, j}}^{1}},\ldots ,\inH{X_n - \mu_{\mathfrak{d, j}}, \varphi_{\mathfrak{d, j}}^{J}}
\end{pmatrix}^\top, 
\end{equation}
where $\mu_{\mathfrak{d, j}}$ is the mean curve within the node $\mathfrak{S}_{\mathfrak{d, j}}$. 

We can retrieve the groups (clusters) of curves considering a mixture model as in Section \ref{subsec:mixture_curves} for the columns of the matrix of scores. 
At each node $\mathfrak{S}_{\mathfrak{d, j}}$,
for each $K = 1, \dots, K_{max}$, we fit a GMM as in \eqref{GMM_def} to the columns of the matrix $C_{\mathfrak{d, j}}$.
The resulting models are denoted as $\{\mathcal{M}_1, \dots, \mathcal{M}_{K_{max}}\}$. To fit $\mathcal M_1$, we use the standard mean and variance matrix estimation of Gaussian distributions. To fit each of the models $\{\mathcal{M}_2, \dots, \mathcal{M}_{K_{max}}\}$, we use an EM algorithm \cite{dempster_maximum_1977}. In particular, the EM algorithm assigns a label to each curve in the node.  We next consider the \BIC{}, defined in \cite{schwarz_estimating_1978}, and determine 
\begin{equation}\label{eq:K_hat}
\widehat{K}_{\mathfrak{d, j}} = \argmax_{K = 1, \dots, K_{max}} \BIC(\mathcal{M}_K) = \argmax_{K = 1, \dots, K_{max}} \{2\log(\mathcal{L}_K) - \kappa\log{\lvert \mathfrak{S}_{\mathfrak{d, j}}\rvert} \},
\end{equation}
where $\mathcal{L}_K$ is the likelihood function of the $K-$components multivariate Gaussian mixture model $\mathcal{M}_K$ for the data at the node $\mathfrak{S}_{\mathfrak{d, j}}$, that is 
$$
\mathcal{L}_K = \prod_{n = 1}^{\lvert \mathfrak{S}_{\mathfrak{d, j}} \rvert}\sum_{k = 1}^{K}p_kf_k(\mathfrak{c}_{n, \mathfrak{d, j}} \mid \mathbf{m}_{ J,k  }, \Sigma_{ J,k  }),
$$
$\kappa = K + K J  + K J ( J + 1) / 2 - 1$ is the dimension of the model $\mathcal{M}_K$ and $\lvert \mathfrak{S}_{\mathfrak{d, j}}\rvert$ is the cardinality of the set $\mathfrak{S}_{\mathfrak{d, j}}$.
If $\widehat{K}_{\mathfrak{d, j}} > 1$, we split $\mathfrak{S}_{\mathfrak{d, j}}$ using the model $\mathcal{M}_2$, that is a mixture of two Gaussian vectors. Otherwise, the node is considered to be a terminal node and the construction of the tree is stopped for this node.

The recursive procedure continues downward until one of the following stopping rules are satisfied: there are less than $\mathtt{minsize}$ observations in the node or the estimation $\widehat{K}_{\mathfrak{d, j}}$ of the number of clusters in the mode $\mathfrak{S}_{\mathfrak{d, j}}$ is equal to $1$. The value of the positive integer $\mathtt{minsize}$ is set by the user. When the algorithm ends, a label is assigned to each leaf (terminal node). The resulting tree is referred to as the maximal binary tree. A quasi-formal description of Algorithm \ref{alg:tree_construction}, closer to the high-level computer language, is provided in the Section \ref{app:algo} in the Supplementary Material.  

Our algorithm provides a partition of the sample. Each observation belongs to a leaf which is associated with a unique label. In a perfect case, this tree will have the same number of leaves as the number of mixture components of $X$. In practice, it is rarely the case, and the number of leaves may be much larger than the number of clusters. That is why an agglomerative step, that we call the joining step, should also be considered. Note that we do not have to pre-specify the number of clusters before performing the joining step. Moreover, if the number of clusters in the maximal tree is that wanted by the user, it is not necessary to run the joining step. 

The original procedure, developed in \cite{fraiman_interpretable_2013}, included a pruning step. For each sibling node, this step computes a measure of dissimilarities between them and collapses them if this measure is lower than a predefined threshold. They need this step because their splitting criteria is based on the deviance reduction between a node and its children. The algorithm is stopped when this deviance reduction is less than a predefined threshold. The pruning step is possibly used therefore to revise an undesirable split. In our case, the splitting criteria is not based on how much deviance we gain at each node of the tree but on an estimation of the number of modes of a Gaussian mixture model. Thus, the pruning step is not useful in our case.

This method has three hyperparameters that should be set by the user. The first one is $J$, the number of components kept when an MFPCA is run.  In our implementation, $J=\sum_{p=1}^P J^{(p)} $, and the user chooses the values $J^{(p)} $. According to the notation in Lemma \ref{lemma:KLtoMKL}, for each $1\leq p \leq P$, $J^{(p)} $ represents the numbers of scores for each coordinate (or variable) $X^{(p)}$. 
The default  value for each $J^{(p)} $ is set to explain $95\%$ of the variance of the variable $X^{(p)}$ within the data. 
Let us point out that the $ J^{(p)} $ need not be large in order to detect a mixture structure. In fact, as explained in Section \ref{nb_sco_sm} in the Supplementary material, the first score could already detect the mixture in many situations. 
Therefore, when the rule of percentage of explained variance applied for each $p$, leads to large $J$, we recommend considering only one or two scores for each, or at least some coordinates $X^{(p)}$. One can thus keep $  J  $ reasonably low. Higher dimension  $  J  $ for the vector of scores for which we consider a Gaussian mixture model could result  in less performance for detecting a mixture structure in the nodes, especially for those with few data. 
The second hyperparameter is the minimum number of elements within a node to be considered to be split ($\mathtt{minsize}$).  In practice, it should be larger than $2$, 
 the default value is set to $10$. The last one, $K_{max}$, refers to the number of Gaussian mixture models to try in order to decide if there is at least two clusters in the data. Its default value is set to $5$ but, in practice, to reduce computation, it could be set to $2$. Thus, we will just compare the model with two modes against that with one group.  In some rare cases, we noticed that $K_{max} = 2$ did not allow a mixture structure to be detected while this was possible with slightly larger $K_{max}$. This was the reason for proposing a default $K_{max}$ larger than 2, even if, as pointed out by a reviewer, $K_{max} = 2$ would be the natural default value with a binary splitting, as we propose.

\subsection{Joining step}

In this step, the idea is to join terminal nodes which do not necessarily share the same direct ascendant. Let $\mathcal{G} = \left(V, E\right)$ be a graph where $V = \{\mathfrak{S}_{\mathfrak{d, j}}, 0 \leq j < 2^\mathfrak{d}, 0 \leq \mathfrak{d} < \mathfrak{D} \mathbin{\mid} \mathfrak{S}_{\mathfrak{d, j}} \;\text{is a terminal node}\}$ is a set of vertices and $E \subseteq \mathbf{E}$ is a set of edges with $\mathbf{E}$ the complete set of unordered pairs of vertices $\left\{(\mathfrak{S}_{\mathfrak{d, j}}, \mathfrak{S}_{\mathfrak{d^\prime, j^\prime}}) \mathbin{\mid} \mathfrak{S}_{\mathfrak{d, j}}, \mathfrak{S}_{\mathfrak{d^\prime, j^\prime}} \in V \;\text{and}\; \mathfrak{S}_{\mathfrak{d, j}} \neq \mathfrak{S}_{\mathfrak{d^\prime, j^\prime}}\right\}$. Let us clarify the definition of the set of edges $E$. Consider $(\mathfrak{S}_{\mathfrak{d, j}}, \mathfrak{S}_{\mathfrak{d^\prime, j^\prime}}) \in \mathbf{E}$. We are interested in the union of the nodes in the pair, $\mathfrak{S}_{\mathfrak{d, j}} \cup \mathfrak{S}_{\mathfrak{d^\prime, j^\prime}}$. After performing an MFPCA on the set $\mathfrak{S}_{\mathfrak{d, j}} \cup \mathfrak{S}_{\mathfrak{d^\prime, j^\prime}}$, we compute the matrix $C_{(\mathfrak{d, j}) \cup (\mathfrak{d^\prime, j^\prime})}$ using \eqref{eq:proj_matrix}. For $K = 1, \dots, K_{max}$, we fit a $K$-component Gaussian mixture model using an EM algorithm to $C_{(\mathfrak{d, j}) \cup (\mathfrak{d^\prime, j^\prime})}$. Then, if the estimated number of clusters $\widehat{K}_{(\mathfrak{d, j}) \cup (\mathfrak{d^\prime, j^\prime})}$ using \eqref{eq:K_hat}, is equal to $1$, we add the pair to $E$. Finally, we have 
\begin{equation}\label{eq:set_edges}
E = \left\{(\mathfrak{S}_{\mathfrak{d, j}}, \mathfrak{S}_{\mathfrak{d^\prime, j^\prime}}) \mathbin{\mid} \mathfrak{S}_{\mathfrak{d, j}}, \mathfrak{S}_{\mathfrak{d^\prime, j^\prime}} \in V,\; \mathfrak{S}_{\mathfrak{d, j}} \neq \mathfrak{S}_{\mathfrak{d^\prime, j^\prime}} \;\text{and}\; \widehat{K}_{(\mathfrak{d, j}) \cup (\mathfrak{d^\prime, j^\prime})} = 1\right\}.
\end{equation}
We associate with each element $(\mathfrak{S}_{\mathfrak{d, j}}, \mathfrak{S}_{\mathfrak{d^\prime, j^\prime}})$ of $E$, the value of the \BIC{} that corresponds to $\widehat{K}_{(\mathfrak{d, j}) \cup (\mathfrak{d^\prime, j^\prime})}$. The edge of $\mathcal{G}$ that corresponds to the maximum of the \BIC{} is then removed and the associated vertices are joined. Thus, there is one cluster less. This procedure is run recursively until no pair of nodes can be joined according to the \BIC{} or  only one node in the tree remains. Finally, unique labels are associated with  each remaining element of $V$ and the partition of $\mathfrak{S}_{0, 0}$, denoted as $\mathcal{U}$, is returned. A quasi-formal description of Algorithm \ref{alg:join_nodes}, closer to a high-level computer language, is provided in the Section \ref{app:algo} in the Supplementary Material.

Compared to other clustering algorithms, our approach inherits 
the interpretability of the trees algorithm. The joining step we propose  likely reduces the number of spurious clusters, which is supposed to improve the interpretability. Moreover, when two clusters are joined, one can still investigate them and interpret the fact that the algorithm merges the two clusters.

\subsection{Classification of new observations}\label{class_new_obs}

With at hand, a partition $\mathcal{U}$ of $\mathcal{S}_{N_0}$ obtained from a learning set of $N_0$ realizations of the process $X$, we aim to classify $N_1$ new trajectories of $X$ from what we call, the online dataset. Denote $\mathcal{S}_{N_1}$ this set of new trajectories. 

Let $\mathfrak{X}$ be an element of the set $\mathcal{S}_{N_1}$. The descent of the tree $\mathfrak{T}$ is performed as  follows. Let $\mathfrak{S}_{\mathfrak{d, j}}, 0 \leq j < 2^\mathfrak{d}, 0 \leq \mathfrak{d}< \mathfrak{D}$ be the node at hand such that $\mathfrak{S}_{\mathfrak{d, j}}$ is not a terminal node. We compute the projection $\mathfrak{X}$ onto the eigenfunctions $\Phi_{\mathfrak{d, j}}$. This results in the vector
$$\left(\inH{\mathfrak{X} - \mu_{\mathfrak{d, j}}, \varphi^1_{\mathfrak{d, j}}}, \dots, \inH{\mathfrak{X} - \mu_{\mathfrak{d, j}},  \varphi^{J }_{\mathfrak{d, j}}} \right)^\top.$$
Using the $2$-component GMM fitted on this node, we then compute the posterior probability of $\mathfrak{X}$  belonging to each of the components. At this point, we have the probability for $\mathfrak{X}$ belonging to each node given that it belongs to the parent node. We write $\mathtt{Pa}(\mathfrak{S})$, the set of parent nodes of the node $\mathfrak{S}$ which includes the node itself. We denote $\PP^\star(\cdot) \coloneqq \PP(\cdot \mid X_1, \dots, X_{N_0})$. Then, the probability for $\mathfrak{X}$ to be in the node $\mathfrak{S}_{\mathfrak{d, j}}$ is then given by
$$\PP^\star(\mathfrak{X} \in \mathfrak{S}_\mathfrak{d, j}) = \prod_{\mathfrak{S} \in \mathtt{Pa}(\mathfrak{S}_{\mathfrak{d, j}})}^{} \PP^\star(\mathfrak{X} \in \mathfrak{S} \mid \mathfrak{X} \in \mathtt{Pa}(\mathfrak{S})).$$
We have $\PP^\star(\mathfrak{X} \in \mathfrak{S}_{0, 0}) = 1$. In order to compute the probabilities of belonging to each element of the partition $\mathcal{U}$, let $\mathfrak{S}_{\mathfrak{d, j}}$ and $\mathfrak{S}_{\mathfrak{d^\prime, j^\prime}}$ be two terminal nodes that have been joined in the joining step into the node $\mathfrak{S}_{\mathfrak{d, j}} \cup \mathfrak{S}_{\mathfrak{d^\prime, j^\prime}}$. Because the sets $\mathfrak{S}_{\mathfrak{d, j}}$ and $\mathfrak{S}_{\mathfrak{d^\prime, j^\prime}}$ are disjoint,
$$\PP^\star(\mathfrak{X} \in \mathfrak{S}_{\mathfrak{d, j}} \cup \mathfrak{S}_{\mathfrak{d^\prime, j^\prime}}) = \PP^\star(\mathfrak{X} \in \mathfrak{S}_{\mathfrak{d, j}}) + \PP^\star(\mathfrak{X} \in \mathfrak{S}_{\mathfrak{d^\prime, j^\prime}}).$$
Finally, each cluster is associated with a probability of the observation $\mathfrak{X}$ such that $\sum_{U \in \mathcal{U}}^{} \PP^\star(\mathfrak{X} \in U) = 1.$
We assign $\mathfrak{X}$ to the cluster with the largest probability. 

This procedure is run for each observation within $\mathcal{S}_{N_1}$ and results in a partition $\mathcal{V}$ of $\mathcal{S}_{N_1}$. A quasi-formal description of Algorithm \ref{alg:classify_one_obs}, closer to a high-level computer language, is provided in the Section \ref{app:algo} in the Supplementary Material.

\section{Empirical analysis}\label{sec:emp_analysis}

Using simulated data, in this section, we illustrate the behavior of our clustering algorithm and compare it with some competitors. A real data application on a vehicle trajectory dataset is also carried out.

Our \texttt{fCUBT} procedure is compared to diverse competitors in the literature that are both designed for univariate and multivariate functional data: \texttt{FunHDDC} (\cite{bouveyron_model-based_2011,schmutz_clustering_2020} and \cite{funHDDC_2019} for the \textbf{\textsf{R}} implementation) and \texttt{Funclust} (\cite{jacques_funclust_2013,jacques_model-based_2014} and \cite{Funclustering_2014} for the \textbf{\textsf{R}} implementation).   We use the model  $[a_{kj}b_kQ_kD_k]$ for \texttt{FunHDDC}. We also considered the model $[ab_kQ_kD_k]$, which yields similar results.  
Moreover, our approach competes with the methodology described in \cite{ieva_multivariate_2013}, which corresponds to the $k$-means algorithm with a suitable distance for functional data. In particular, the methodology in \cite{ieva_multivariate_2013} uses the following distances:
\begin{align*}
d_1(X, Y) &= \left(\sum_{p=1}^{P}\int_{\mathcal{T}_p} \left(X^{(p)}(t_p) - Y^{(p)}(t_p)\right)^2dt_p\right)^{1/2} \quad\text{and}\\ 
d_2(X, Y) &= \left(\sum_{p=1}^{P}\int_{\mathcal{T}_p} \left(\frac{dX^{(p)}(t_p)}{dt_p} - \frac{dY^{(p)}(t_p)}{dt_p}\right)^2dt_p\right)^{1/2},
\end{align*}
where $dX^{(p)}(t_p)/dt_p$ is the first derivative of $X^{(p)}(t_p)$. These two methods are denoted as $k$\texttt{-means-}$d_1$ and $k$\texttt{-means-}$d_2$ in the following. We use the implementation developed in a 2019 Madrid university report by A. Hernando Bernab\'e for both univariate and multivariate functional data. We also compare our algorithm with a Gaussian Mixture Model, fitted using an EM algorithm, on the coefficients of a functional principal components analysis on the dataset with a fixed number of components, quoted as \texttt{FPCA+GMM} in the following. Finally, we consider only the first step of the \texttt{fCUBT} algorithm, which is the growth of the tree, to point out the usefulness of the joining step. The method will be referred as \texttt{Growing} in the following.

We are greatly interested in the ability of the algorithms to retrieve the true number of clusters $K$. When the true labels are available, the estimated partitions are compared with the true partition using the Adjusted Rand Index (\ARI) (\cite{hubert_comparing_1985}, see Section \ref{app:numerical_illustation} in the Supplementary Material for a definition).

\subsection{Simulation experiments}\label{subsec:simu}

We consider three simulations scenarios with varying degrees of difficulty. Each experiment is repeated 500 times. 

\begin{scenario}\label{scenario:1}
In our first scenario, we consider that the random curves are observed without noise. The number of clusters is fixed at $K = 5$, $P=1$, $\setT_1=[0,1]$. An independent  sample of $N = 1000$ univariate curves is simulated according to the following model: for $t \in [0, 1]$:
$$\begin{array}{ll}
\text{Cluster 1:} & X(t) = \mu_1(t) + c_{11}\phi_1(t) + c_{12}\phi_2(t) + c_{13}\phi_3(t),\\
\text{Cluster 2:} & X(t) = \mu_1(t) + c_{21}\phi_1(t) + c_{22}\phi_2(t) + c_{23}\phi_3(t),\\
\text{Cluster 3:} & X(t) = \mu_2(t) + c_{11}\phi_1(t) + c_{12}\phi_2(t) + c_{13}\phi_3(t),\\
\text{Cluster 4:} & X(t) = \mu_2(t) + c_{21}\phi_1(t) + c_{22}\phi_2(t) + c_{23}\phi_3(t),\\
\text{Cluster 5:} & X(t) = \mu_2(t) + c_{21}\phi_1(t) + c_{22}\phi_2(t) + c_{23}\phi_3(t) - 15t,\\
\end{array}$$
where $\phi_k$'s are the eigenfunctions of the Wiener process which are defined by $\phi_k(t) = \sqrt{2}\sin\{(k - 0.5)\pi t\}, k = 1, 2, 3,$
and the mean functions by $\mu_1(t) = 20 / \{1 + \exp(-t)\}$ and $\mu_2(t) = -25 / \{1 + \exp(-t)\}$.
The $c_{ij}$'s are random normal variables defined by
\begin{align}\label{c_law_s1}
c_{11} &\sim \mathcal{N}(0, 16), & c_{12} &\sim \mathcal{N}\left(0, 64 / 9\right), & c_{13} &\sim \mathcal{N}\left(0, 16 / 9\right), \\
c_{21} &\sim \mathcal{N}(0, 1), & c_{22} &\sim \mathcal{N}\left(0, 4 / 9\right), & c_{23} &\sim \mathcal{N}\left(0, 1 / 9\right). 
\end{align}
The mixing proportions are equal, and the curves are observed on $101$ equidistant points. As shown in Figure \ref{fig:scenario_1}, the five clusters cannot be retrieved using only the mean curve per cluster: cluster 1 (blue) and 2 (orange) share the same mean function $\mu_1$ and similarly cluster 3 (green) and 4 (pink) share the same mean function $\mu_2$. As a consequence, clustering algorithms based on distances to the mean of the clusters, such as $k$-means, are not expected to perform well in this case. For \texttt{FunHDDC} and \texttt{Funclust}, the functional form of the data is reconstructed using a cubic B-spline basis, smoothing with $25$ basis functions.

\begin{figure}
    \centering
    \includegraphics[scale=0.5]{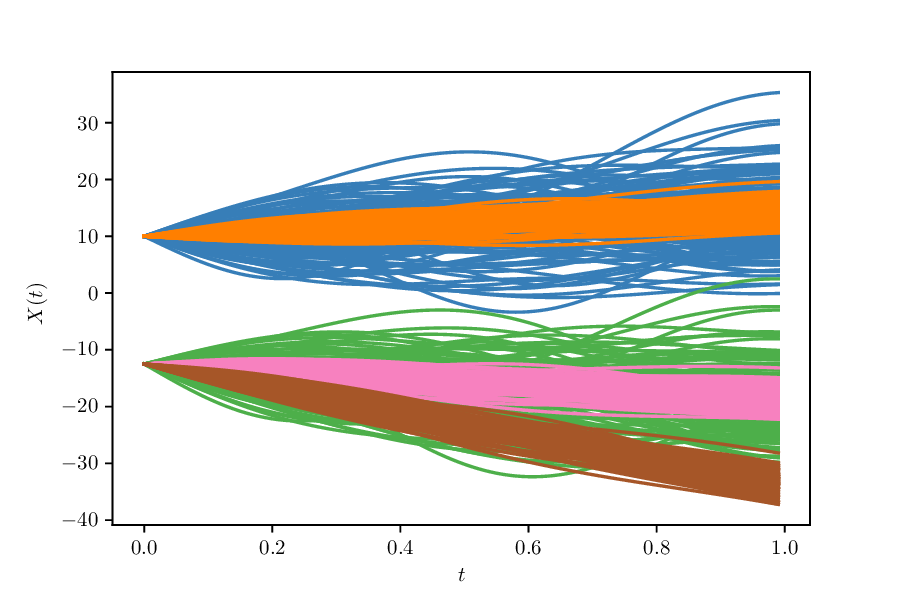}
    \caption{Simulated data for Scenario 1.}
    \label{fig:scenario_1}
\end{figure}
\end{scenario}

\begin{scenario}\label{scenario:2}
The second simulation is a modification of the data simulation process of \cite[scenario C]{schmutz_clustering_2020}. 
Here, we consider that the measurements of the random curves are noisy. Thus, for this scenario, the number of clusters is fixed at $K = 5$, $P=2$, $\setT_1=\setT_2 =[0,1]$. An independent sample of $N = 1000$ bivariate curves is simulated according to the following model~: for $t_1,t_2  \in [0, 1]$,
$$\begin{array}{llcl}
\text{Cluster 1:} & X^{(1)}(t_1) = h_1(t_1) + b_{0.9}(t_1), & & X^{(2)}(t_2) = h_3(t_2) + 1.5 \times b_{0.8}(t_2), \\
\text{Cluster 2:} & X^{(1)}(t_1) = h_2(t_1) + b_{0.9}(t_1), & & X^{(2)}(t_2) = h_3(t_2) + 0.8 \times b_{0.8}(t_2), \\
\text{Cluster 3:} & X^{(1)}(t_1) = h_1(t_1) + b_{0.9}(t_1), & & X^{(2)}(t_2) = h_3(t_2) + 0.2 \times b_{0.8}(t_2), \\
\text{Cluster 4:} & X^{(1)}(t_1) = h_2(t_1) + 0.1 \times b_{0.9}(t_1), & & X^{(2)}(t_2) = h_2(t_2) + 0.2 \times b_{0.8}(t_2), \\
\text{Cluster 5:} & X^{(1)}(t_1) = h_3(t_1) + b_{0.9}(t_1), & & X^{(2)}(t_2) = h_1(t_2) + 0.2 \times b_{0.8}(t_2).\\
\end{array}$$
The functions $h$ are defined, by 
$h_1(t) = \left(6 - \lvert 20t - 6\rvert\right)_+ / 4$, $h_2(t) = \left(6 - \lvert 20t - 14\rvert\right)_+ / 4$ and $h_3(t) = \left(6 - \lvert 20t - 10\rvert\right)_+ / 4$, for $t \in [0, 1]$. (Here, $(\,\cdot\,)_+$ denotes the positive part of the expression between the brackets.) The functions $b_H$ are defined, for $t \in [0, 1]$, by $b_H(t) = (1 + t)^{-H}B_H(1 + t)$ where $B_H(\cdot)$ is a fractional Brownian motion with Hurst parameter $H$. 
The mixing proportions are set to be equal.

The data to which we apply the clustering are obtained as in \eqref{data_real}.  Each component curve is observed at $101$ equidistant points over $[0,1]$. The bivariate error vectors have zero-mean Gaussian independent components with variance 1/2. Figure \ref{fig:scenario_2} presents  the smoothed curves from the simulated data. 
Smoothing was done using the methodology of \cite{golovkine_learning_2020}.  
As pointed out by \cite{schmutz_clustering_2020}, the different clusters cannot be identified using only one variable: cluster 1 (blue) is similar to cluster 3 (green) for variable $X^{(1)}(t)$ and in like manner, cluster 1 (blue) is like cluster 2 (orange) and cluster 3 (green) for variable $X^{(2)}(t)$. The brown and pink groups might be considered as ``noise'' clusters that aim to make  discrimination between the other groups harder. Hence, clustering methods that are specialized for univariate data, should fail to retrieve true membership using only $X^{(1)}(t)$ or $X^{(2)}(t)$. For \texttt{FunHDDC} and \texttt{Funclust}, the functional form of the data is reconstructed using a cubic B-spline basis, smoothing with $25$ basis functions.

\begin{figure}
    \begin{subfigure}{.5\textwidth}
    \centering
    \includegraphics[scale=0.5]{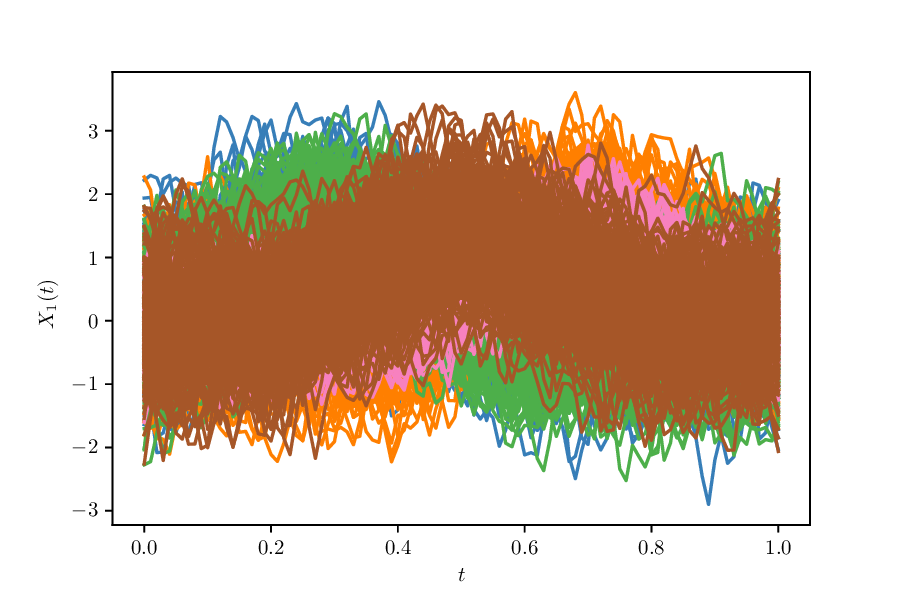}
    \end{subfigure}
    \begin{subfigure}{.5\textwidth}
    \centering
    \includegraphics[scale=0.5]{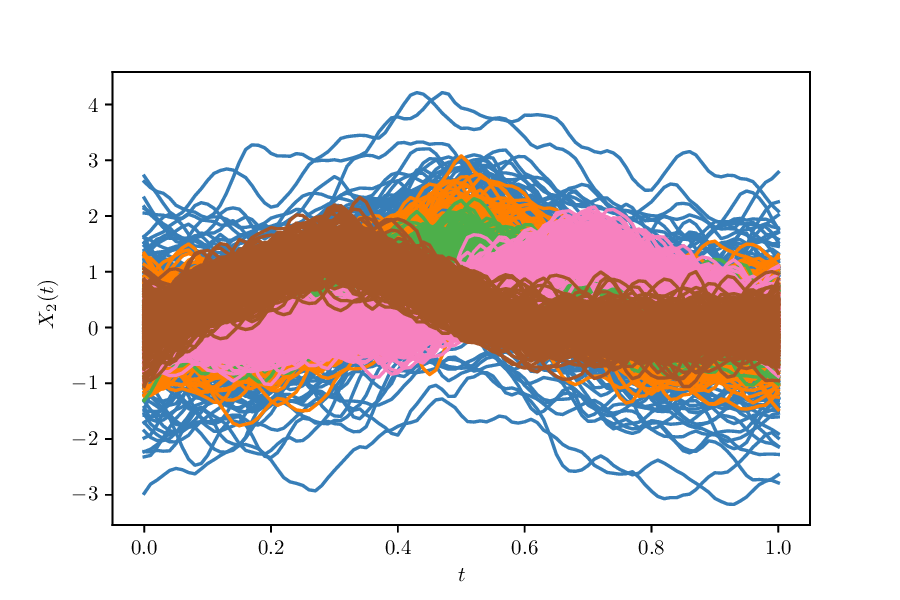}
    \end{subfigure}
    \caption{Simulated data for Scenario 2.}
    \label{fig:scenario_2}
\end{figure}
\end{scenario}

\begin{scenario}\label{scenario:3}
The last simulation is the same as the second one, except we add some correlation between the components. So, for each $n \in \{1, \dots, N\}$, we observe a realization of the vector $X = \left(X^{(1)} + \alpha X^{(2)}, X^{(2)}\right)^\top$, where $\alpha = 0.4$. Figure \ref{fig:scenario_3} presents the smoothed version, using the methodology of \cite{golovkine_learning_2020}, of the simulated data. For \texttt{FunHDDC} and \texttt{Funclust}, the functional form of the data is reconstructed using a cubic B-spline basis, smoothing with $25$ basis functions.

\begin{figure}
    \begin{subfigure}{.5\textwidth}
    \centering
    \includegraphics[scale=0.5]{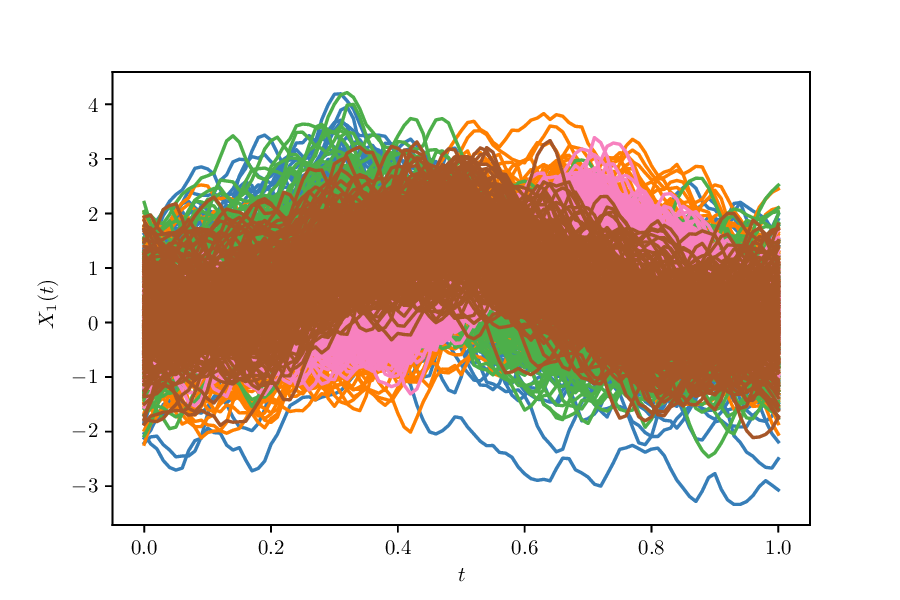}
    \end{subfigure}
    \begin{subfigure}{.5\textwidth}
    \centering
    \includegraphics[scale=0.5]{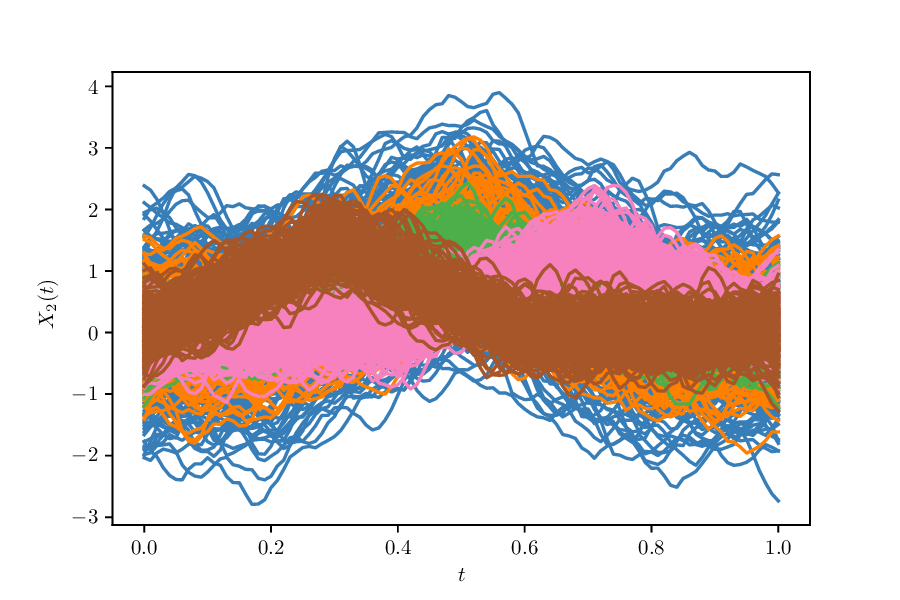}
    \end{subfigure}
    \caption{Simulated data for Scenario 3.}
    \label{fig:scenario_3}
\end{figure}

\end{scenario}

\subsubsection{Model selection}

We investigate the selection of the number of clusters for each of methods on each of the simulations. The way to return the number of clusters in a dataset depends on the algorithm. Thus, the selection for the model \texttt{fCUBT} and \texttt{Growing} is based on the BIC. Similarly, the BIC is also used for the \texttt{FunHDDC} algorithm. For all the other methods, we test all the models between $K = 1$ and $K = 8$, and return the model that maximizes the ARI criteria as the selected model. The simulation settings have been repeated $500$ times and the model $[a_{kj}b_kQ_kD_k]$ is used for the \texttt{FunHDDC} algorithm.

Table \ref{tab:n_clusters} summarizes the results of the $500$ simulations for each scenario. We remark that the \texttt{fCUBT} algorithm performs well in retrieving the right number of clusters in all the scenarios, being the first or second method in terms of retrieving percentage. Quite surprisingly, the $k$\texttt{-means-}$d_2$ algorithm performs very well for the second and third scenarios. It indicates that the distances between the derivatives of the curves are much more informative than the distance between the original ones. The accuracy of the selection of the number of clusters in competitors, designed for functional data, \texttt{FunHDDC} and \texttt{Funclust}, is very poor. This result has also been pointed out in \cite{zambom_selection_2019} where the simulated data are much  simpler. Finally, the results on \texttt{Growing} show the usefulness of the joining step. Thus, at the end of the growing step, we may have a large number clusters (even greater than $10$) but with very few data in some of them, and so the joining step allows us to get rid of them, and thus have more relevant clusters.

\begin{table}
\centering
\begin{tabular}{crrrrrrrrrrr}
 & Method & \multicolumn{10}{c}{Number of clusters $K$} \\
    &   & 1 & 2 & 3 & 4 & 5 & 6 & 7 & 8 & 9 & 10+ \\
\midrule
&\texttt{fCUBT} & - & - &  -  & - & 98 & 2 & - &  - &  - & - \\
&\texttt{Growing} & - & - & - & - & 59 & 23 &  8 & 3 & 3 & 4 \\
&\texttt{FPCA+GMM} & - & - & - & 1 & 53 & 32 & 11 & 3 & - & - \\
Scenario 1 &\texttt{FunHDDC} & - & 35 & 46 & 16 & 2 & 1 & - & - & - & - \\
&\texttt{Funclust} & - & 44 & 44 & 11 & 1 & - & - & - &  - & - \\
&$k$\texttt{-means-}$d_1$ & - & - & 15 & 15 & 60 & 2 & 7 & 7 & - & - \\
&$k$\texttt{-means-}$d_2$ & - & - & - & - & 5 & 29 & 36 & 30 & - & - \\
\midrule
&\texttt{fCUBT} & - & - & - & - & 70 & 18 & 10 & 2 & - & - \\
&\texttt{Growing} & - & - & - & - & 52 & 18 & 12 & 9 & 3 & 6 \\
&\texttt{FPCA+GMM} & - & - & - & - & 27 & 45 & 25 & 3 & - & - \\
Scenario 2 &\texttt{FunHDDC} & - & 1 & 1 & - & 2 & 4 & 5 & 4 & 25 & 59 \\
&\texttt{Funclust} & - & 28 & 17 & 15 & 16 & 14 & 7 & 3 & - & - \\
&$k$\texttt{-means-}$d_1$ & - & - & - & 2 & 5 & 8 & 18 & 67 & - & - \\
&$k$\texttt{-means-}$d_2$ & - & - & 5 & 12 & 82 & 1 & - & - & - & - \\
\midrule
&\texttt{fCUBT} & - & - & - & - & 67 & 24 & 7 & 2 & - & - \\
&\texttt{Growing} & - & - & - & - & 60 & 18 & 8 & 6 & 3 & 5 \\
&\texttt{FPCA+GMM} & - & - & - & - & 41 & 40 & 16 & 3 & - & - \\
Scenario 3 &\texttt{FunHDDC} & - & 1 & - & 3 & 7 & 11 & 11 & 15 & 20 & 32\\
&\texttt{Funclust} & - & 7 & 18 & 19 & 20 & 20 & 13 & 3 & - &  - \\
&$k$\texttt{-means-}$d_1$ & - & - & - & - & 3 & 14 & 21 & 62 & - & - \\
&$k$\texttt{-means-}$d_2$ & - & 1 & 1 &  9 & 87 & 1 & - & 1 & - & - \\
\end{tabular}
\caption{Number of clusters selected for each model, expressed  as a  percentage   over $500$ simulations}
\label{tab:n_clusters}
\end{table}

\subsubsection{Benchmark with existing methods}

Our algorithm is compared to competitors in the literature with respect to the \ARI{} criterion on the three scenario settings. All the competitors are applied for $K = 1$ to $K = 8$ groups for each of the scenarios and we return the best \ARI{} found regardless of the number of clusters. 

Figure \ref{fig:simu_ARI} presents clustering results for all the tested models for the \ARI{} criterion. We see that our algorithm performs well in all the scenarios. On the contrary, the \texttt{FunHDDC} and \texttt{Funclust} competitors do not perform well on this simulated data. Both $k$\texttt{-means-}$d_1$ and $k$\texttt{-means-}$d_2$ demonstrate acceptable \ARI{} although not as good as \texttt{fCUBT}. The \texttt{FPCA+GMM} algorithm has similar results as \texttt{fCUBT} in terms of \ARI{} because \ARI{} is not penalized when the number of clusters is not the true one. So, as long as the clusters are not mixed, the \ARI{} will be good, even if a large cluster is split into multiple small ones. The same phenomenon appears in the case of \texttt{Growing} compare to \texttt{fCUBT}.

\begin{figure}
     \centering
     \begin{subfigure}[b]{0.35\textwidth}
         \centering
         \includegraphics[scale=0.35]{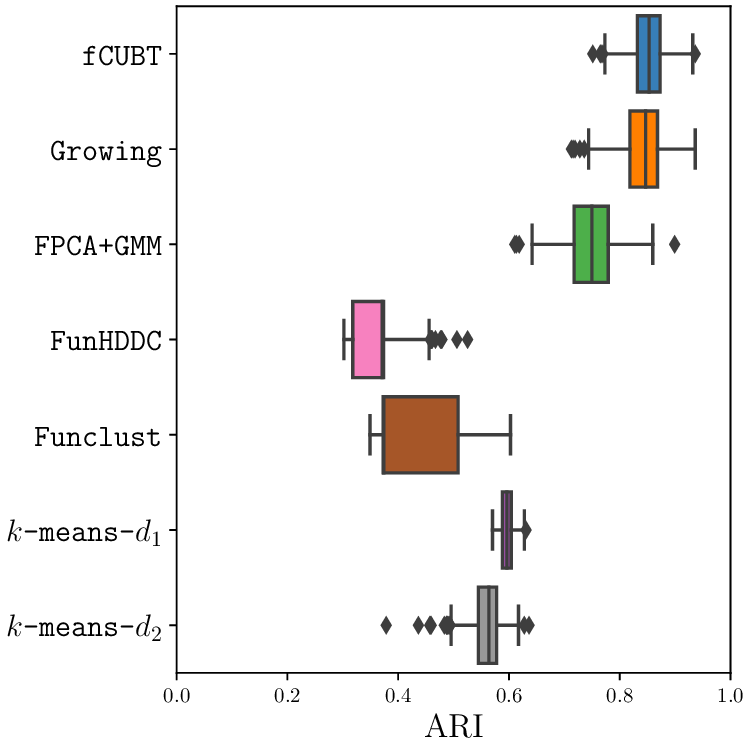}
         \caption{Scenario 1}
         \label{fig:ARI_scenario_1}
     \end{subfigure}
     \begin{subfigure}[b]{0.3\textwidth}
         \centering
         \includegraphics[scale=0.35]{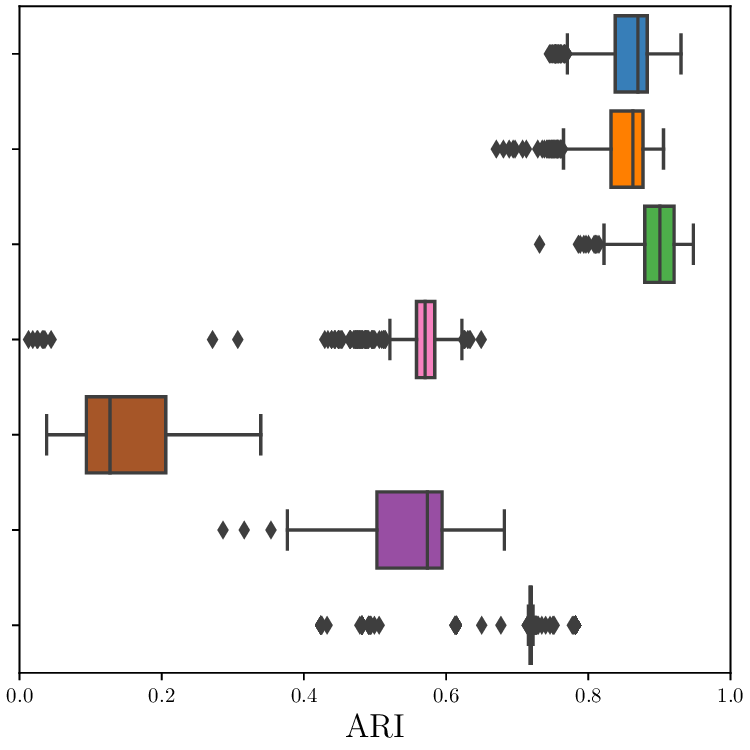}
         \caption{Scenario 2}
         \label{fig:ARI_scenario_2}
     \end{subfigure}
     \begin{subfigure}[b]{0.3\textwidth}
         \centering
         \includegraphics[scale=0.35]{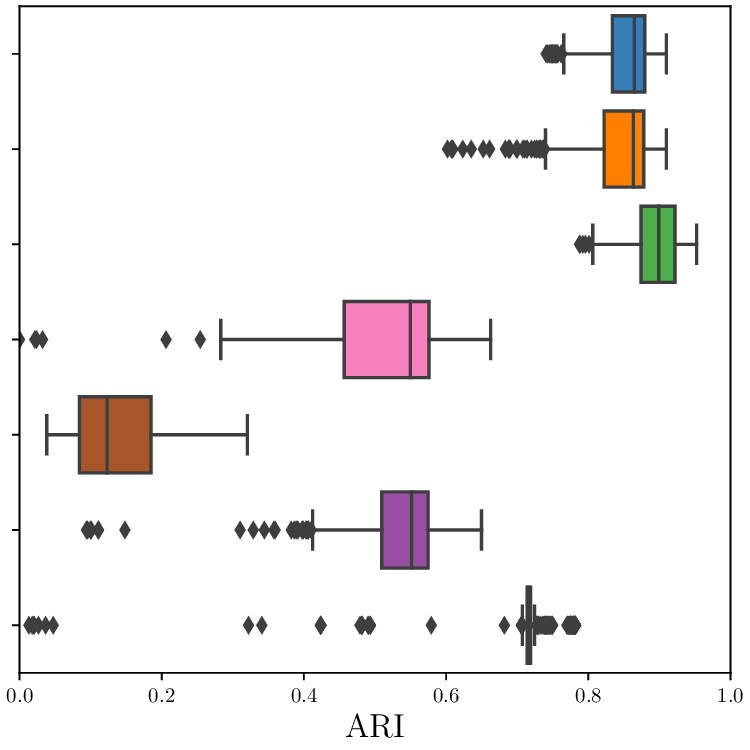}
         \caption{Scenario 3}
         \label{fig:ARI_scenario_3}
     \end{subfigure}
     \caption{Estimation of \ARI{} for all tested models on $500$ simulations.}
    \label{fig:simu_ARI}
\end{figure}

\subsubsection{Comments on the classification of new set of curves}

We now analyze the performance of the algorithm for the classification of a new set of curves. For each of the simulated scenarios, we apply the following process. We learn a tree $\mathfrak{T}$ as well as the partition $\mathcal{U}$ from the learning set $\mathcal{S}_{N_0}$. Different sizes of learning sets  are considered, $N_0 = 200$, $500$ or $1000$. We generate a new set of data, $\mathcal{S}_{N_1}$, referred to as the online dataset, of size $N_1 = 1000$. As the data are simulated, we know the true labels of each observations within $\mathcal{S}_{N_1}$. We denote the true partition of $\mathcal{S}_{N_1}$ by $\mathcal{V}$. We then classify new observations from the online set $\mathcal{S}_{N_1}$ and denote the obtained partition by $\mathcal{V}^\prime$. Partitions are compared using $\ARI(\mathcal{V}, \mathcal{V}^\prime)$. 

The simulations are performed $500$ times, and the results are plotted in Figure \ref{fig:simu_prediction}. The three scenarios present similar patterns. Thus, when $N_0 = 200$, the partition $\mathcal{U}$ obtained using the $\mathtt{fCUBT}$ algorithm is not able to capture all elements within each cluster. In this case, the \ARI{} is less than $0.8$. When $N_0 = 500$, the partition $\mathcal{U}$ is now sufficiently accurate to represents the clusters (\ARI{} $> 0.8$). However, the stability of the clusters is not guaranteed, given the large variance in the estimation of the \ARI{}. Finally, when $N_0 = 1000$, we have both accurate partitioning $\mathcal{U}$ (\ARI{} $> 0.8$) and stable clusters (low variance).
 
\begin{figure}
     \centering
     \begin{subfigure}[b]{0.3\textwidth}
         \centering
         \includegraphics[scale=0.35]{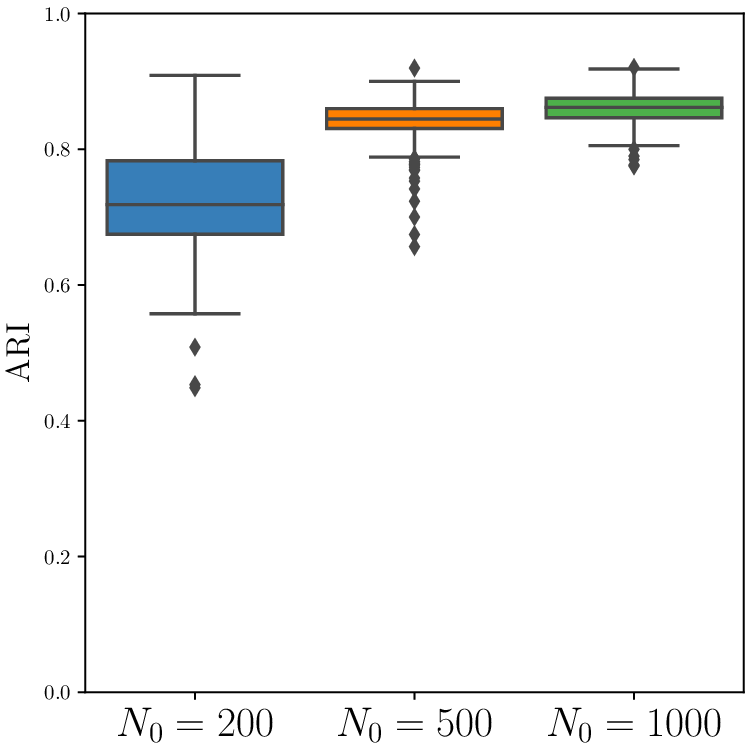}
         \caption{Scenario 1}
         \label{fig:prediction_scenario_1}
     \end{subfigure}
     \begin{subfigure}[b]{0.3\textwidth}
         \centering
         \includegraphics[scale=0.35]{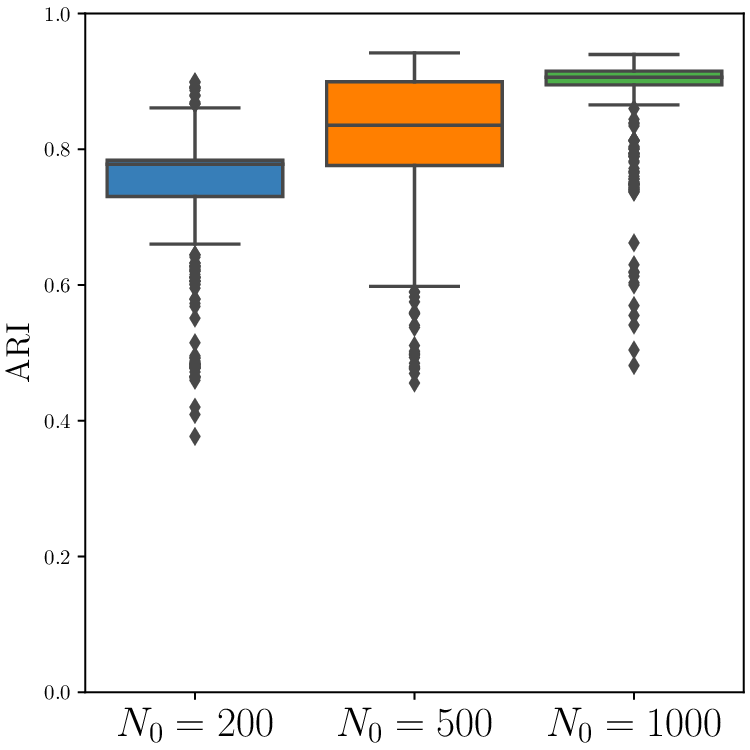}
         \caption{Scenario 2}
         \label{fig:prediction_scenario_2}
     \end{subfigure}
     \begin{subfigure}[b]{0.3\textwidth}
         \centering
         \includegraphics[scale=0.35]{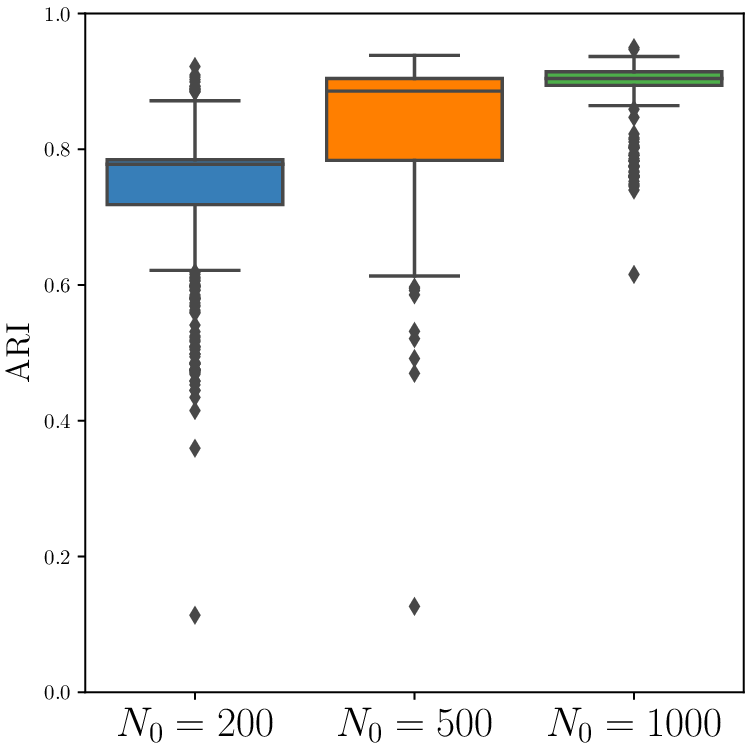}
         \caption{Scenario 3}
         \label{fig:prediction_scenario_3}
     \end{subfigure}
     \caption{ Estimation of \ARI{} with respect to the size of the learning dataset when the tree is used as a supervised classifier.}
    \label{fig:simu_prediction}
\end{figure}

\subsubsection{A comparison with supervised methods}\label{sec:compa_sup}

We also compare our functional clustering procedure with two supervised approaches available in \cite{scikit-learn}, that are a Gaussian Process Classifier (\texttt{GPC}) and a Random Forest Classifier (\texttt{Random Forest}). See, for instance, \cite{3569} for the description of \texttt{GPC}. 
We first perform an MFPCA to extract features that explain $99\%$ of the variance for each component $X^{(p)}$ within the data. We then fit  \texttt{GPC} and  \texttt{Random Forest} to the extracted features. 

For each  scenario, we generate samples of $N = 1000$ curves. 
We then randomly sample $2/3$ of the $N$ curves to build the training subset and the remaining ones are gathered in the test subset. The supervised models are trained on the training subset, using the true value of $K$,  and we predict the outcome on the test subset. The \ARI{} is finally computed using the true labels and the prediction. 
For the \texttt{fCUBT} method, we consider the training subset to learn the clusters. We next predict the outcome on the test subset considering it as a set of new observations, and using the procedure proposed in Section \ref{class_new_obs}.  For comparison purposes, the \ARI{} for \texttt{fCUBT} is computed using only the replications where our approach detected the correct number of clusters $K$. 
We repeated the experiments $500$ times, and the results are plotted in Figure \ref{fig:simu_comparison_review}. 
We remark that our unsupervised method is as good as the supervised ones when the true number of clusters is correctly estimated. The good performance of \texttt{fCUBT} compared to \texttt{Random Forest} could be explained by the model-based construction of our approach, while the random forest is nonparametric. The good performance of \texttt{fCUBT}
compared to \texttt{GPC}, especially in Scenario 1, is somehow more surprising. 
An additional simulation experiment where  \texttt{fCUBT} is used directly on the test subset, is reported in the Supplementary Material. Our approach performs well in that case also.

\begin{figure}
     \centering
     \begin{subfigure}[b]{0.3\textwidth}
         \centering
         \includegraphics[scale=0.35]{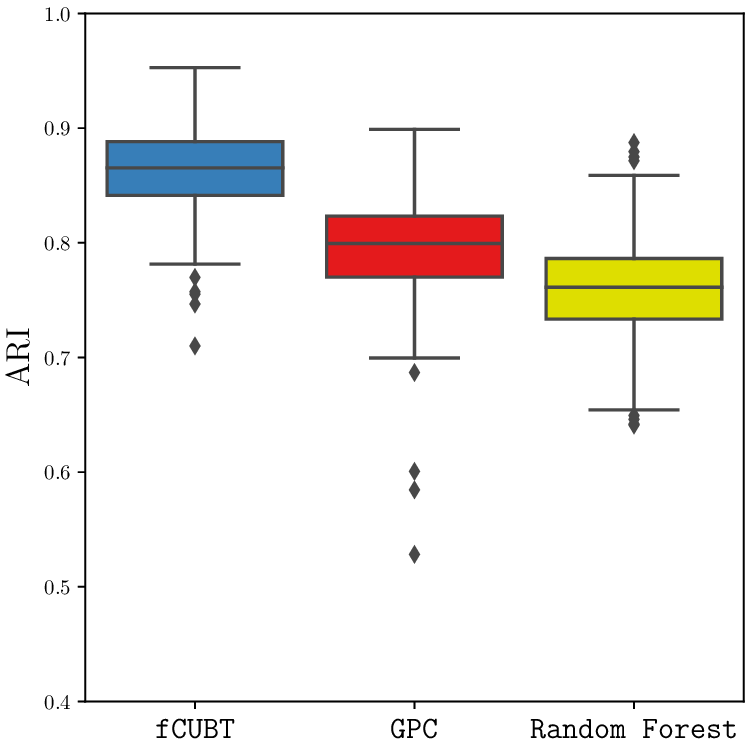}
         \caption{Scenario 1 $(491 / 500)$}
         \label{fig:comparison_scenario_1}
     \end{subfigure}
     \begin{subfigure}[b]{0.3\textwidth}
         \centering
         \includegraphics[scale=0.35]{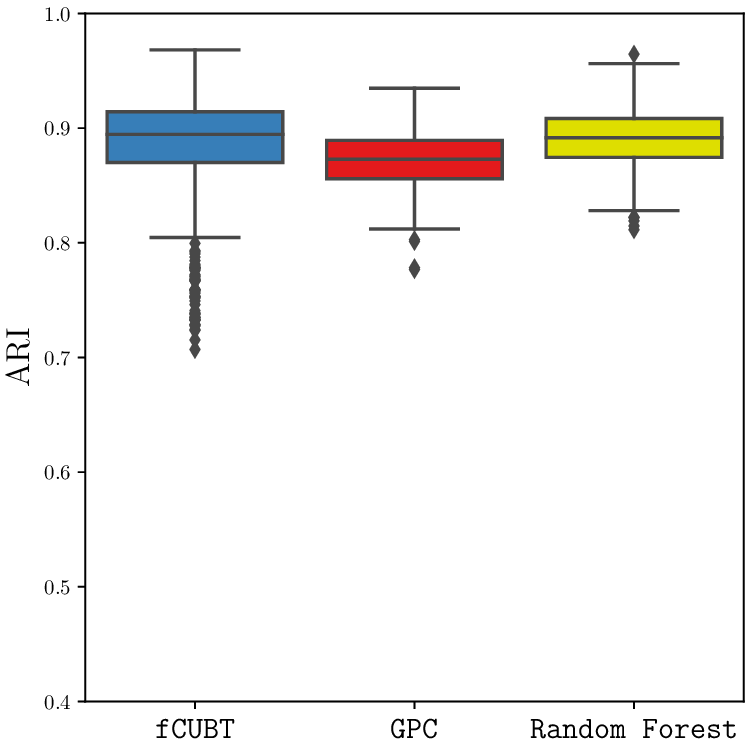}
         \caption{Scenario 2 $(436 / 500)$}
         \label{fig:comparison_scenario_2}
     \end{subfigure}
     \begin{subfigure}[b]{0.3\textwidth}
         \centering
         \includegraphics[scale=0.35]{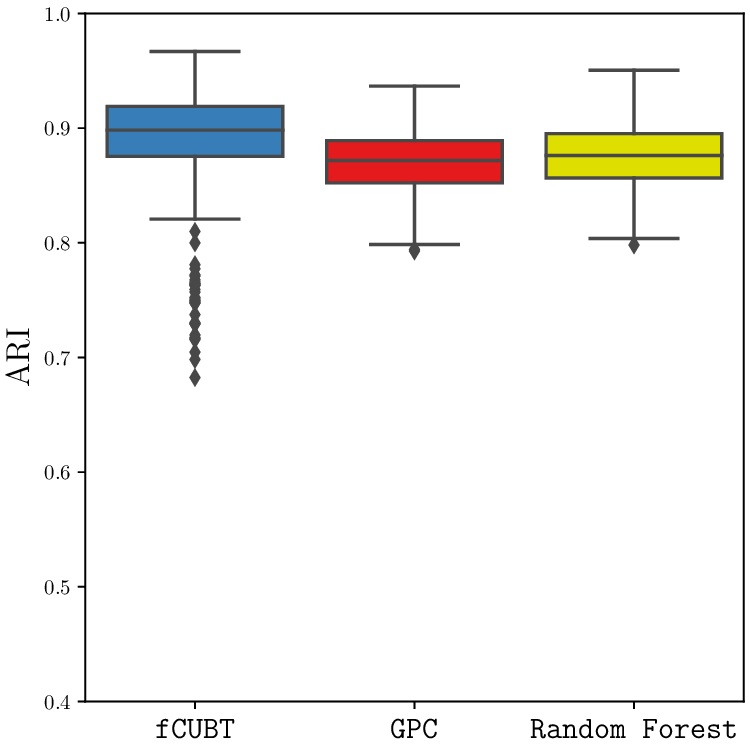}
         \caption{Scenario 3 $(431 / 500)$}
         \label{fig:comparison_scenario_3}
     \end{subfigure}
     \caption{Estimation of \ARI{} for the comparison with supervised models using $500$ replications. The number of replications among the 500 experiments where  \texttt{fCUBT} recovers the correct number of clusters $K$, is given between parenthesis}
    \label{fig:simu_comparison_review}
\end{figure}

\subsection{Real data analysis: the rounD dataset}

In this section, our method is applied to a part of the rounD dataset \cite{rounDdataset}, which are ``naturalistic road user trajectories recorded at German roundabouts''. For our illustration, we consider a subset of the rounD dataset, that corresponds to one particular roundabout  with four exits. Details are provided in the Supplementary Material. 

 The dataset we use contains $18$ minutes  of trajectories for road users at $7$a.m. To describe the motion of the vehicles, we use six coordinates $X^{(p)}$ given by the position, the speed and the acceleration, each of them decomposed into a two-dimensional coordinate. 
The speed limit is $50$\si{\kilo\meter\per\hour}.  In total, the dataset contains trajectories, velocities and accelerations for $N_0 = 348$ individual road users that passed through this roundabout during this period, recorded every $0.04$s. The number of measurements for each curve varies from $131$ to $1265$. We rescale the measurement times for each of the $348$ curves such that the first measurement corresponds to $t = 0$ and the last one to $t = 1$. Figure \ref{fig:sample} presents a random sample of five observations extracted from the data. In order to have comparable results, we remove the pedestrians and the bicycles from the data. Moreover, curves with less than $200$ or more than $800$ measurements are also removed. The curves with less than $200$ points are probably the road users that are present when the recording starts (or ends), and thus, their trajectory will not be complete. The curves with more than $800$ are likely to be trajectories with inconsistency. Finally, there are $311$ remaining observations in the dataset.  We aim to provide a clustering and give some physical interpretation of the clusters. Thus, our clustering procedure \texttt{fCUBT} is run on the cleaned data.  We chose $J^{(p)} = 1$ for both the growing and joining step,  and we set $K_{max} = 3$ and $\mathtt{minsize} = 20$. It returns $23$ groups. Figure \ref{fig:cluster_4} presents an example of a clusters we obtain.  The trajectories with different entries and exits in the roundabout are well split into different clusters. Several clusters correspond to the same entries and exists and they are distinguished by different velocity and acceleration profiles. In particular, the algorithm differentiates the vehicles which stopped before entering the roundabout  from those which did not (see Figure \ref{fig:cluster_11}). Vehicles with atypical trajectories, such as those making a complete additional loop before taking the exit are also weel separated. The Gaussian assumption used in our model-based clustering seems reasonable in this application, according to the results from several normality tests we performed following \cite{norm_test}. The details  are given in the Supplementary Material.

\begin{figure}
    \centering
    \begin{subfigure}[b]{\textwidth}
        \centering
        \includegraphics[scale=0.4]{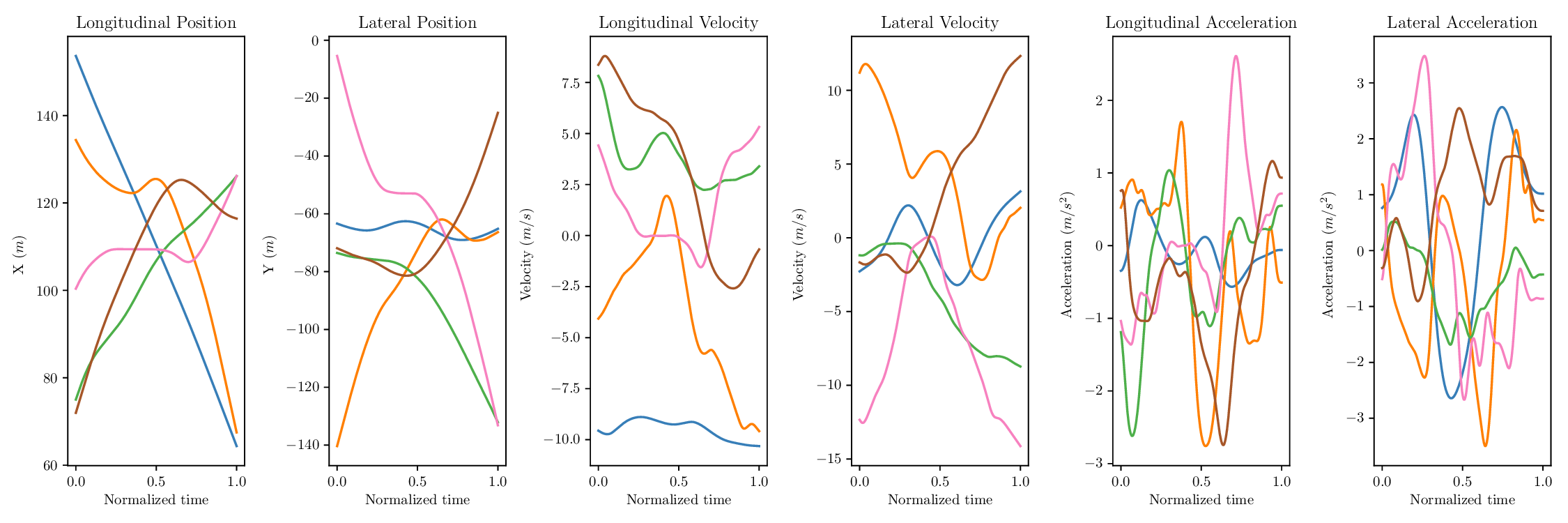}
        \caption{}
        \label{fig:sample_roundabout}
    \end{subfigure}
    \vskip\baselineskip
    \begin{subfigure}[b]{0.49\textwidth}
        \centering
        \includegraphics[scale=0.4]{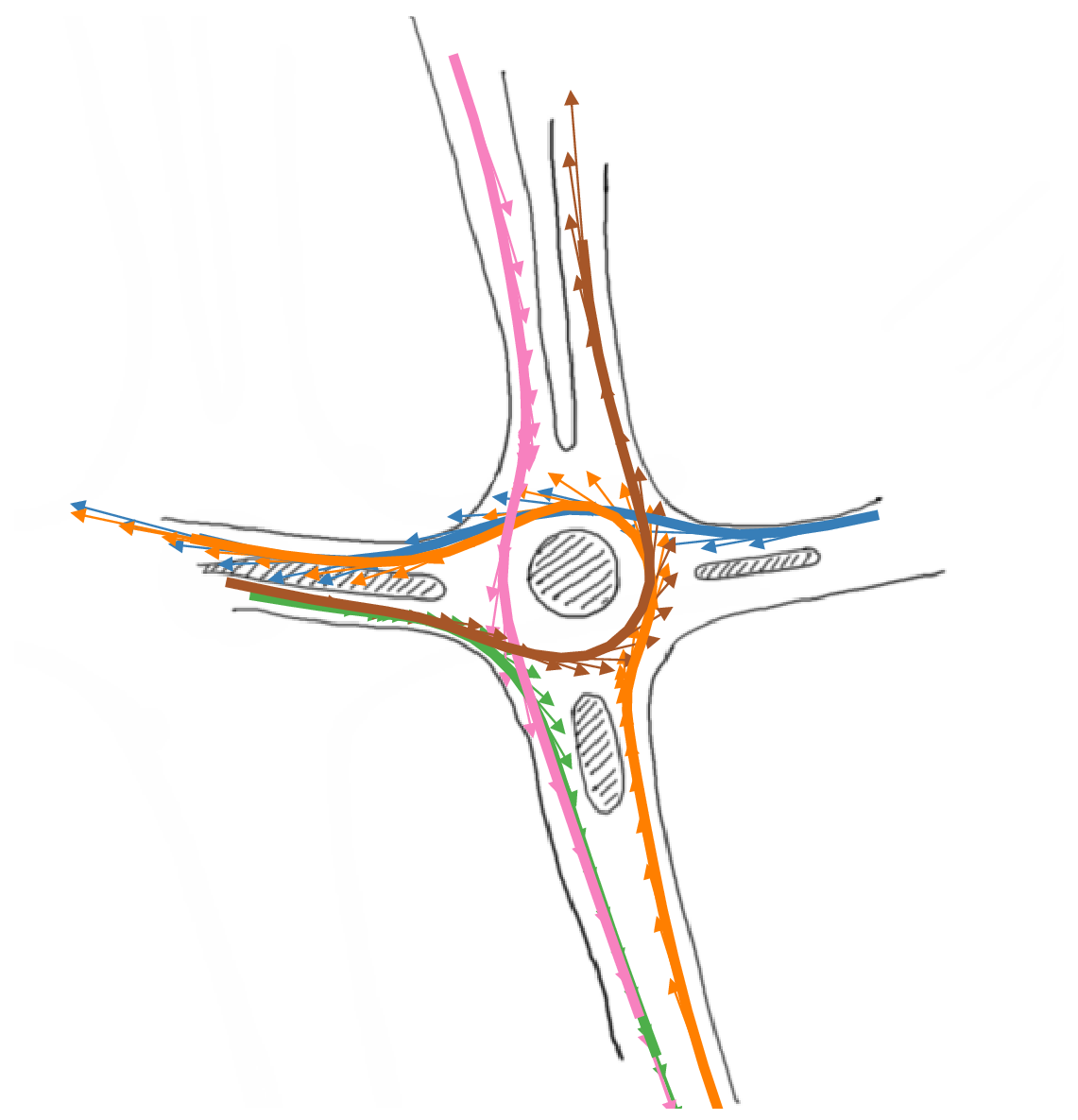}
        \caption{}
        \label{fig:sample_roundabout_vel}
    \end{subfigure}
    \hfill
    \begin{subfigure}[b]{0.49\textwidth}
    \centering
        \includegraphics[scale=0.4]{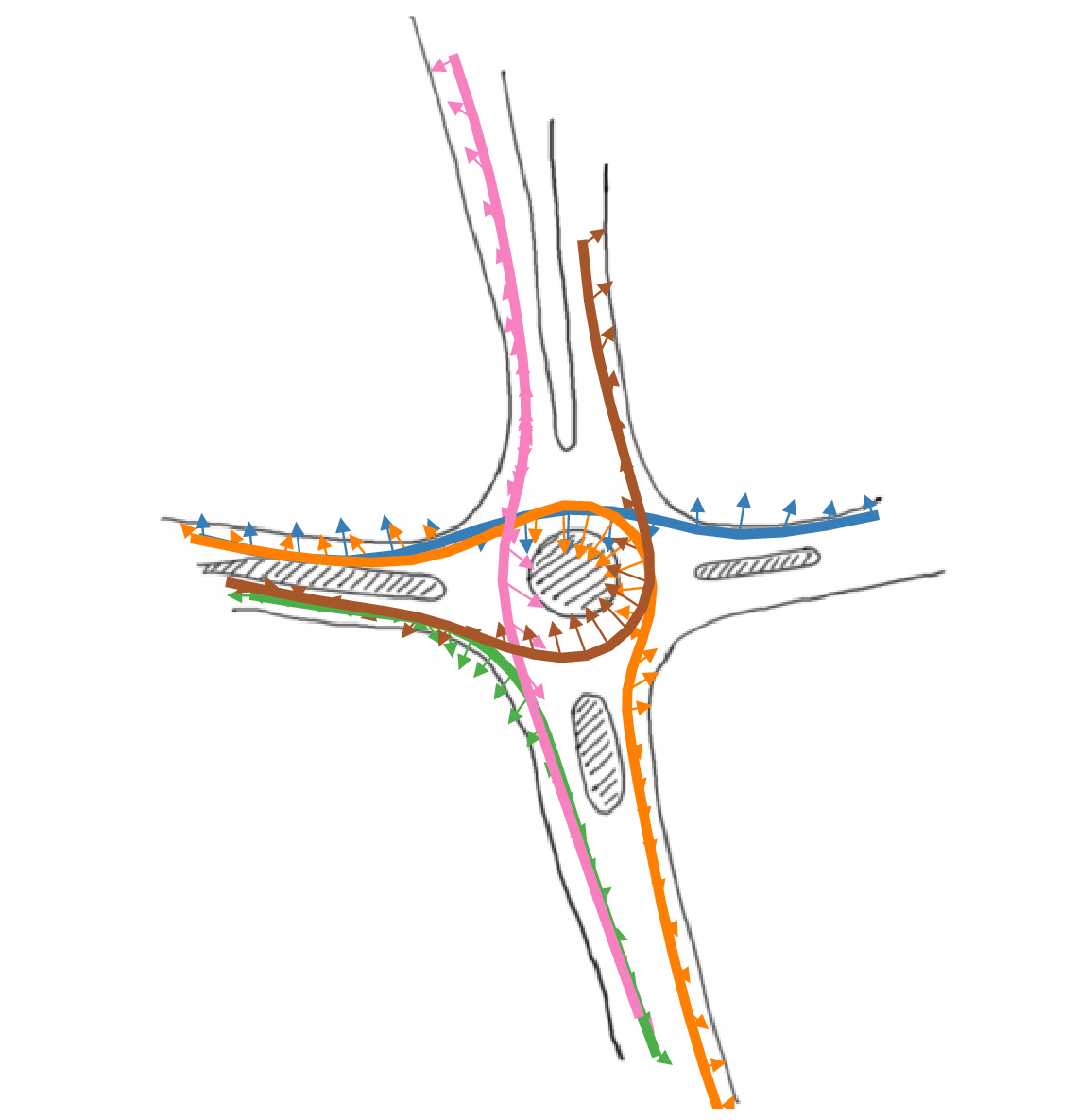}
        \caption{}
        \label{fig:sample_roundabout_acc}
    \end{subfigure}
    \caption{rounD dataset illustration -- a sample of five trajectories (a). The trajectories in the roundabout reference frame where the arrows represent the magnitude and direction of the velocity (b) and acceleration (c).}
    \label{fig:sample}
\end{figure}

\begin{figure}
    \centering
    \begin{subfigure}[b]{0.49\textwidth}
        \centering
        \includegraphics[scale=0.30]{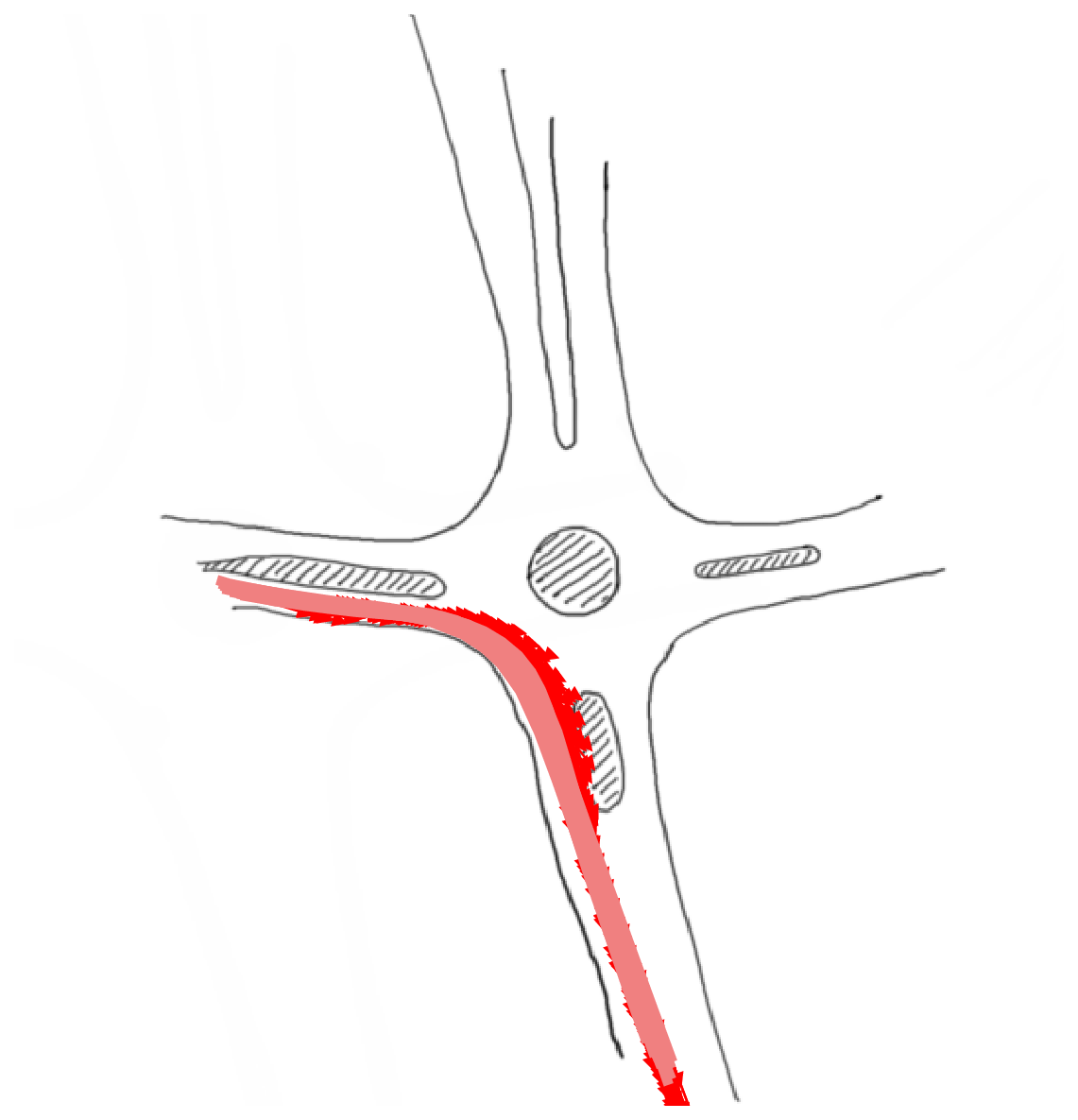}
        \caption{}
        \label{fig:cluster_vel}
    \end{subfigure}
    \hfill
    \begin{subfigure}[b]{0.49\textwidth}
    \centering
        \includegraphics[scale=0.30]{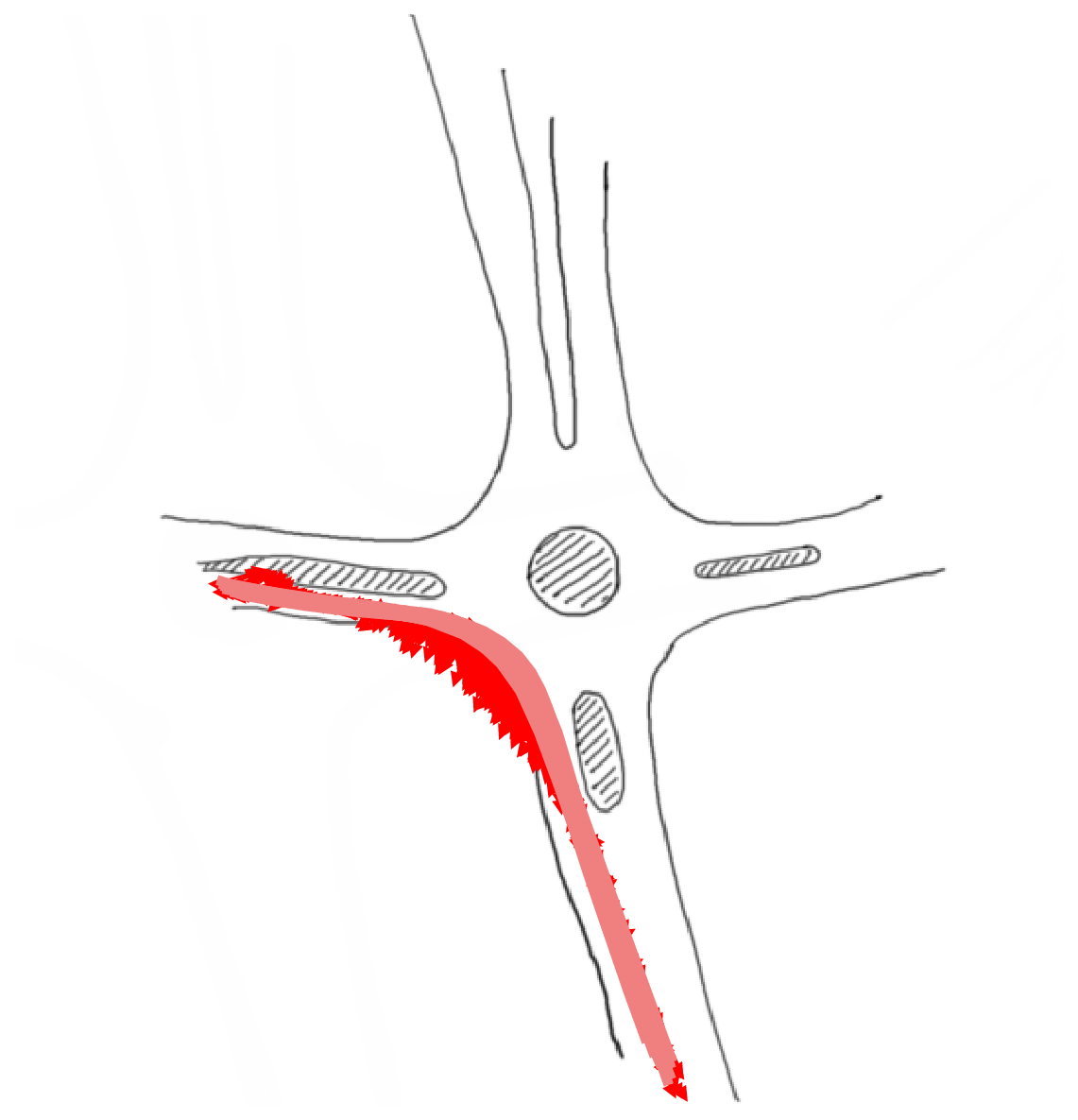}
        \caption{}
        \label{fig:cluster_acc}
    \end{subfigure}
    \caption{rounD dataset: An example of a cluster found using the $\mathtt{fCUBT}$ method. The trajectories are plotted in the roundabout reference frame. Each red curve represents a trajectory of an observation and the red arrows represent the magnitude and direction of the velocity (a) and acceleration (b).}
    \label{fig:cluster_4}
\end{figure}

\begin{figure}
    \centering
    \begin{subfigure}[b]{0.49\textwidth}
        \centering
        \includegraphics[scale=0.35]{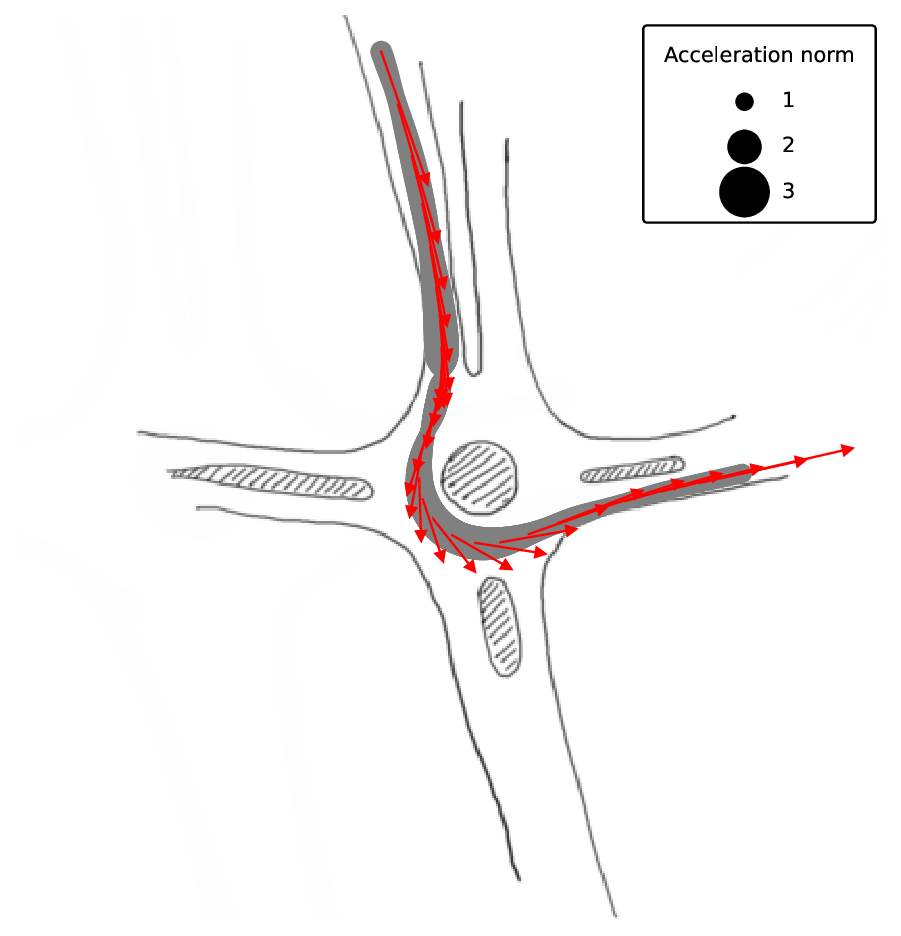}
        \caption{}
        \label{fig:cluster_vel}
    \end{subfigure}
    \hfill
    \begin{subfigure}[b]{0.49\textwidth}
    \centering
        \includegraphics[scale=0.35]{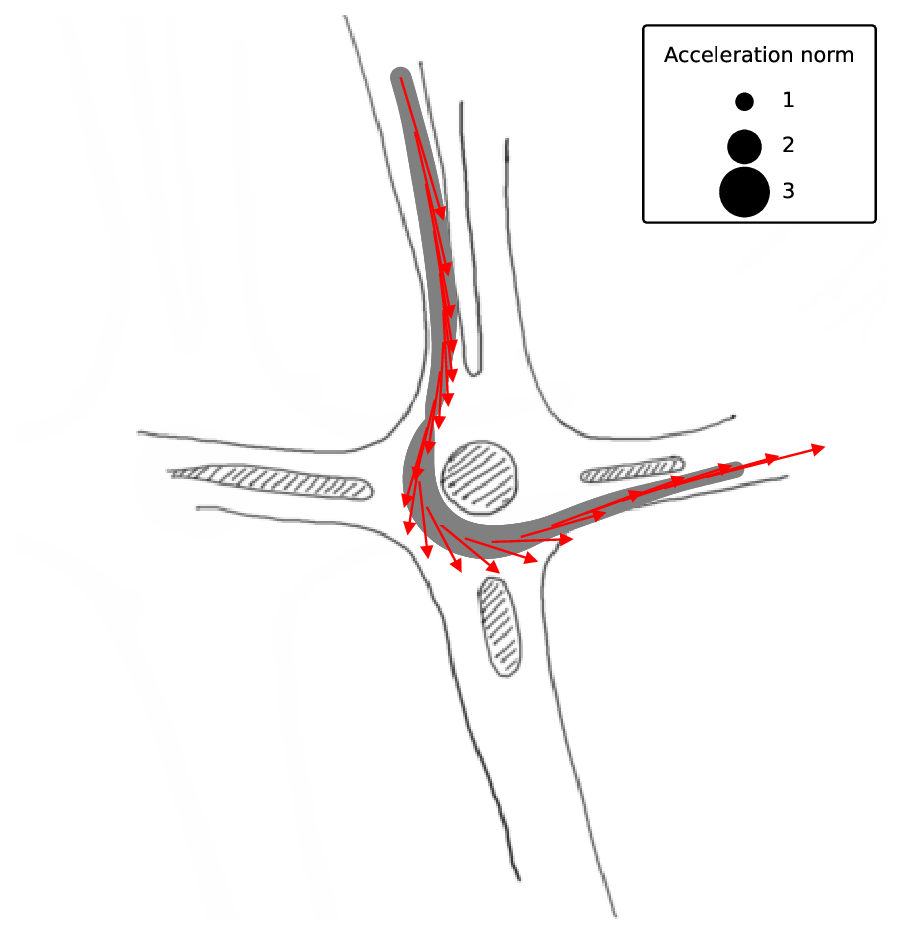}
        \caption{}
        \label{fig:cluster_acc}
    \end{subfigure}
    \caption{rounD dataset: Two different clusters with similar trajectory shape but with different velocity and acceleration profiles: (a) Stop when arriving at the roundabout. (b) Do not stop when arriving at the roundabout.}
    \label{fig:cluster_11}
\end{figure}

\section{Extension to images}\label{sec:extensions}

The $\mathtt{fCUBT}$ algorithm was introduced above for multivariate functional data which could be defined on different domains, possibly of different dimensions. In this section, we present the results of a simulation experiment with a process $X$ with two components, one defined over a compact interval on the real line, the other one defined over a square in the plane. In such situations, the univariate fPCA, performed for each component for the computation of the MFPCA basis, is replaced by a suitable basis expansion for higher dimensional functions. In particular, the eigendecomposition of image data can be performed using the FCP-TPA algorithm for regularized tensor decomposition \cite{allen_multi-way_2013}.  See also \cite{happ_multivariate_2018} and \cite{wang2020}. 

\setcounter{scenario}{3}
\begin{scenario}\label{scenario:4} As in the previous scenarios, the number of clusters is fixed at $K = 5$. Moreover, $P=2$, $\setT_1=[0,1]$ and $\setT_2=[0,1]\times [0,1]$. A sample of $N = 500$ curves is simulated according to the following model, for $s, t \in [0, 1]$:
$$\begin{array}{ll}
X^{(1)}(t) &= a\phi_1(t) + b\phi_2(t) + c\phi_3(t) 
, \\
X^{(2)}(s, t) &= d\phi_1(s) 
\phi_1(t) + e\phi_1(s) 
\phi_2(t) + f\phi_2(s) 
\phi_1(t) + g\phi_2(s) 
\phi_2(t),
\end{array}$$
where $\phi_k$'s are the eigenfunctions of the Wiener process. The coefficients $a, b, c, d, e, f, g$ are random normal variables with parameters defined in Table \ref{tab:coef_scenerio4}. The mixing proportions are taken to be equal. The noisy curves are observed over $100$ equidistant points and the noisy  images are observed over a 2-D grid of $100 \times 100$ points. The measurement errors are introduced as in \eqref{data_real}.  The errors for $X^{(1)}$ are independent, zero-mean, Gaussian variables of variance $\sigma^2 = 0.05$, while the errors for $X^{(2)}$ are bivariate, zero-mean, Gaussian vectors with independent components of  variance $\sigma^2 = 0.05$. The experiment was repeated 500 times. 

\begin{table}
\centering
\begin{tabular}{lrrrrrrr} 
            & \multicolumn{7}{c}{Coefficients (mean / std)} \\
            & $a$ & $b$ & $c$ & $d$ & $e$ & $f$ & $g$ \\
\midrule
 Cluster 1  & $3, 0.5$  & $2, 1.66$ & $1, 1.33$ & $4, 1$ & $0, 0.5$ & $0, 0.1
 $ & $-2, 0.05$     \\
 Cluster 2  & $1, 0.5$  & $-2, 1$   & $0, 1$    & $4, 0.8$ & $0, 0.7$ & $0, 0.08$ & $-2, 0.07$  \\
 Cluster 3  & $1, 0.4$  & $-2, 0.8$ & $0, 0.8$  & $-3, 1$ & $-4, 0.5$ & $0, 0.1$ & $0, 0.05$    \\
 Cluster 4  & $-2, 1$   & $0, 2$    & $-1, 2$   & $0, 0.1$ & $2, 0.1$ & $0, 0.05$ & $0, 0.025$  \\
 Cluster 5  & $-2, 0.2$ & $0, 0.5$  & $-1, 0.5$ & $0, 2$ & $2, 1$ & $0, 0.2$ & $1, 0.1$ \\  
\end{tabular}
\caption{Scenario 4 -- Coefficient for $X^{(1)}(t)$ and $X^{(2)}(s, t)$.}
\label{tab:coef_scenerio4}
\end{table}

By construction, the clusters cannot be retrieved only using only the noisy curves or the noisy images. Thus, the clustering algorithm has to considered these features for the grouping. Examples of realizations from this simulation experiment are shown in Figure \ref{fig:scenario_4_simulated_data}.

\begin{figure}
    \centering
    \begin{subfigure}[b]{\textwidth}
        \centering
        \includegraphics[scale=0.5]{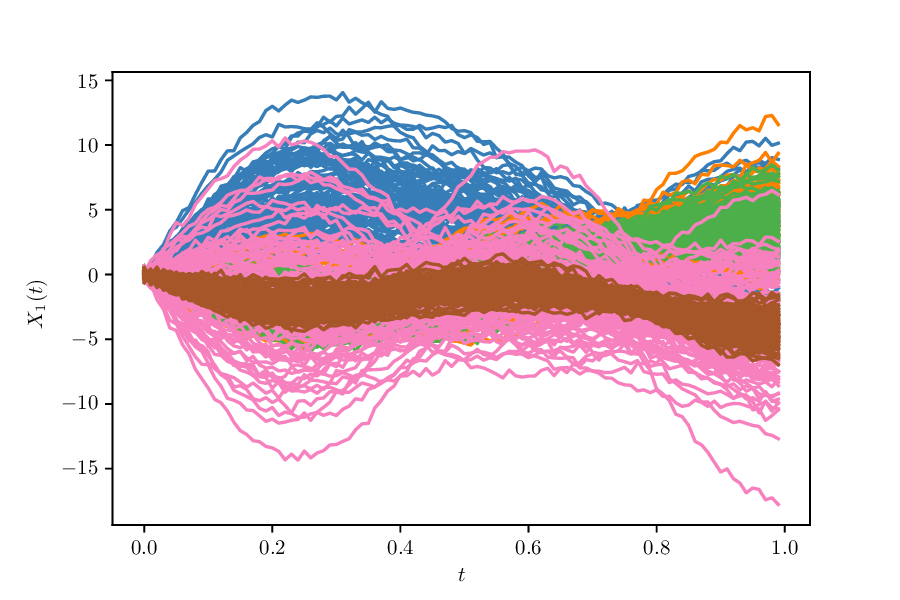}
        \caption{}
        \label{fig:scenario-4_curves}
    \end{subfigure}
    \\
    \begin{subfigure}[b]{0.19\textwidth}
        \centering
        \includegraphics[scale=0.22]{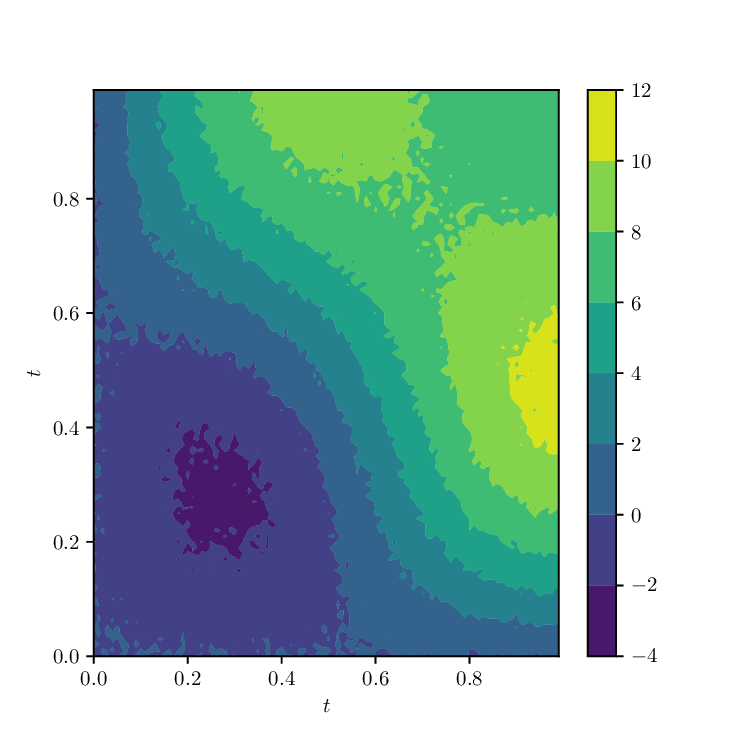}
        \caption{}
        \label{fig:scenario-41}
    \end{subfigure}
    \hfill
    \begin{subfigure}[b]{0.19\textwidth}
        \centering
        \includegraphics[scale=0.22]{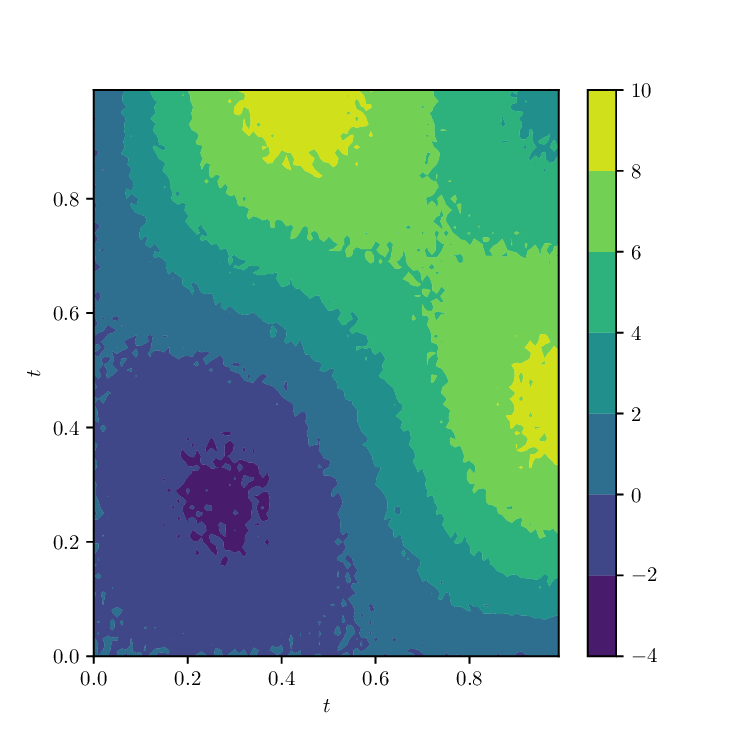}
        \caption{}
        \label{fig:scenario-42}
    \end{subfigure}
    \hfill
    \begin{subfigure}[b]{0.19\textwidth}
        \centering
        \includegraphics[scale=0.22]{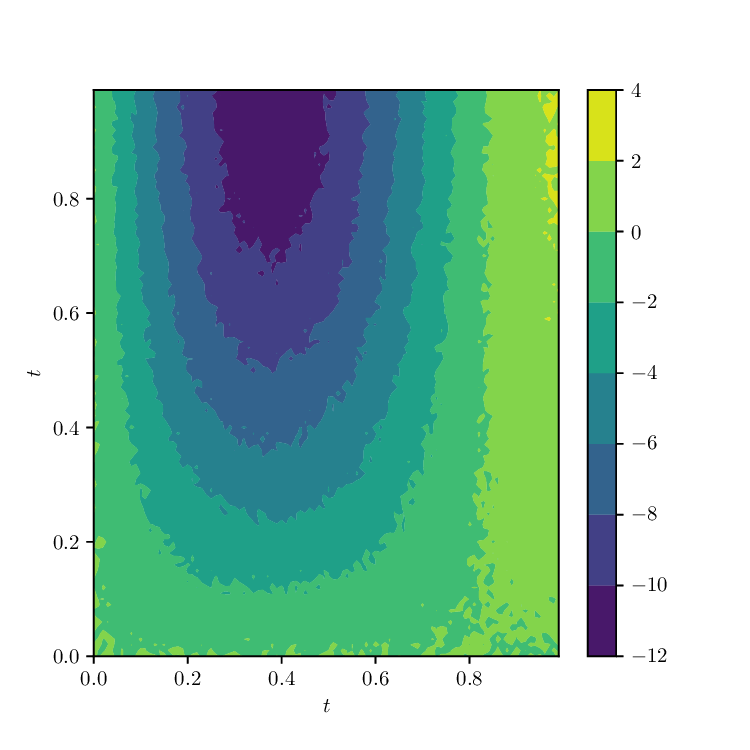}
        \caption{}
        \label{fig:scenario-43}
    \end{subfigure}
    \hfill
    \begin{subfigure}[b]{0.19\textwidth}
        \centering
        \includegraphics[scale=0.22]{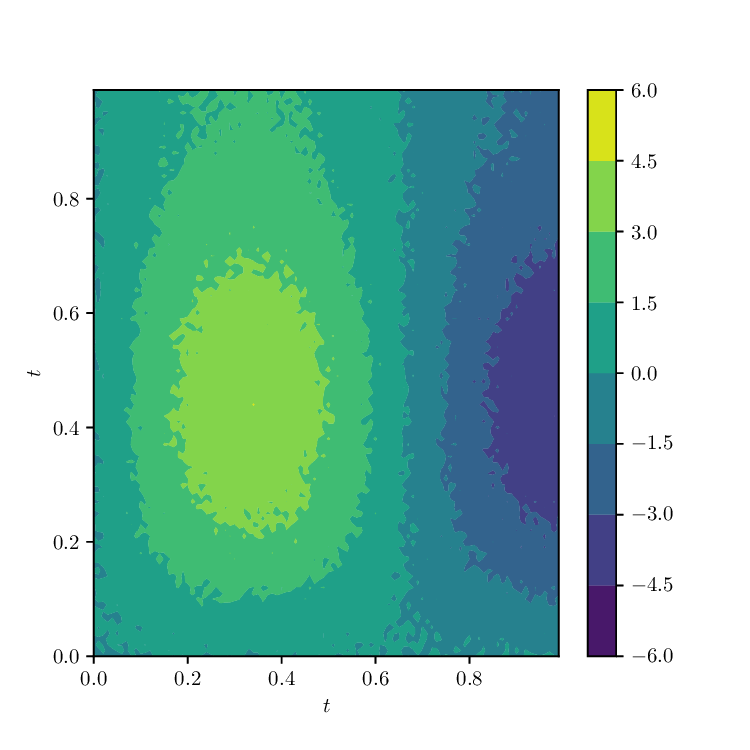}
        \caption{}
        \label{fig:scenario-44}
    \end{subfigure}
    \hfill
    \begin{subfigure}[b]{0.19\textwidth}
        \centering
        \includegraphics[scale=0.22]{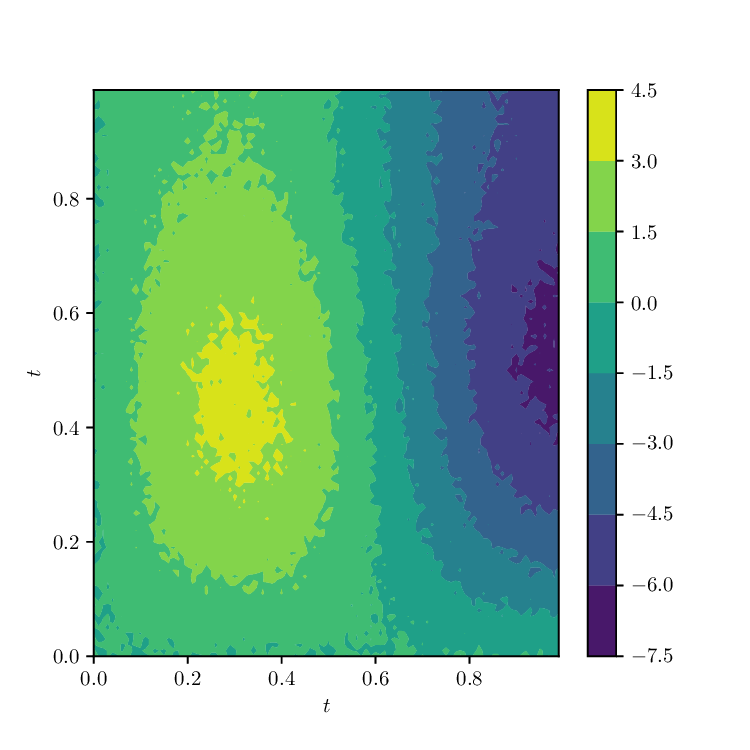}
        \caption{}
        \label{fig:scenario-45}
    \end{subfigure}

    \caption{Examples of simulated data for  Scenario 4~: (a)  $500$ realizations of the process $X^{(1)}$. (b)--(f) One realization of each of the clusters from the process $X^{(2)}$.}
    \label{fig:scenario_4_simulated_data}
\end{figure}

The results of the \texttt{fCUBT} procedure are given in the Table \ref{tab:n_cluster_4} and \ref{tab:n_cluster_4_ari} for both the estimated number of clusters and the \ARI. We remark that in more than half of the cases, our algorithm estimates the number of clusters correctly. However, as  cluster $4$ and $5$ are hard to discriminate, it returns four estimated clusters. This phenomenon is also reflected in the \ARI{} results. In fact, the \ARI{} presents a bimodal distribution where one mode is centered around $0.98$ and the other one around $0.77$. It appears that when the \ARI{} is around $0.98$, the number of clusters is well estimated, while when the \ARI{} is close to $0.77$, it is not. Nevertheless, even if in some replications, the number of clusters is wrongly estimated, overall the results are good. This simulation experiment provides strong evidence that our algorithm could  also be used in such complex situations to which, apparently,  the existing clustering algorithms have not yet been extended. 

\begin{table}
\begin{subtable}{\textwidth}
\centering
\begin{tabular}{rrrrr}
Number of clusters $K$ &    4 &      5  &      6 &      9 \\
\midrule
\texttt{fCUBT} &  43 &  52 &  4 &  1 \\
\end{tabular}
\caption{Scenario 4 -- Number of clusters selected for \texttt{fCUBT} for $500$ simulations as a percentage.}
\label{tab:n_cluster_4}
\end{subtable}%

\begin{subtable}{\textwidth}
\centering
\begin{tabular}{rrrrrrrrrrrr}
Quantile & 0.0 & 0.1 & 0.2 & 0.3 & 0.4 & 0.5 & 0.6 & 0.7 & 0.8 & 0.9 &  1.0 \\
\midrule
\texttt{fCUBT} &  0.70 &  0.76 &  0.77 &  0.77 &  0.78 &  0.78 &  0.96 &  0.98 &  0.99 &  0.99 &  1.0 \\
\end{tabular}
\caption{Scenario 4 -- Quantile of the ARI for $500$ simulations.}
\label{tab:n_cluster_4_ari}
\end{subtable}

\caption{Results for the Scenario 4.}
\label{tab:results_scenario4}
\end{table}

\end{scenario}

\section{Conclusion}\label{sec:conclusion}

The $\mathtt{fCUBT}$ algorithm has been proposed which is a model-based clustering method for functional data based on unsupervised binary trees. It works both on univariate and multivariate functional data defined in possibly different, multidimensional domains. The method is particularly suitable for finding the correct number of clusters within the data with respect to the model assumption. When the complete tree has been grown, $\mathtt{fCUBT}$ can be used for supervised classification. The open-source implementation can be accessed at \url{https://github.com/StevenGolovkine/FDApy} while scripts to reproduce the simulation and real-data analysis are at \url{https://github.com/StevenGolovkine/fcubt}.

A Reviewer asked about whether the Gaussian assumption imposed in model \eqref{eq:mixture_curve} is strict. First, one can use, for instance, the tests proposed in \cite{norm_test} to check this assumption using the terminal nodes of the tree, before or after the joining step. Some evidence of the effectiveness of these tests is provided in the Supplement. On the other hand, we also provide in the Supplement some simulation-based evidence on the robustness of our algorithm to departures from the Gaussian assumption. The conclusion of our simulation experiments is that, in general, the performance of our algorithm diminishes when the Gaussianity condition is not met, but remains good and comparable to that of the competitors, at least for some types of departures.

\section*{Acknowledgment}

The authors wish to thank Groupe Renault and the ANRT (French National Association for Research and Technology) for their financial support via the CIFRE convention no.~2017/1116. Valentin Patilea gratefully acknowledges support from the Joint Research Initiative “Models and
mathematical processing of very large data” under the aegis of Risk Foundation, in partnership
with MEDIAMETRIE and GENES, France.

\appendix

\section{Lemmas and Proofs}\label{app:proof_lemma}


\begin{proof}[Proof of the Lemma \ref{lemma:c_j}]
Using the linearity of the inner product, we may rewrite for each $j \geq 1, c_j$ as
\begin{equation*}
c_j = \inH{X - \mu, \psi_j} = \sum_{k=1}^{K}\inH{\mu_k, \psi_j}\1_{\{Z = k\}} - \inH{\mu, \psi_j} + \sum_{l \geq 1}^{} \xi_l\inH{\phi_l, \psi_j}.
\end{equation*}
Since a linear combination of independant Gaussian distributions is still Gaussian, the conditional distribution $c_j \given Z = k$ has a Gaussian distribution for all $k \in \{1, \dots, K\}$, $j \geq 1$. Moreover, by the definition of $X$ and the linearity of the inner product, for any $i, j\geq 1$ and $ k \in \{1, \dots, K\}$,
 \begin{align*}
\EspCond{c_j}{Z = k} 
	   &= \sum_{k^\prime=1}^{K}\inH{\mu_{k^\prime}, \psi_j}\EspCond{\1_{\{Z = k^\prime\}}}{Z = k} - \inH{\mu, \psi_j} + \sum_{l \geq 1}^{} \EspCond{\xi_l}{Z = k}\inH{\phi_l, \psi_j} \\
	   &= \inH{\mu_k - \mu, \psi_j}.
\end{align*}
Next, for any $i, j \geq 1$ and $ k \in \{1, \dots, K\}$,
\begin{multline*}
\CovCond{c_i, c_j}{Z = k} = \EspCond{c_ic_j}{Z = k} - \EspCond{c_i}{Z = k}\EspCond{c_j}{Z = k}\\
	= \sum_{p=1}^{P}\sum_{q=1}^{P}\int_{[0,1]}^{}\int_{[0,1]}^{}\EspCond{(X - \mu)^{(p)}(s_p)(X - \mu)^{(q)}(t_q)}{Z = k}\psi_i^{(p)}(s_p)\psi_j^{(q)}(t_q)ds_pdt_q \\
	\qquad - \inH{\mu_k - \mu, \psi_i}\inH{\mu_k - \mu, \psi_j}.
\end{multline*}
By definition, for any $1\leq p,q\leq P$,
\begin{align*}
&\EspCond{(X - \mu)^{(p)}(s_p)(X - \mu)^{(q)}(t_q)}{Z = k}\\
&\quad = \sum_{k^\prime = 1}^K\sum_{k^{\prime\prime} = 1}^K \mu_{k^\prime}^{(p)}(s_p)\mu_{k^{\prime\prime}}^{(q)}(t_q)\EspCond{\1_{\{Z = k^\prime\}}}{Z = k}\EspCond{\1_{\{Z = k^{\prime\prime}\}}}{Z = k} \\
&\qquad\qquad + \sum_{k^\prime = 1}^K \mu_{k^\prime}^{(p)}(s_p)\sum_{j \geq 1}\phi_{j}^{(q)}(t_q)\EspCond{\xi_{j}\1_{\{Z = k^\prime\}}}{Z = k} \\
&\qquad\qquad + \sum_{k^{\prime \prime}= 1}^K \mu_{k^{\prime\prime}}^{(q)}(t_q)\sum_{l \geq 1}\phi_{l}^{(p)}(s_p)\EspCond{\xi_{l}\1_{\{Z = k^{\prime\prime}\}}}{Z = k} \\
&\qquad\qquad + \sum_{j \geq 1}^{}\sum_{l \geq 1}\phi_j^{(p)}(s_p)\phi_{l}^{(q)}(t_q)\EspCond{\xi_j\xi_{l}}{Z = k} \\
&\qquad\qquad - \mu^{(q)}(t_q)\sum_{k^\prime = 1}^K \mu_{k^\prime}^{(p)}(s_p)\EspCond{\1_{\{Z = k^\prime\}}}{Z = k} \\
&\qquad\qquad - \mu^{(p)}(s_p)\sum_{k^{\prime\prime} = 1}^K \mu_{k^{\prime\prime}}^{(q)}(t_q)\EspCond{\1_{\{Z = k^{\prime\prime}\}}}{Z = k} \\
&\qquad\qquad - \mu^{(p)}(s_p)\sum_{l \geq 1} \phi_{l}^{(q)}(t_q)\EspCond{\xi_l}{Z = k} \\
&\qquad\qquad - \mu^{(q)}(t_q)\sum_{l \geq 1} \phi_{l}^{(p)}(s_p)\EspCond{\xi_l}{Z = k} \\
&\qquad\qquad + \mu^{(p)}(s_p)\mu^{(q)}(t_q) \\
&\quad = \mu^{(p)}(s_p)\mu^{(q)}(t_q) + \mu_k^{(p)}(s_p)\mu_k^{(q)}(t_q) - \mu^{(p)}(s_p)\mu_k^{(q)}(t_q) - \mu_k^{(p)}(s_p)\mu^{(q)}(t_q) \\
&\qquad\qquad + \sum_{l \geq 1}^{} \sigma_{kl}^2\phi_l^{(p)}(s_p)\phi_l^{(q)}(t_q) \\
&\quad = \left(\mu_k^{(p)}(s_p) - \mu^{(p)}(s_p)\right)\left(\mu_k^{(q)}(t_q) - \mu^{(q)}(t_q)\right) + \sum_{l \geq 1}^{} \sigma_{kl}^2\phi_l^{(p)}(s_p)\phi_l^{(q)}(t_q). \\
\end{align*}
Thus, 
\begin{align*}
\CovCond{c_i,c_j}{Z = k} &= \inH{\mu_k - \mu, \psi_i}\inH{\mu_k - \mu, \psi_j} +  \sum_{l \geq 1}^{} \sigma_{kl}^2\inH{\phi_l, \psi_i}\inH{\phi_l, \psi_j}  \\
	&\qquad - \inH{\mu_k - \mu, \psi_i}\inH{\mu_k - \mu, \psi_j} \\
	&= \sum_{l \geq 1}^{} \sigma_{kl}^2\inH{\phi_l, \psi_i}\inH{\phi_l, \psi_j}.
\end{align*}
Taking $i=j$ in the conditional covariance, we deduce 
\begin{equation*}
\tau_{kj}^2 = \VarCond{c_j}{Z = k} = \sum_{l\geq 1}^{} \sigma_{kl}^2\inH{\phi_{l}, \psi_j}^2.
\end{equation*}
For the marginal distribution of the $c_j$, the zero-mean  is obtained as follows:
$$\EE(c_j) = \sum_{k=1}^{K}\PP(Z = k)\EspCond{c_j}{Z = k} = \sum_{k=1}^{K}p_k\inH{\mu_k - \mu, \psi_j} = 0.$$
For the marginal covariance, we can write
\begin{flalign*}
\Cov(c_i, c_j) &= \EE(c_ic_j) \\
	&= \sum_{k=1}^{K}p_k\EspCond{c_ic_j}{Z = k} \\
	&= \sum_{k=1}^{K}p_k\left(\CovCond{c_i,c_j}{Z = k} + \EspCond{c_i}{Z = k}\EspCond{c_j}{Z = k}\right) \\
	&= \sum_{k=1}^K p_k\left(\sum_{l\geq 1}\inH{\phi_l, \psi_i}\inH{\phi_l, \psi_j}\sigma_{kl}^2 + \inH{\mu_k - \mu, \psi_i}\inH{\mu_k - \mu, \psi_j}\right).
\end{flalign*}
This concludes the proof.
\end{proof}

In applications, one cannot use an infinite number of terms in the representation of the process $X$, and has to truncate such a representation. The following lemma shows that such a truncation could be arbitrarily accurate.

\begin{lemma}\label{lemma:approx}
Let $X$ be defined as in \eqref{eq:mixture_curve} for some orthonormal basis  $\{\phi_j\}_{j \geq 1}$. 
Let $\{\psi_j\}_{j \geq 1}$ be another orthonormal basis in $\setH$ with to which $X$ has the decomposition 
\begin{equation*}
X(\pointt) = \sum_{j \geq 1}^{} c_j\psi_j(\pointt), \quad \pointt \in \mathcal{T}.
\end{equation*}
Then
\begin{equation*}
\lim_{J \rightarrow \infty} \EE\left(\normH{X - X_{\lceil J \rceil}}^2\right) =0, \quad\text{where}\quad X_{\lceil J \rceil}(t) = \sum_{j=1}^J c_j\psi_j(\pointt), \pointt \in \mathcal{T}.
\end{equation*}
\end{lemma}

\begin{proof}[Proof of Lemma \ref{lemma:approx}]
Given  $\{\psi_j\}_{j \geq 1}$  an orthonormal basis of $\setH$, $X$ can be written in the form $X(t) = \sum_{j\geq 1}^{} c_j\psi_j(t), t \in \setT$, with random variables $c_j = \inH{X - \mu, \psi_j}$. Then 
\begin{align*}
\normH{X - X_{\lceil J \rceil}}^2 &= \normH{\sum_{j=J+1}^{\infty}c_j\psi_j}^2 \\
	&= \sum_{p=1}^{P}\int_{\mathcal{T}_p}^{}\left(\sum_{j=J+1}^{\infty}c_j\psi_j^p(t)\right)\left(\sum_{j^\prime=J+1}^{\infty}c_{j^\prime}\psi_{j^\prime}^p(t)\right)dt \\
	&= \sum_{j=J+1}^{\infty}\sum_{j^\prime=J+1}^{\infty}c_jc_{j^\prime}\inH{\psi_j, \psi_{j^\prime}} \\
	&= \sum_{j=J+1}^{\infty}c_j^2.
\end{align*}
Moreover, 
\begin{equation*}
\EE\left(\sum_{j \geq 1} c_j^2\right) = \EE\left(\sum_{j\geq 1} \inH{X - \mu, \psi_j}^2\right) = \EE\left(\inH{X  - \mu}^2\right) < \infty.
\end{equation*}
From this, and the fact that the remainder of a convergent series tends to zero, we have 
\begin{equation*}
\EE\left(\normH{X - X_{\lceil J \rceil}}^2\right) = \EE\left(\sum_{j=J+1}^{\infty}c_j^2\right) \xrightarrow[J \rightarrow \infty]{} 0.
\end{equation*}
This conclude the proof.
\end{proof}

\medskip
\medskip
\medskip
The following lemma shows that the MFPCA basis is the one which will induce the most accurate truncation for a given truncation number $J$. Therefore, among the workable bases one could use in practice, the MFPCA basis is likely to be a privileged one.

\begin{lemma}\label{lemma:best_basis}
Let $X$ be defined as in \eqref{eq:mixture_curve}. 
Let $\{\psi_j\}_{j \geq 1}$ be some orthonormal basis in $\setH$ and  $\{\varphi_j\}_{j \geq 1}$ be the MFPCA basis. Let $\mu$ be the mean curve as defined in \eqref{eq:c_j}.  Then, for any $J\geq 1$ such that $\lambda_{1}> \lambda_{2}\ldots>\lambda_{J}> \lambda_{J+1}$, 
\begin{equation}
\EE\left(\LnormH{X - \mu  - \sum_{j=1}^{J} \inH{X - \mu, \psi_j}  \psi_j  }^2\right) \geq \EE\left(\LnormH{X - \mu -  \sum_{j=1}^{J}\mathfrak{c}_{j}\varphi_j  }^2\right).
\end{equation}
\end{lemma}

\begin{proof}[Proof of Lemma \ref{lemma:best_basis}]
First, let us note that, since the bases are orthogonal, the minimization of the truncation error for a given $J$ is equivalent to the maximization of the sum of the variances $\inH{\Gamma \psi_j, \psi_j}$, $1\leq j\leq J$. Moreover, for each $j\geq 1$, we have 
\begin{align*}
 & \!\!\!\!\! \!\!\!\!\! \EE\left[\inH{X - \mu,  \varphi_j}^2\right]  \\
	&= \EE\left[\left(\sum_{p=1}^{P}\inLp{(X - \mu)^{(p)},  \varphi_j^{(p)}}\right)\left(\sum_{q=1}^{P}\inLp{(X - \mu)^{(q)},  \varphi_j^{(q)}}\right) \right] \\
	&= \EE\left[\left(\sum_{p=1}^{P}\int_{[0,1]} (X - \mu)^{(p)}(s_p) \varphi_j^{(p)}(s_p)ds_p\right)\left(\sum_{q=1}^{P}\int_{[0,1]} (X - \mu)^{(q)}(t_q) \varphi_j^{(q)}(t_q)dt_q\right) \right] \\
	& = \sum_{p=1}^{P}\int_{[0,1]}\sum_{q=1}^{P}\int_{[0,1]} \EE\left[(X - \mu)^{(p)}(s_p) 
	 (X - \mu)^{(q)}(t_q) 
\right]
 \varphi_j^{(p)}(s_p) \varphi_j^{(q)}(t_q)ds_pdt_q  \\
	&= \sum_{p=1}^{P}\int_{[0,1]}\sum_{q=1}^{P}\int_{[0,1]} C_{p, q}(s_p, t_q) \varphi_j^{(p)}(s_p) \varphi_j^{(q)}(t_q)ds_pdt_q \\
	&= \sum_{q=1}^{P}\int_{[0,1]}\sum_{p=1}^{P} \inLp{C_{p, q}(\cdot, t_q),  \varphi_j^{(p)(\cdot)}} \varphi_j^{(q)}(t_q)dt_q \\
	&= \sum_{q=1}^{P}\int_{[0,1]}\inH{C_{\cdot, q}(\cdot, t_q),  \varphi_j(\cdot)} \varphi_j^{(q)}dt_q \\
	&= \sum_{q=1}^{P}\int_{[0,1]} \left(\Gamma \varphi_j\right)^{(q)}(t_j) \varphi_j^{(q)}(t_q)dt_q \\
	&= \sum_{q=1}^{P} \inLp{\left(\Gamma \varphi_j\right)^{(q)}(t_j),  \varphi_j^{(q)}(t_q)} \\
	&= \inH{\Gamma \varphi_j,  \varphi_j}
\end{align*}
Since  the MfPCA basis is characterized by the property \eqref{property_fpca},
for any orthonormal basis $\{\psi_j\}_{j \geq 1}$, we necessarily have 
\begin{equation*}
\sum_{j=1}^{J} \Var(\inH{X - \mu,  \varphi_j}) = \sum_{j=1}^{J} \inH{\Gamma \varphi_j, \varphi_j} \geq \sum_{j=1}^{J} \inH{\Gamma \psi_j, \psi_j} .
\end{equation*}
This concludes the proof.
\end{proof}

\section{Algorithms}\label{app:algo}

\begin{algorithm*}
\SetAlgoLined
\KwInput{A training sample $\mathcal{S}_{N_0} = \{X_1, \dots, X_{N_0}\}\subset \setH$,  $J^{(p)}$,  
 $K_{max}$ and $\mathtt{minsize}$.}

\KwInitialization{Set $(\mathfrak{d} ,\mathfrak{j} )=(0, 0)$ and $\mathfrak{S}_{0, 0} = \mathcal{S}_{N_0}$. }

\KwFPCA{Perform a MFPCA with  $J$   components on the data  in the node $\mathfrak{S}_{\mathfrak{d, j}}$ and get the set of eigenvalues $\Lambda_{\mathfrak{d, j}}$ associated with a set of eigenfunctions $\Phi_{\mathfrak{d, j}}$. Build the matrix $C_{\mathfrak{d, j}}$ defined in \eqref{eq:proj_matrix}.}

\KwGMM{For each $K = 1, \dots, K_{max}$, fit  $K-$components GMM  using an EM algorithm on the columns of the matrix $C_{\mathfrak{d, j}}$. The models are denoted by $\{\mathcal{M}_1, \dots, \mathcal{M}_{K_{max}}\}$. 
The number of mixture components 
is estimated by $\widehat{K}_{\mathfrak{d, j}}$ defined in \eqref{eq:K_hat} using the \BIC.}

\KwStop{Test if the node indexed by $(\mathfrak{d}, \mathfrak{j})$ is a terminal node, that is if $\widehat{K}_{\mathfrak{d, j}} = 1$ or if there are less than $\mathtt{minsize}$ elements in $\mathcal{S}_{\mathfrak{d, j}}$. If the node is terminal, then stop the construction of the tree for this node, otherwise go to the next step.}

\KwConstruction{A non-terminal node indexed by $(\mathfrak{d}, \mathfrak{j})$ is split into two subnodes as follows:
\begin{enumerate}
	\item Fit a $2-$component GMM using an EM algorithm.
	\item For each element of $\mathfrak{S}_{\mathfrak{d, j}}$,  compute the posterior probability to belong to the first component.
	\item Form the children nodes as $\mathfrak{S}_{\mathfrak{d} + 1, 2\mathfrak{j}} = \{ \text{elements of } \mathfrak{S}_{\mathfrak{d, j}} \text{ with posterior probability } \geq 1/2 \}$ and $\mathfrak{S}_{\mathfrak{d} + 1, 2\mathfrak{j} + 1} = \mathfrak{S}_{\mathfrak{d, j}}  \setminus \mathfrak{S}_{\mathfrak{d} + 1, 2\mathfrak{j}}.$
\end{enumerate}}

\KwRecursion{Continue the procedure by applying the \textbf{Computation of the MFPCA components} step to the nodes $(\mathfrak{d} + 1, 2\mathfrak{j})$ and $(\mathfrak{d} + 1, 2\mathfrak{j} + 1)$.}

\KwOutput{A set of nodes $\{\mathfrak{S}_{\mathfrak{d, j}},~ 0 \leq \mathfrak{j} < 2^{\mathfrak{d}},~ 0 \leq \mathfrak{d} < \mathfrak{D}\}$.}

\caption{Construction of a tree $\mathfrak{T}$}
\label{alg:tree_construction}
\end{algorithm*}

An example of a maximal tree is given in Figure \ref{fig:max_tree}. It corresponds to a simulated dataset defined in Section \ref{sec:emp_analysis}, Scenario $1$, where  $J^{(p)}$  is set to explain $95\%$ of the variance at each node of the tree, $K_{max} = 5$ and $\mathtt{minsize} = 10$. This tree has six leaves, whereas Scenario $1$ contains only five clusters. In this illustration, the joining step will be helpful to join the group with only $4$ curves, corresponding to the node $\mathfrak{S}_{3, 2}$, with another group, hopefully the node $\mathfrak{S}_{3, 1}$. 

\begin{figure}
    \centering
    \includegraphics[scale=0.55]{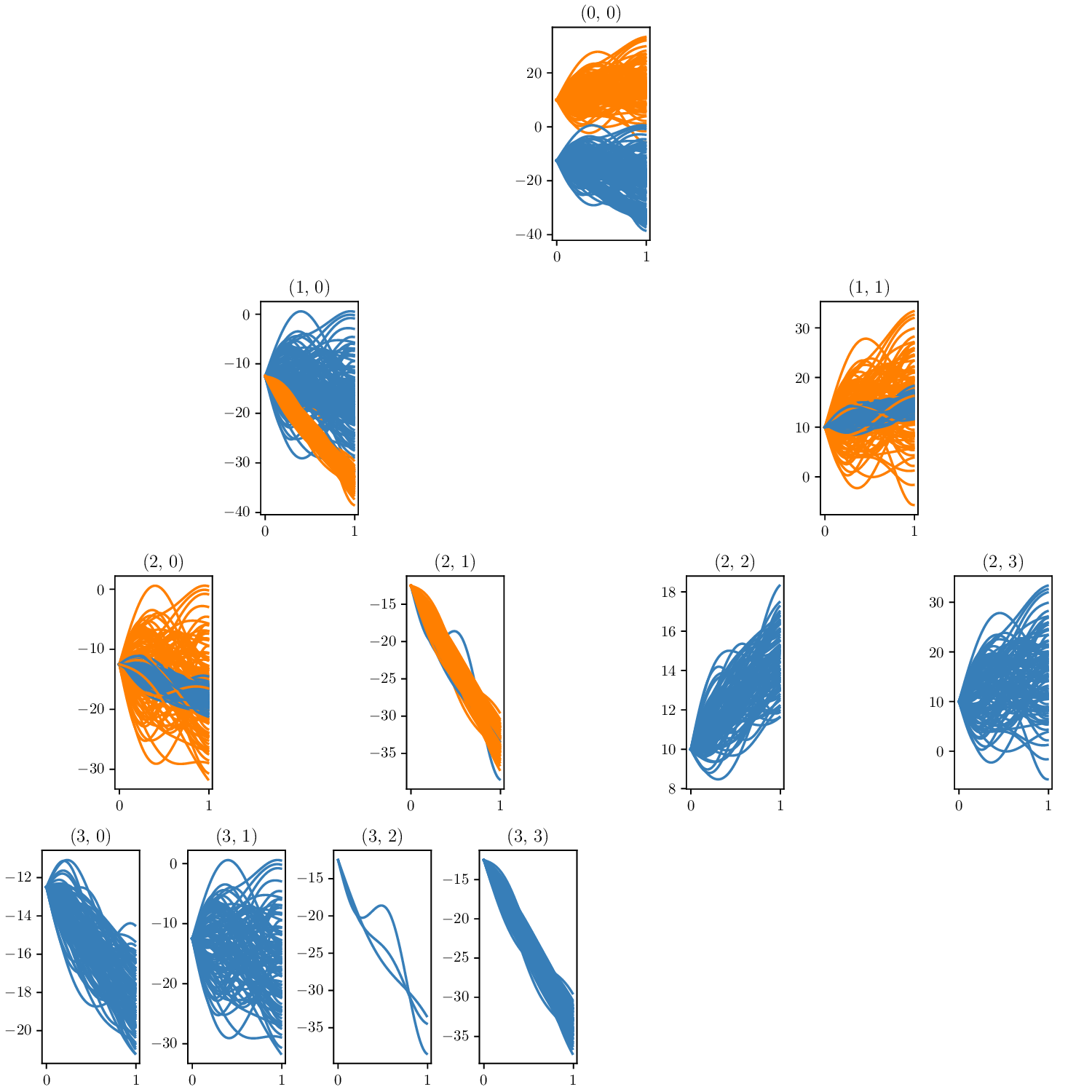}
    \caption{Illustration of maximal tree for Scenario 1 simulated data (different scale for each graph).}
    \label{fig:max_tree}
\end{figure}

\begin{algorithm*}
\SetAlgoLined
\KwInput{A set of nodes $\{\mathfrak{S}_{\mathfrak{d, j}},~ 0 \leq \mathfrak{j} < 2^{\mathfrak{d}},~ 0 \leq \mathfrak{d} < \mathfrak{D}\}$,  $J^{(p)}$,   
and $K_{max}$.}

\KwInitialization{Build the set of terminal nodes 
$$V = \{\mathfrak{S}_{\mathfrak{d, j}}, 0 \leq j < 2^\mathfrak{d}, 0 \leq \mathfrak{d} < \mathfrak{D} \mathbin{\mid} \mathfrak{S}_{\mathfrak{d, j}} \;\text{is a terminal node}\}.$$}
\KwGraph{Build the set $E$ defined in \eqref{eq:set_edges} and denote by $\mathcal{G}$ the graph $(V, E)$. Associate with each edge $(\mathfrak{S}_{\mathfrak{d, j}}, \mathfrak{S}_{\mathfrak{d^\prime, j^\prime}})$ the value of the \BIC{} that corresponds to $\widehat{K}_{(\mathfrak{d, j}) \cup (\mathfrak{d^\prime, j^\prime})}$.}

\KwStop{If $E$ is empty or $V$ is reduced to a unique element, stop the algorithm.}

\KwAggregate{Let $(\mathfrak{S}_{\mathfrak{d, j}}, \mathfrak{S}_{\mathfrak{d^\prime, j^\prime}})$ be the edge with the maximum \BIC{} value. Then, remove this edge and replace the asssociated vertice by $\mathfrak{S}_{\mathfrak{d, j}} \cup \mathfrak{S}_{\mathfrak{d^\prime, j^\prime}}$.}

\KwRecursion{Continue the procedure by applying the \textbf{Creation of the graph} step with $\{V \setminus \{\mathfrak{S}_{\mathfrak{d, j}}, \mathfrak{S}_{\mathfrak{d^\prime, j^\prime}}\}\} \cup \{\mathfrak{S}_{\mathfrak{d, j}} \cup \mathfrak{S}_{\mathfrak{d^\prime, j^\prime}}\}$.}

\KwOutput{A partition $\mathcal{U}$ of $\mathfrak{S}_{0, 0}$ and labels associated to each element of $\mathcal{S}_{N_0}$.}
\caption{Joining step}
\label{alg:join_nodes}
\end{algorithm*}

\begin{algorithm*}
\SetAlgoLined
\KwInput{A new realization $\mathfrak{X}$ of the process $X$, the complete tree $\mathfrak{T}$ and the partition $\mathcal{U}$.}

\For{$\mathfrak{S}_{\mathfrak{d, j}} \in \{\mathfrak{S}_{\mathfrak{d, j}},~ 0 \leq \mathfrak{j} < 2^{\mathfrak{d}},~ 0 \leq \mathfrak{d} < \mathfrak{D}\}$}{
\begin{enumerate}
	\item Compute the vector
	$$\left(\inH{\mathfrak{X} - \mu_{\mathfrak{d, j}}, \varphi^1_{\mathfrak{d, j}}}, \dots, \inH{\mathfrak{X} - \mu_{\mathfrak{d, j}},  \varphi^{J}_{\mathfrak{d, j}}} \right)^\top;$$

	\item Compute the posterior probability to belong to each component of the $2$-components GMM fitted on $\mathfrak{S}_{\mathfrak{d, j}}$;

	\item Compute the probability to be in $\mathfrak{S}_{\mathfrak{d, j}}$ as
	$$\PP^\star(\mathfrak{X} \in \mathfrak{S}_\mathfrak{d, j}) = \prod_{\mathfrak{S} \in \mathtt{Pa}(\mathfrak{S}_{\mathfrak{d, j}})}^{} \PP^\star(\mathfrak{X} \in \mathfrak{S} \mid \mathfrak{X} \in \mathtt{Pa}(\mathfrak{S})).$$
\end{enumerate}
}

\For{$U \in \mathcal{U}$}{Compute $\PP^\star(\mathfrak{X} \in U)$.}

\KwOutput{A label for $\mathfrak{X}$ which is defined as $\arg\max_{U \in \mathcal{U}} \PP^\star(\mathfrak{X} \in U)$.}
\caption{Classify one observation}
\label{alg:classify_one_obs}
\end{algorithm*}

\subsection{On the number of scores $J^{(p)}$ and $J$}\label{nb_sco_sm}

In many situations, considering one score for each coordinate $X^{(p)}$ (\emph{i.e.}, taking $J^{(p)} =1$) suffices to detect a mixture structure. This allows $J$ to be kept small when $P$ is large. We provide here some details on why one score could be sufficient for revealing a mixture structure. For simplicity, let $P=1$ and $K=2$. Let
\begin{equation}
X(t) = \sum_{k=1}^{K}\mu_k(t)\1_{\{Z = k\}} + \sum_{j \geq 1}^{} \xi_j\phi_j(t), \quad t \in \setT,
\end{equation}
with some  basis  $\{\phi_j\}_{j \geq 1}$, and $k\in\{1,2\}$
Let $\{\varphi_j\}_{j \geq 1}$ be the FPCA basis. By Lemma \ref{lemma:c_j}, considering only $\varphi_1$, we have 
\begin{equation}
c_1 = \inH{X - \mu, \varphi_j}, \quad \text{where}\quad \mu(\cdot) = p_1\mu_1(\cdot) + (1-p_1) \mu_2(\cdot), \quad \text{and}\quad p_1 = \mathbb P (Z=1)\in(0,1). 
\end{equation}
Then, for $k\in\{1,2\}$, 
\begin{equation}\label{eq:c_j_law}
c_1 \given Z = k \sim \nLaw{m_{k1}}{\tau_{k1}^2}, \;\;\text{where}\;\; m_{k1} = \inH{\mu_k - \mu, \varphi_1} \;\;\text{and}\;\; \tau_{k1}^2 = \sum_{l \geq 1}^{}\inH{\phi_{l}, \varphi_1}^2\sigma_{kl}^2.
\end{equation}
Let us list the cases where the mixture structure can be detected by the first score. 

\textit{Case 1.} We have $m_{11}\neq m_{21}$, that is 
$
 \inH{\mu_1 - \mu_2, \varphi_1} \neq 0.
$
In particular, this requires $\mu_1 \neq  \mu_2$. 

\textit{Case 2.} We have $m_{11}= m_{21}$, but 
$\tau_{11}^2\neq \tau_{21}^2$, that is 
$$
\sum_{l \geq 1}^{}\inH{\phi_{l}, \varphi_1}^2\{ \sigma_{1l}^2 -  \sigma_{2l}^2\} \neq 0.
$$
In particular, if $\inH{\phi_{l}, \varphi_1}=0$ when $l>1$, we  need $\sigma_{11}^2 \neq  \sigma_{22}^2$.

\section{Numerical illustrations}\label{app:numerical_illustation}

We use the Adjusted Rand Index to compare the different clustering algorithms. In the following, we provide a description of this criterion.
When the true labels are available, the estimated partitions are compared with the true partition using the Adjusted Rand Index (ARI) \cite{hubert_comparing_1985}, which is an ``adjusted for chance'' version of the Rand Index \cite{rand_objective_1971}. Let $\mathcal{U} = \{U_1, \dots, U_r\}$ and $\mathcal{V} = \{V_1, \dots, V_s\}$ be two different partitions of $\mathcal{S}_N$, \emph{i.e}
\begin{gather*}
U_i \subset \mathcal{S}_N, \quad 1 \leq i \leq r, \qquad V_j \subset \mathcal{S}_N, \quad 1 \leq j \leq s, \\
\mathcal{S}_N = \bigcup_{i = 1}^r U_i = \bigcup_{j = 1}^s V_j, \\
\text{and}\quad U_i \cap U_{i^\prime} = \varnothing, \quad 1 \leq i, i^\prime \leq r, \qquad V_j \cap V_{j^\prime} = \varnothing, \quad 1 \leq j, j^\prime \leq s.
\end{gather*}
We denote by $n_{ij} \coloneqq \lvert U_i \cap V_j \rvert, 1 \leq i \leq r, 1 \leq j \leq s,$ the number of elements of $\mathcal{S}_N$ that are common to the sets $U_i$ and $V_j$. $n_{i \cdot} \coloneqq \lvert U_i \rvert$ (or $n_{\cdot j} \coloneqq \lvert V_j \rvert$) then corresponds to the number of elements in $U_i$ (or $V_i$). With these notations, the ARI is defined as 
\begin{equation}\label{eq:ARI}
\text{ARI}(\mathcal{U}, \mathcal{V}) = \frac{\sum_{i = 1}^{r}\sum_{j = 1}^{s}\binom{n_{ij}}{2} - \left[\sum_{i = 1}^{r}\binom{n_{i \cdot}}{2}\sum_{j = 1}^{s}\binom{n_{\cdot j}}{2}\right] / \binom{N}{2}}{\frac{1}{2}\left[\sum_{j = 1}^{s}\binom{n_{\cdot j}}{2} + \sum_{j = 1}^{s}\binom{n_{\cdot j}}{2}\right]- \left[\sum_{i = 1}^{r}\binom{n_{i \cdot}}{2}\sum_{j = 1}^{s}\binom{n_{\cdot j}}{2}\right] / \binom{N}{2}}.
\end{equation}

\subsection{Computation times}

In Figure \ref{fig:comptime} we report the computation times for $R = 100$ replications of each of the scenarios. For the Scenarios $1$, $2$ and $3$, we consider $N = 1000$ and the curves are observed over $101$ equidistant points. The $J^{(p)}$ parameter is set to represent $95\%$ of the variance in the data, and $K_{max}$ and $minsize$ are also set to their default values. For Scenario 4, where $P=2$, we consider $N = 500$ and the curves are observed over $100$ equidistant points and the images are observed over a $2$-D grid of $100 \times 100$ points. We choose $J^{(1)}$ that represents $95\%$ of the variance for the curves, and $J^{(2)} = 2$ for the images. We keep $K_{max}$ and $minsize$ as default. In all situations considered, the computation time is inferior to one minute most of the time.  All the computations were performed on a MacBook Pro mid-2014 with $2.6$GHz Inter Core i5 with $8$Go RAM.

\begin{figure}
	\centering
	\includegraphics[scale=0.45]{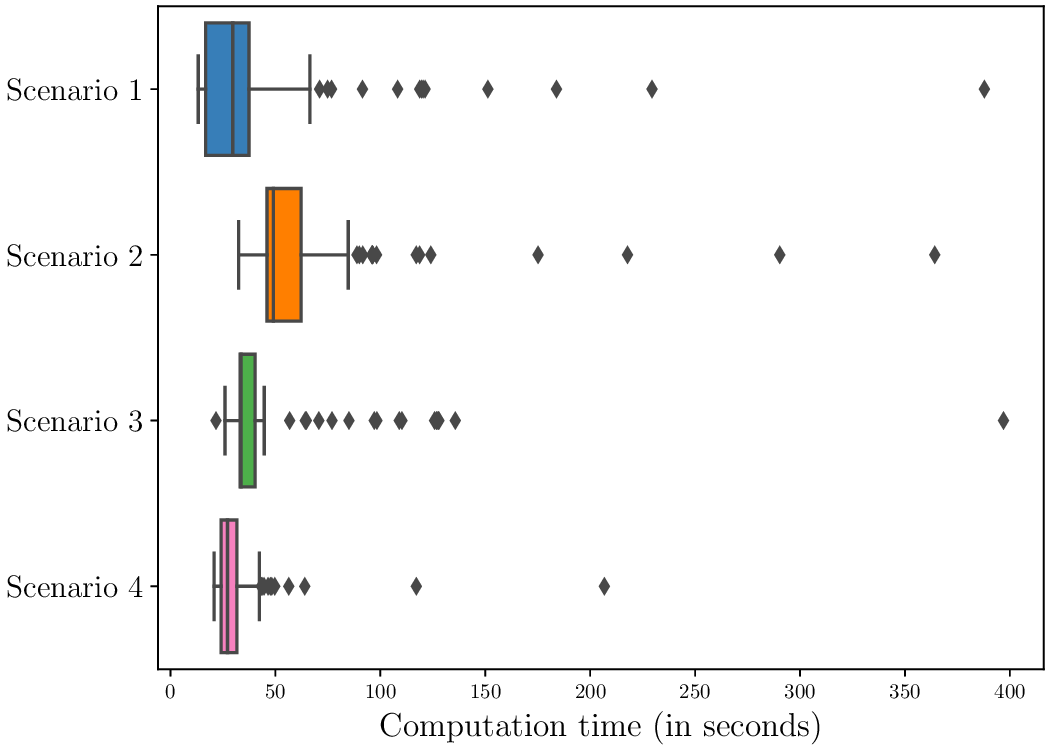}
	\caption{Computation times for all the scenarios from 100 experiments with $N = 1000$ for the Scenarios $1$, $2$ and $3$ and $N = 500$ for the Scenario $4$.}
	\label{fig:comptime}
\end{figure}

\subsection{Influence of $J^{(p)}$}

The  results in this section correspond to $500$ experiments with $N = 1000$, for all Scenarios. 
In Figure \ref{fig:ncomp_comparison} we present a comparison of the performance of the algorithm for different values of $J^{(p)}$. For Scenario 1, the performance does not improve as $J^{(p)}$ increases. For Scenario 2 and 3, larger $J^{(p)}$ leads to slightly larger \ARI{}. 

\begin{figure}
     \centering
     \begin{subfigure}[b]{0.3\textwidth}
         \centering
         \includegraphics[scale=0.35]{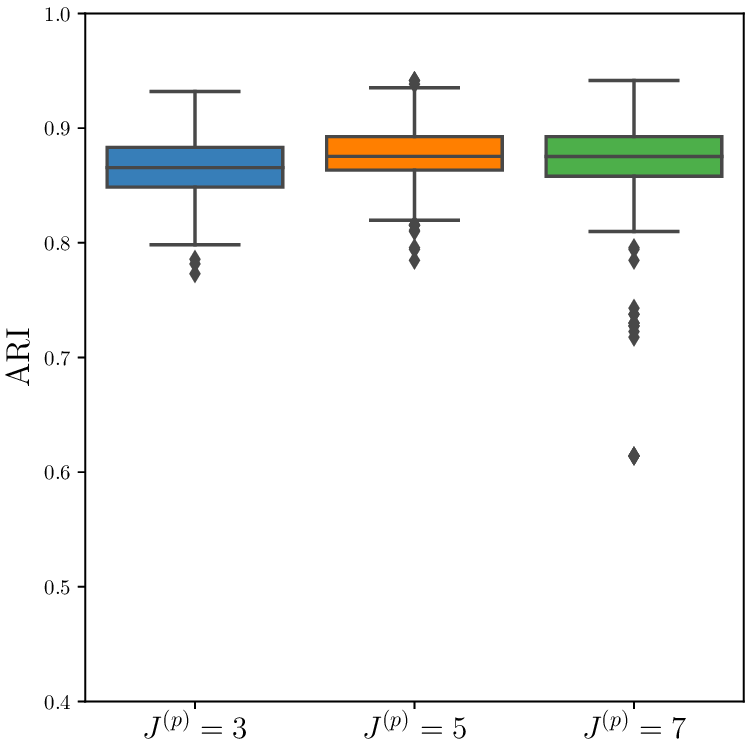}
         \caption{Scenario 1}
         \label{fig:influence_scenario_1}
     \end{subfigure}
     \begin{subfigure}[b]{0.3\textwidth}
         \centering
         \includegraphics[scale=0.35]{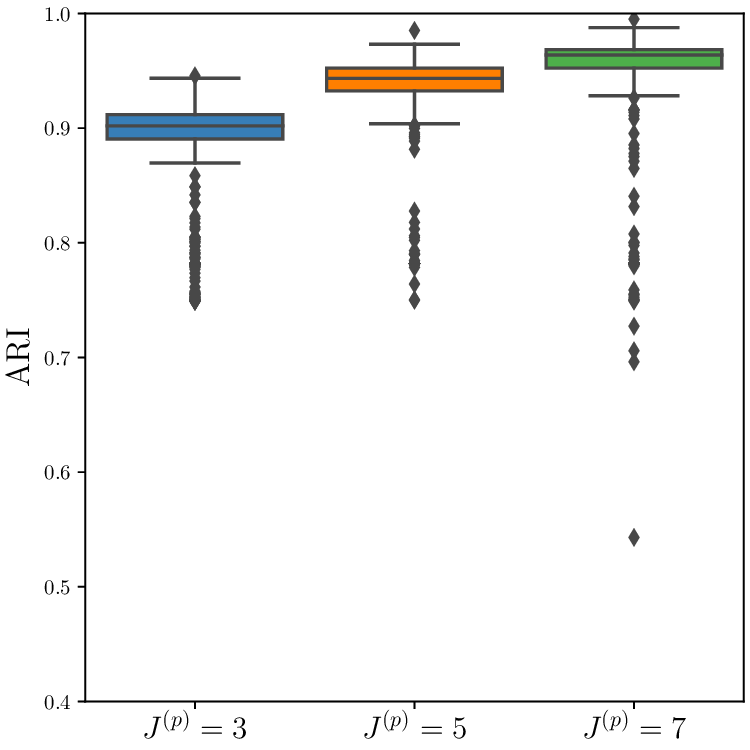}
         \caption{Scenario 2}
         \label{fig:influence_scenario_2}
     \end{subfigure}
     \begin{subfigure}[b]{0.3\textwidth}
         \centering
         \includegraphics[scale=0.35]{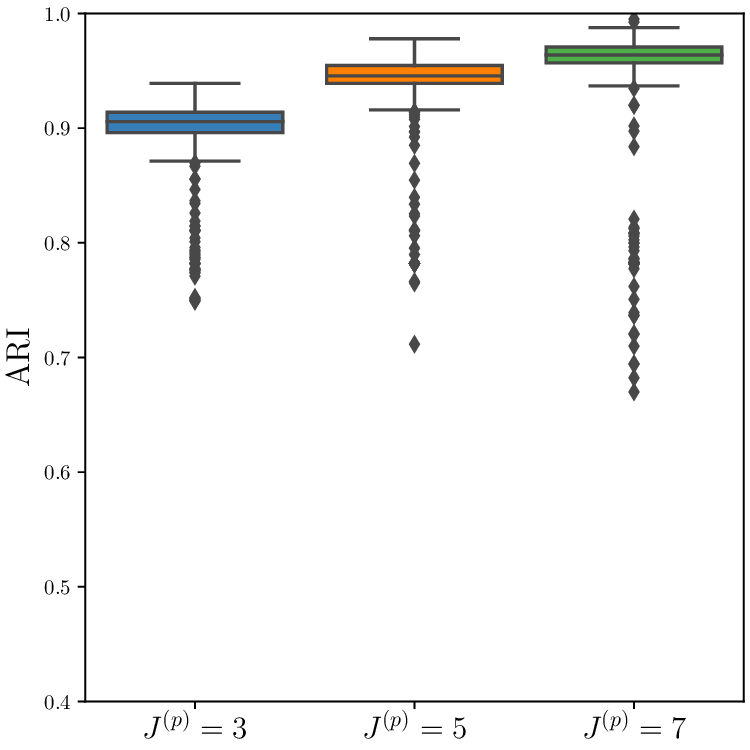}
         \caption{Scenario 3}
         \label{fig:influence_scenario_3}
     \end{subfigure}
     \caption{Estimation of \ARI{} for the influence of $J^{(p)} $}
    \label{fig:ncomp_comparison}
\end{figure}

Table \ref{tab:scenario2_ncomnode} presents the estimation of the number of clusters and \ARI{} when $J^{(p)}$, and thus $J$, varies between the nodes. We use $J^{(p)} = 4$ for the first split and we then reduce to $J^{(p)} = 2$ the rest of the construction of the tree. We remark that the number of clusters is well estimated. The \ARI{} does not really change compared to the case where we do not change $J^{(p)}$ through the tree. 
The investigation of an alternative data-driven choice of $J^{(p)} $ is left for future work.

\begin{table}
\begin{subtable}{\textwidth}
\centering
\begin{tabular}{lrrrr}
Number of clusters &      5 &      6 &      7 &      8 \\
\midrule
\texttt{fCUBT} &  95 &  3 &  1 &  1 \\
\end{tabular}
\caption{Scenario 2 -- Number of clusters selected for \texttt{fCUBT} for $500$ simulations as a percentage.}
\label{tab:n_cluster_2_ncompnode}
\end{subtable}%

\begin{subtable}{\textwidth}
\centering
\begin{tabular}{lrrrrrrrrrrr}
Quantile &       0.0 &       0.1 &       0.2 &       0.3 &       0.4 &       0.5 &       0.6 &       0.7 &       0.8 &       0.9 &       1.0 \\
\midrule
\texttt{fCUBT} &  0.59 &  0.87 &  0.89 &  0.90 &  0.90 &  0.90 &  0.91 &  0.91 &  0.92 &  0.92 &  0.94 \\
\end{tabular}
\caption{Scenario 2 -- Quantile of the ARI for $500$ simulations.}
\label{tab:n_cluster_2_ari_ncompnode}
\end{subtable}
\caption{Results for Scenario 2 with different $J^{(p)} $ per nodes}
\label{tab:scenario2_ncomnode}
\end{table}

\subsection{On the Gaussian assumption}

A Reviewer asked whether  the multivariate Gaussian distribution is a strict assumption. Our  Lemma \ref{lemma:c_j} crucially uses the Gaussian assumption. If the $\xi_j$ in (2.1) are not Gaussian, in general the change of basis does not preserve the distributions. If one assumes that the basis $\{\phi_j\}_{j \geq 1}$ is given, it could be possible to consider mixture models with other distributions for the $\xi_j$. This would be a different modeling approach, which will be considered elsewhere. 

 The Gaussian assumption could be tested using, for instance, the tests proposed in \cite{norm_test}. We investigated the effectiveness of such tests in our simulation framework. We consider six multivariate normality tests: Mardia's (MJB), Royston's, Henze-Zirkler's, Energy, Doornik-Hansen's and Lobato-Velasco's normality tests. The tests are run on the leaves of the grown trees for $100$ replications of each of the four  Scenarios considered above. As the number of leaves can be different for each tree, we test each leaf independently, and gather the results. Figure \ref{fig:gaussian_assumption} reports the $p$-values of the different tests under the null hypothesis for our different scenarios, with $N=1000$. For Scenario 1, only the Royston's multivariate normality test performs satisfactory, that is in most of the replications it does not reject the Gaussian scores hypothesis at the $5\%$ level. For Scenarios 2, 3 and 4, all the tests perform well and indicate the data are likely to our Gaussian assumption.  

\begin{figure}
     \centering
     \begin{subfigure}[b]{0.24\textwidth}
         \centering
         \includegraphics[scale=0.3]{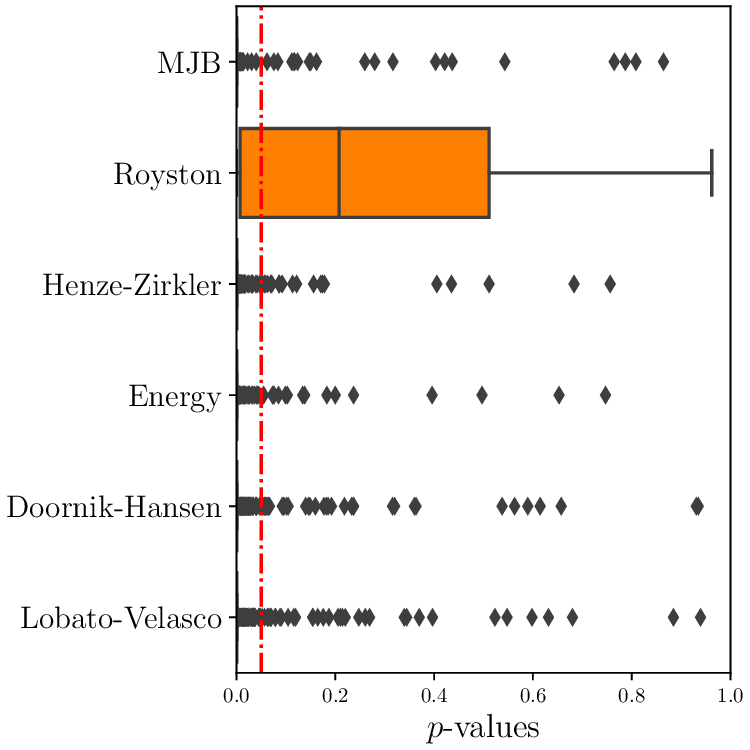}
         \caption{Scenario 1}
         \label{fig:gaussian_assumption_1}
     \end{subfigure}
     \begin{subfigure}[b]{0.24\textwidth}
         \centering
         \includegraphics[scale=0.3]{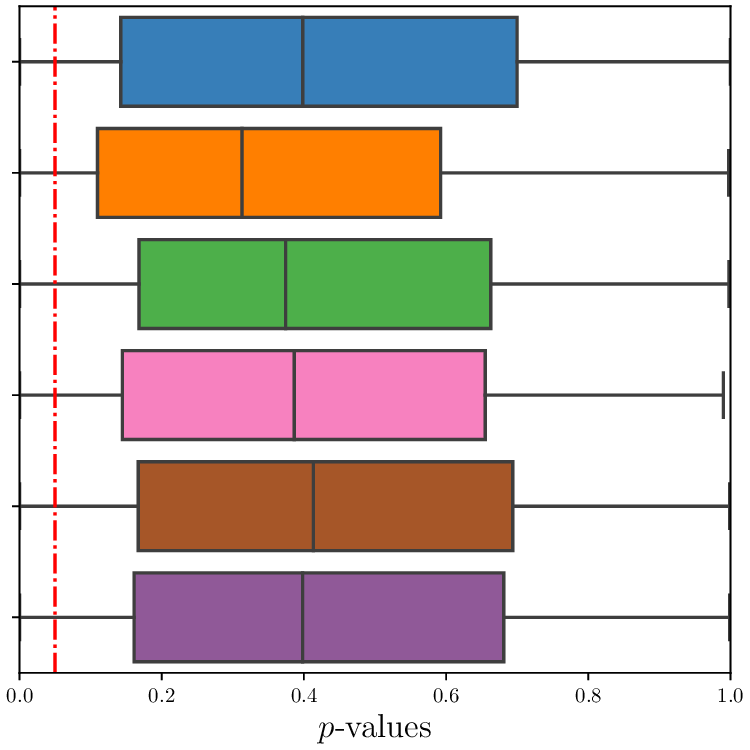}
         \caption{Scenario 2}
         \label{fig:gaussian_assumption_2}
     \end{subfigure}
     \begin{subfigure}[b]{0.24\textwidth}
         \centering
         \includegraphics[scale=0.3]{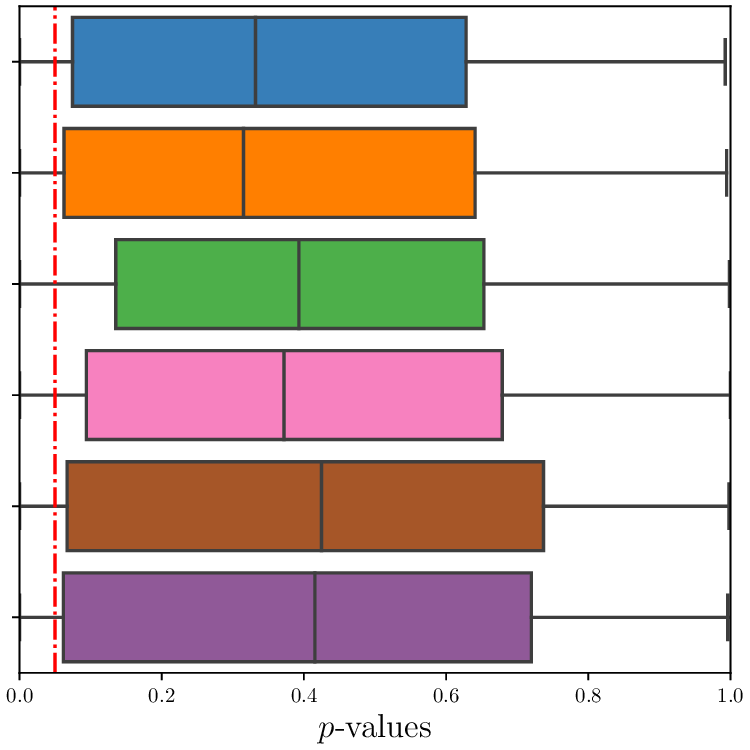}
         \caption{Scenario 3}
         \label{fig:gaussian_assumption_3}
     \end{subfigure}
     \begin{subfigure}[b]{0.24\textwidth}
         \centering
         \includegraphics[scale=0.3]{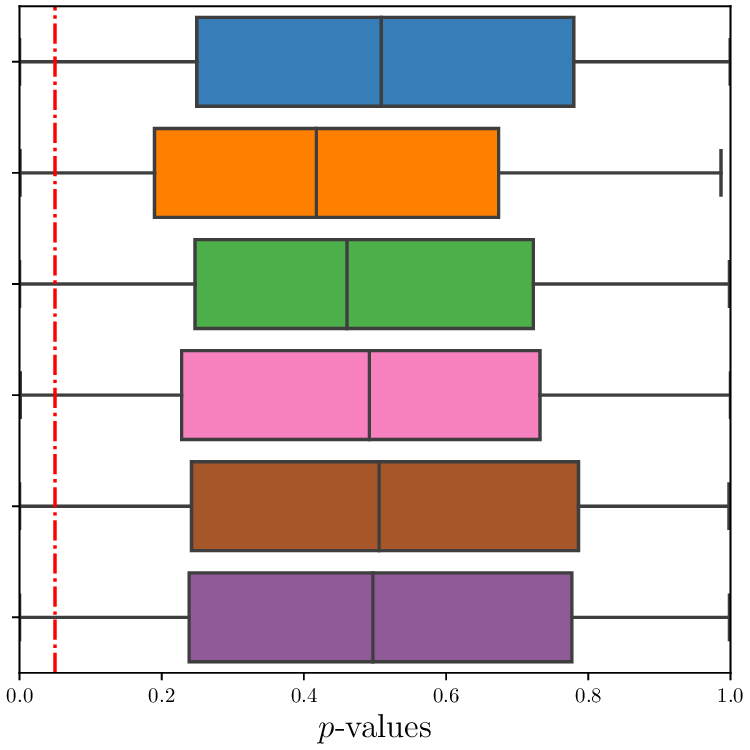}
         \caption{Scenario 4}
         \label{fig:gaussian_assumption_4}
     \end{subfigure}
     \caption{$p$-values for the different normality tests for different Scenarios. The red dashed line represents the $5\%$ level.}
    \label{fig:gaussian_assumption}
\end{figure}

Finally, in reply to one Reviewer's inquiry, we also investigate the robustness of our approach to departures from the Gaussian assumption. 
For this purpose, we reconsider Scenario 1 where the coefficients $c_{11}, c_{12},\ldots,c_{23}$ are generated according to a double exponential distribution with densities $f(x) = \exp(-|x/\beta|)/(2\beta)$, $x\in\mathbb R$, where the different values of the $\beta$ parameter are fixed in a such way as to preserve the variances considered in \eqref{c_law_s1}.   The results obtained from 500 replications are reported in Table \ref{tab:scenario1_gaussain}. 

We notice that our algorithm has a lesser performance for detecting the correct number of clusters, though it is able to do it in a majority of cases (53\%). The quantiles of \ARI{} are also worse than with the Gaussian coefficients $c_{11}, c_{12},\ldots,c_{23}$. However, the performance of the algorithm remains good and comparable to the performance of the competing approaches as we reported in Table \ref{tab:n_clusters}, in the Gaussian coefficients setup. In conclusion, our approach seems to be robust with respect to some departures from the Gaussian assumption. 

\begin{table}
\begin{subtable}{\textwidth}
\centering
\begin{tabular}{lrrrrr}
Number of clusters &      5 &      6 &     7 &      8 &      9 \\
\midrule
\texttt{fCUBT} &  53 &  35 &  8 &  3 &  1 \\
\end{tabular}
\caption{Scenario 1 -- Number of clusters selected for \texttt{fCUBT} for $500$ simulations as a percentage.}
\label{tab:n_cluster_1_laplace}
\end{subtable}%

\begin{subtable}{\textwidth}
\centering
\begin{tabular}{lrrrrrrrrrrr}
Quantile &       0.0 &       0.1 &       0.2 &       0.3 &      0.4 &       0.5 &       0.6 &       0.7 &       0.8 &       0.9 &       1.0 \\
\midrule
\texttt{fCUBT} &  0.49 &  0.65 &  0.68 &  0.69 &  0.71 &  0.72 &  0.73 &  0.74 &  0.75 &  0.77 &  0.82 \\
\end{tabular}
\caption{Scenario 1 -- Quantile of the \ARI{} for $500$ simulations.}
\label{tab:n_cluster_1_ari_laplace}
\end{subtable}
\caption{Results for Scenario 1 with double exponential coefficients}
\label{tab:scenario1_gaussain}

\qquad
\end{table}

\subsection{Another comparison with supervised methods}

This part is similar to Section \ref{sec:compa_sup}, except that, here, we only learn the clusters on the test subset for the \texttt{fCUBT} algorithm.  We performed the simulations $500$ times, and the results are plotted in Figure \ref{fig:simu_comparison}. Between parenthesis, we have written the number of times \texttt{fCUBT} gets the right number of clusters over the $500$ simulations, and so the number of simulations we examine for the computation of the \ARI{}. We point out that this retrieval percentage is much smaller than that in Section \ref{sec:compa_sup} but it is due to the fact that here, we only considers a dataset of size $N_1 = 330$ for applying our algorithms. We remark that our unsupervised method is as good as supervised ones when the true number of classes is found.

\begin{figure}
\qquad 

     \centering
     \begin{subfigure}[b]{0.3\textwidth}
         \centering
         \includegraphics[scale=0.35]{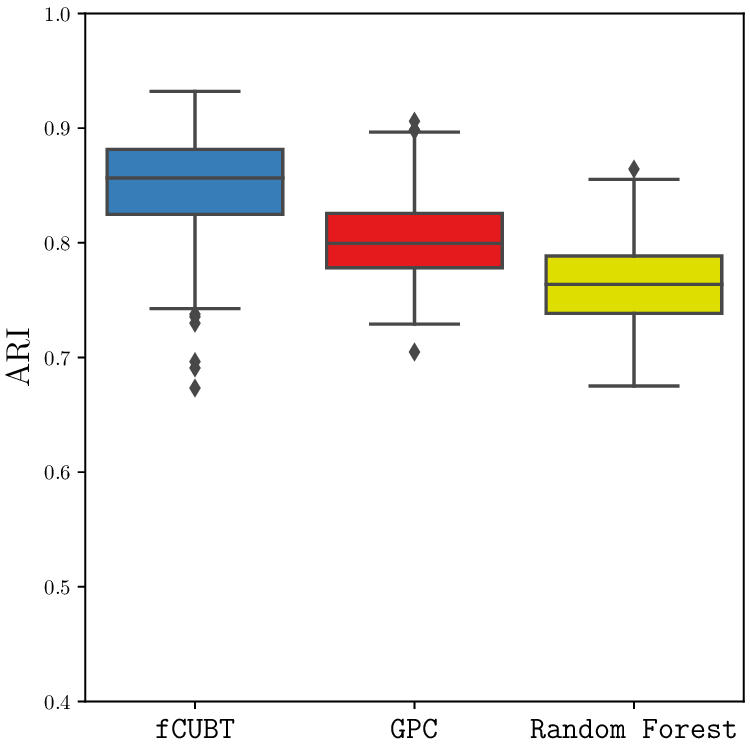}
         \caption{Scenario 1 $(257 / 500)$}
         \label{fig:comparison_scenario_1}
     \end{subfigure}
     \begin{subfigure}[b]{0.3\textwidth}
         \centering
         \includegraphics[scale=0.35]{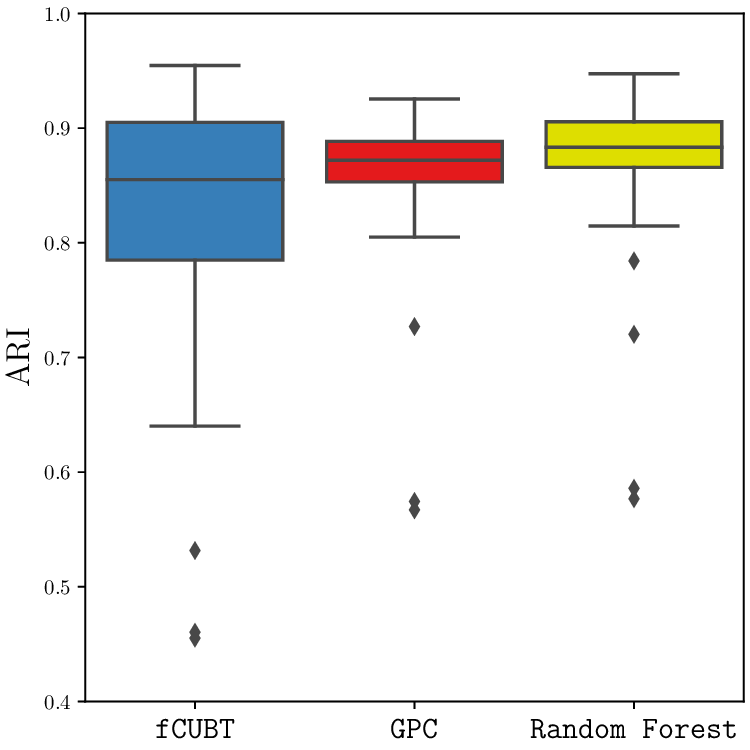}
         \caption{Scenario 2 $(172 / 500)$}
         \label{fig:comparison_scenario_2}
     \end{subfigure}
     \begin{subfigure}[b]{0.3\textwidth}
         \centering
         \includegraphics[scale=0.35]{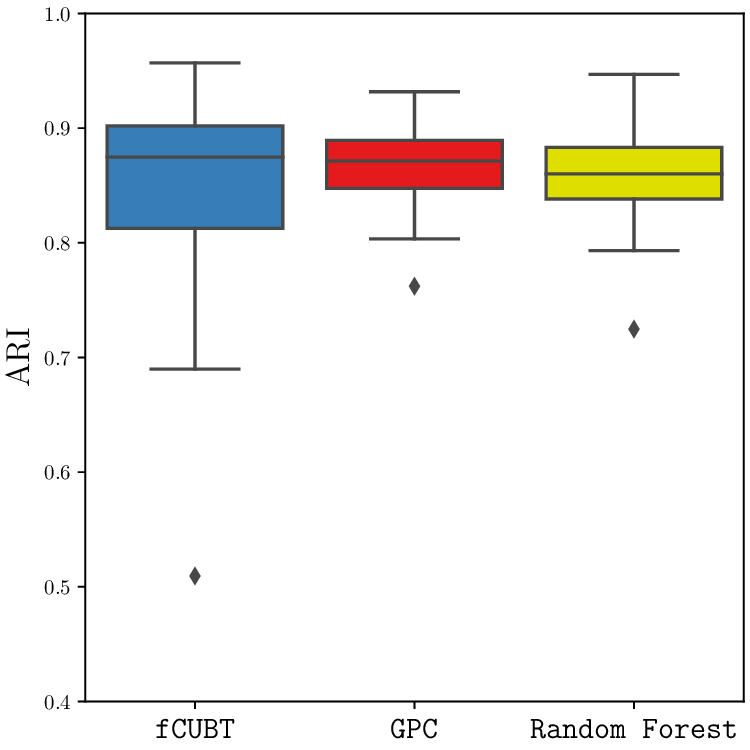}
         \caption{Scenario 3 $(188 / 500)$}
         \label{fig:comparison_scenario_3}
     \end{subfigure}
     \caption{Estimation of \ARI{} for the comparison with supervised models. The number of replications among the 500 where  \texttt{fCUBT} recovers the correct number of clusters $K$ is given between parenthesis}
    \label{fig:simu_comparison}
\end{figure}

\section{Real data analysis: the rounD dataset}\label{app:real_data}

In this section, we provided details on the rounD dataset.  The aerial view of the roundabout is presented in Figure \ref{fig:roundabout}.  
This dataset is part of a set of vehicle trajectory data provided by the Institute for Automotive Engineering (ika) in RWTH Aachen University. One may cite the highD dataset (about highways) \cite{krajewski_highd_2018} and the inD dataset (about intersections) \cite{bock_ind_2019}, such others produced by ika. These datasets are particularly useful for studying the behavior of road users in some specific situations. They start to replace the Next Generation Simulation (NGSIM) study \cite{fhwa_u.s._department_of_transportation_ngsimnext_2006}, widely used in traffic flow studies, as a benchmark for models about trajectory prediction or classification, because they provide more accurate data (see \emph{e.g.} \cite{izquierdo_vehicle_2020,messaoud_attention_2020,messaoud_non-local_2019,diehl_graph_2019,bhattacharyya_conditional_2020,wu_trajectory_2020} for some references). 

\begin{figure}
\centering
\includegraphics[scale=0.5]{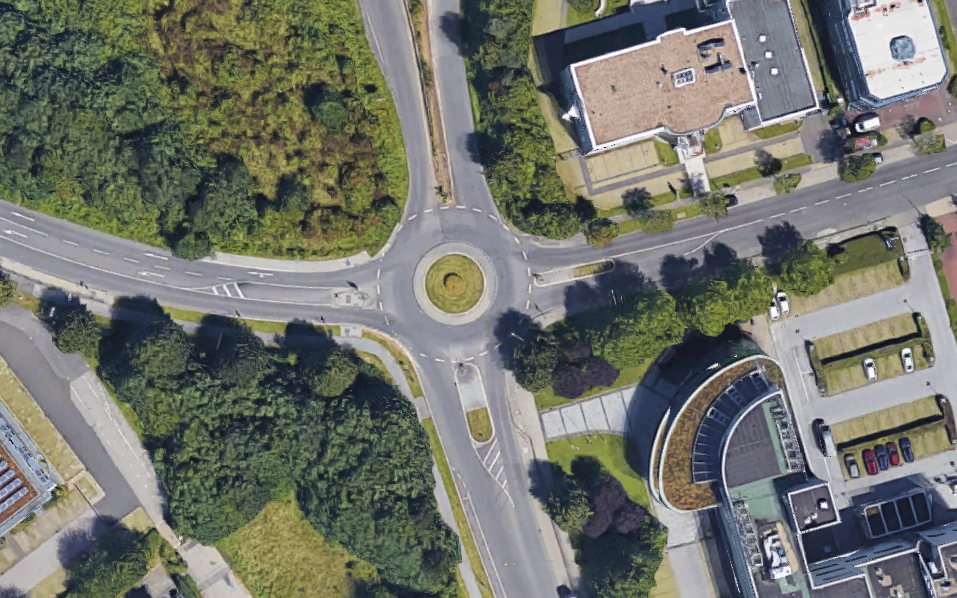}
\caption{rounD dataset -- the considered roundabout. Source: Google Maps}
\label{fig:roundabout}
\end{figure}
%

Figure \ref{fig:node_eigenfunction} presents the first eigenfunction for each coordinate $X^{(p)}$ in the first MFPCA at the beginning of the tree. 

\begin{figure}
    \centering
    \includegraphics[scale=0.4]{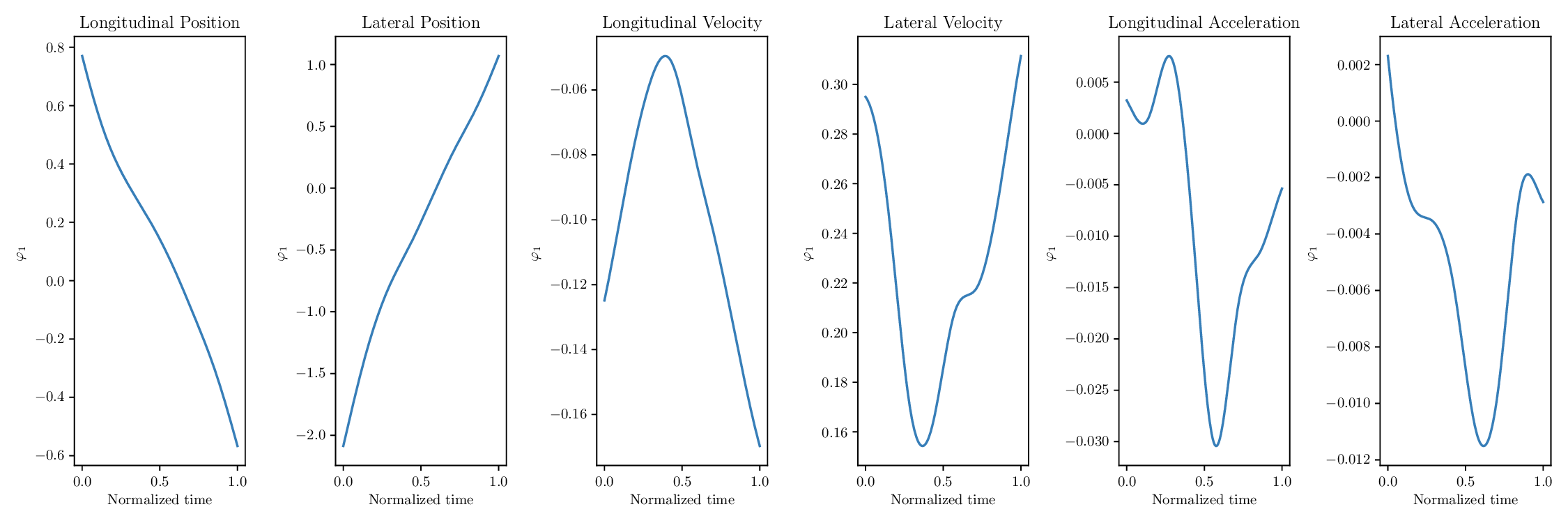}
    \caption{ The first eigenfunction for each coordinate $X^{(p)}$ of the first node of the tree. }
    \label{fig:node_eigenfunction}
\end{figure}

Figure \ref{fig:round_tree} presents the first three levels of the tree, starting from the root, obtained using a small subsample (that represents $311$ multivariate curves) of the rounD dataset. This sub-tree helps to understand how the clusters are built. The two colors represents the binary splitting at this node. We see that the different entry/exit combinations are already separated in the three steps. The split inside an entry/exit scenario is produced in the subsequent binary splits of the tree. 

We tested whether  the Gaussian assumption used for our model-based clustering is reasonable on the rounD dataset. We consider six multivariate normality tests: Mardia’s (MJB), Royston’s, Henze-Zirkler’s, Energy, Doornik-Hansen’s and Lobato- Velasco’s normality tests. The tests are performed on the leaves of the tree with more than $10$ observations. Figure \ref{fig:gaussian_round} presents the $p$-values computed for each test. In a majority of the cases, we do not reject the null hypothesis. The data are thus likely to have a mixed Gaussian distribution.

\begin{sidewaysfigure}[ht]
    \centering
    \includegraphics[width=\textwidth]{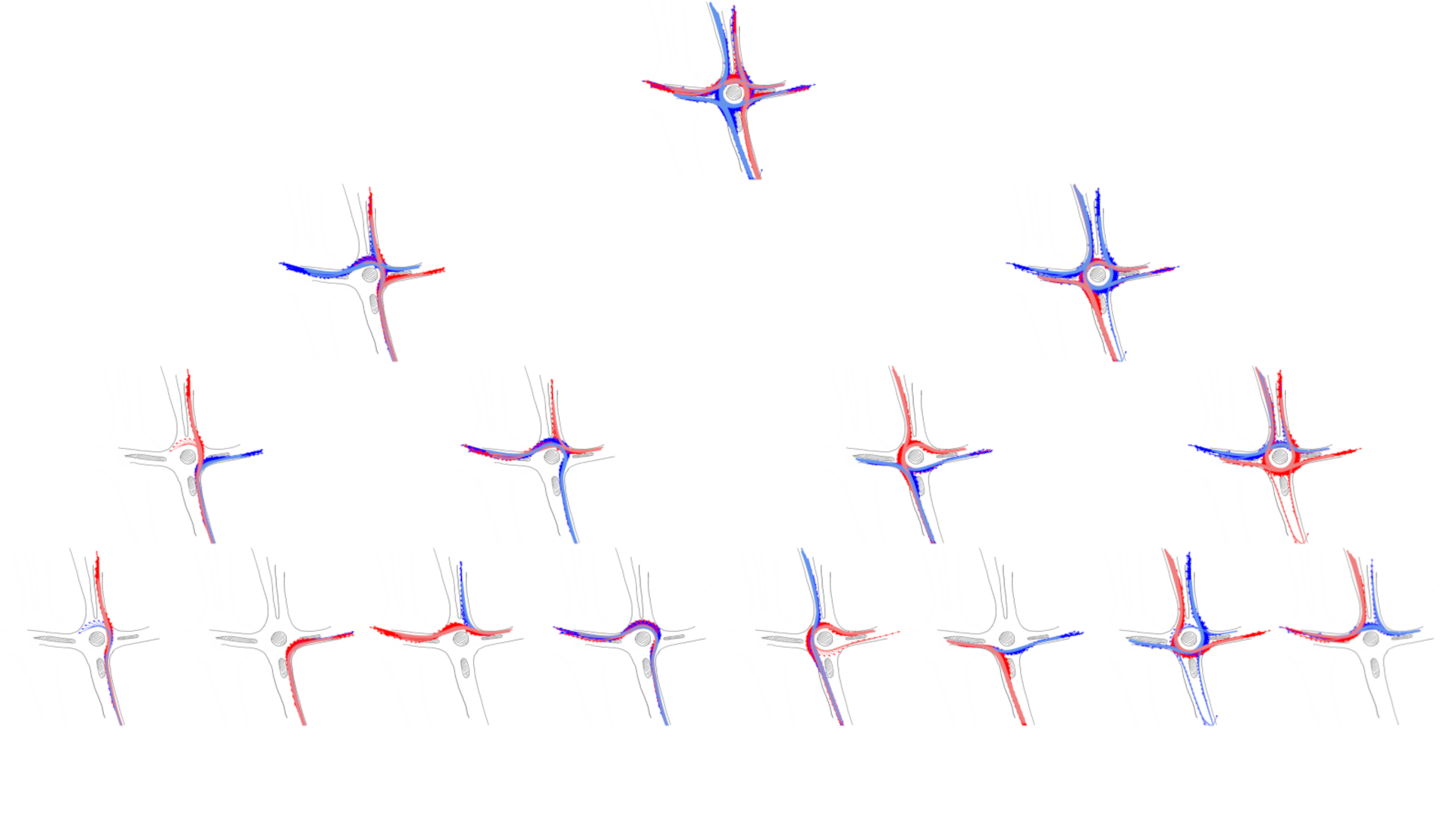}
    \caption{Beginning of the tree for the rounD dataset.}
    \label{fig:round_tree}
\end{sidewaysfigure}

\begin{figure}
    \centering
    \includegraphics[scale=0.45]{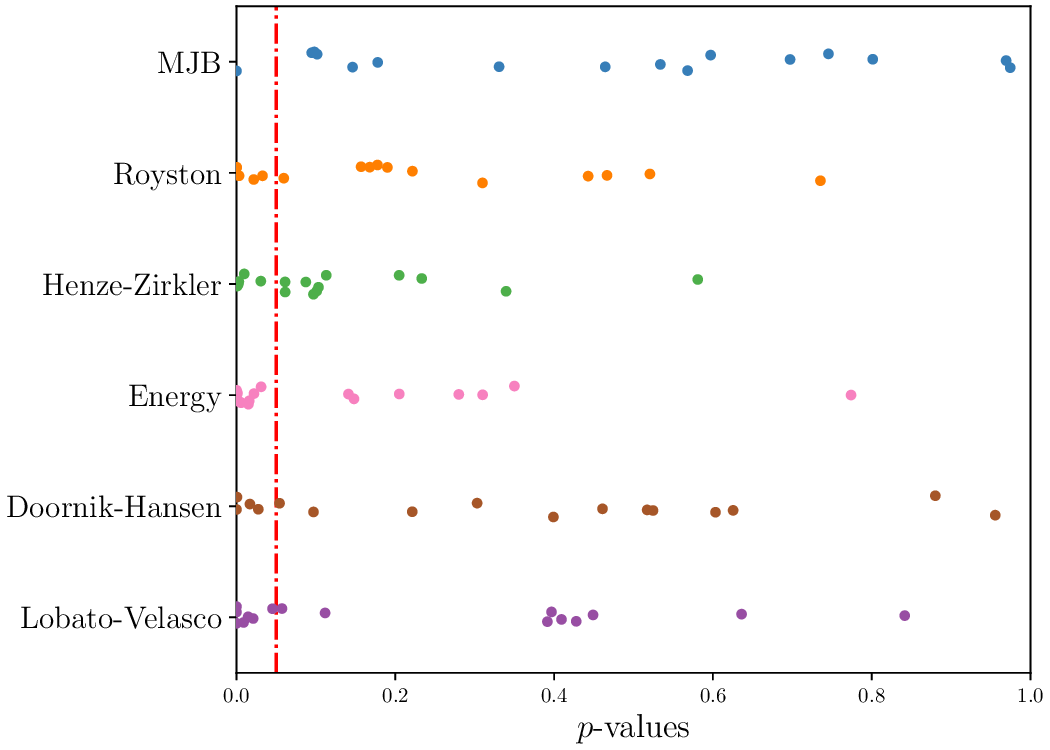}
    \caption{ $p$-values for the different normality tests for the rounD dataset. Each point represents one leaf of the tree with more than $10$ observations. The red dashed line represents the $5\%$ level. }
    \label{fig:gaussian_round}
\end{figure}

\bibliographystyle{abbrvnat}
\bibliography{biblio}

\end{document}